%% file: main.tex
\documentclass{article}

\PassOptionsToPackage{sort, numbers, compress}{natbib}

\usepackage[preprint]{neurips_2026}

\usepackage[utf8]{inputenc} 
\usepackage[T1]{fontenc}    
\usepackage{hyperref}       
\usepackage{url}            
\usepackage{booktabs}       
\usepackage{amsfonts}       
\usepackage{nicefrac}       
\usepackage{microtype}      
\usepackage{xcolor}         

\usepackage{adjustbox}
\usepackage{amsmath}
\usepackage{amssymb}
\usepackage{amsthm}
\usepackage[capitalize,noabbrev]{cleveref}
\usepackage{enumitem}
\usepackage{float}
\usepackage{mathtools}
\usepackage{multicol}
\usepackage{nicematrix}
\usepackage{pgfplots}
\usepackage{subcaption}
\usepackage{thmtools}
\usepackage{tikz} 
\pgfplotsset{compat=1.18}

\usepackage[amsbb, bmbold, mlr]{my_macros}

\newcommand{\titleref}[2]{%
    \texorpdfstring{\Cref{#2}}{#1~\ref{#2}}%
}

\hypersetup{
colorlinks = true,
citecolor  = tabgreen,
linkcolor  = tabred,
filecolor  = tabblue,
urlcolor   = tabblue,
}

\newcommand{\sfgd}{\textsf{SF-GD}}
\newcommand{\sfsgd}{\textsf{SF-SGD}}
\newcommand{\sfggd}{\textsf{SF-(S)GD}}
\newcommand{\sfode}{\textsf{SF-ODE}}
\newcommand{\warmstep}{\ensuremath{T_\mathsf{w}}}
\newcommand{\warmratio}{\ensuremath{p_\mathsf{w}}}
\newcommand{\Ecs}{\ensuremath{E_{\text{cs}}}}
\newcommand{\Eu}{\ensuremath{E_{\text{u}}}}
\DeclareMathOperator{\PsHy}{PH} 

\theoremstyle{plain}
\newtheorem{thm}{Theorem}[section]
\newtheorem{prop}[thm]{Proposition}
\newtheorem{lem}[thm]{Lemma}
\newtheorem{cor}[thm]{Corollary}
\theoremstyle{definition}
\newtheorem{defn}[thm]{Definition}
\newtheorem{asmp}[thm]{Assumption}

\theoremstyle{remark}

\Crefname{asmp}{Assumption}{Assumptions} 
\Crefname{thm}{Theorem}{Theorems}
\Crefname{prop}{Proposition}{Propositions} 
\Crefname{lem}{Lemma}{Lemmata}
\Crefname{cor}{Corollary}{Corollaries}
\Crefname{fact}{Fact}{Facts}
\Crefname{defn}{Definition}{Definitions}
\Crefname{exmp}{Example}{Examples}
\Crefname{obsrv}{Observation}{Observations}
\Crefname{rmk}{Remark}{Remarks}

\title{Understanding Schedule-Free Methods in Nonconvex Optimization: Rate Guarantees and Escaping Saddles}

\author{%
  Jiseok Chae\\
  Natural Science Research Institute\\
  KAIST\\
  Daejeon, Republic of Korea\\
  \texttt{jsch@kaist.ac.kr}\\
  \And
  Donghwan Kim\\
  Department of Mathematical Sciences\\
  KAIST\\
  Daejeon, Republic of Korea\\
  \texttt{donghwankim@kaist.ac.kr}\\
}

\begin{document}

\maketitle

\begin{abstract}
  Schedule-Free methods have attracted growing interest for alleviating the burden of designing and tuning a learning rate scheduler, while matching and sometimes even outperforming optimizers with tuned schedulers. 
  Despite their strong empirical results, their convergence theory in nonconvex optimization, where modern machine learning objectives typically arise, has remained largely unexplored. 
  In this paper, we provide worst-case analyses of Schedule-Free gradient descent and Schedule-Free stochastic gradient descent, in their standard form and without auxiliary modifications or restrictive conditions, for smooth but possibly nonconvex objectives.
  Based on a Lyapunov analysis derived from the continuous-time limiting ordinary differential equation associated with these methods, we show that Schedule-Free gradient descent and Schedule-Free stochastic gradient descent achieve the optimal worst-case convergence rates attainable among first-order methods.
  We further formulate Schedule-Free gradient descent as a nonautonomous dynamical system and prove strict-saddle avoidance under an arbitrarily small one-time perturbation.
  These theoretical results provide a better understanding of the strong performance that Schedule-Free methods demonstrate.
\end{abstract}

\section{Introduction} \label{sec:intro}

There are many first-order methods that can be used to solve a minimization problem \begin{equation} \label{eqn:1}
    \min_{\vx \in \rr^d} \; f(\vx),
\end{equation} ranging from the established \emph{gradient descent}~(GD) and \emph{stochastic gradient descent}~(SGD) methods to those that are now standard in the machine learning community, such as Adam~\citep{King15} and AdamW~\citep{Losh17}.
In modern machine learning tasks, such first-order methods are typically used together with a learning rate \emph{scheduler}, which changes the learning rate during training according to a predefined rule.
However, as training dynamics are largely unpredictable \textit{a priori}, the choice and design of the learning rate scheduler have traditionally relied heavily on ad hoc approaches and heuristics.

Schedule-Free \citep{Defa24} is an optimization scheme that aims to eliminate the need for handcrafted learning rate schedules, which are often difficult to design and tune. 
The update rule of the Schedule-Free method, which combines the techniques of iterate averaging and a momentum-like interpolation with a given base optimizer, is defined as
\begin{subequations} \label{eqn:SF} \begin{align}
    \vy_k &= (1-\beta) \vz_k + \beta \vx_k \label{eqn:yk} \\
    \vz_{k+1} &= \vz_k - \gamma_k \vg_k \label{eqn:zk} \\
    \vx_{k+1} &= (1-c_{k+1}) \vx_k + c_{k+1} \vz_{k+1} \label{eqn:xk}
\end{align} \end{subequations} with initial points $\vx_0 = \vz_0$, where $\beta \in [0,1]$ is a fixed constant, $\{\gamma_k\}_{k \geq 0}$ is the sequence of learning rates which is often set to be a constant sequence, $\{c_{k+1}\}_{k \geq 0}$ is a predefined sequence of averaging rates such that $c_{k+1} \in [0, 1]$, and $\vg_k$ is the update direction produced by the base optimizer at $\vy_k$. 
For example, Schedule-Free stochastic gradient descent (\sfsgd) uses $\vg_k = \nabla f(\vy_k, \zeta_k)$ for some random variable $\zeta_k$ accounting for the stochasticity in the gradient evaluation, as in standard SGD.

A notable interpretation of the Schedule-Free scheme is to understand it as an interpolation between two well-known averaging schemes.
When $\beta = 0$, under the standard choice $c_{k+1} = \frac{1}{k+1}$, \mbox{Schedule-Free} reduces to \emph{Polyak--Ruppert averaging}~\citep{Poly90,Rupp88}, where $\vx_k$ becomes the uniform average over the computed iterates $\vy_0, \dots, \vy_k$.
On the other extreme, when $\beta = 1$, Schedule-Free reduces to a scheme known as \emph{primal averaging}~\citep{Nest15, Tao18}. 
Despite this connection, $\beta = 1$ is seldom used in Schedule-Free methods.
A few common choices for $\beta$ in practice are $0.9$ and $0.98$, following the original experiments conducted by \citet{Defa24}. 
In this paper, as the case $\beta = 1$ is well studied in the prior works on primal averaging, we focus on when $\beta < 1$. 

Schedule-Free methods have gained significant interest from the community, not only because they alleviate the complexities associated with learning rate schedules, but also because they demonstrate strong performance on practical deep learning workloads. 
This has been highlighted by the results from the recent AlgoPerf challenge~\citep{Dahl23}, a benchmark of a wide range of optimization algorithms on large-scale deep learning tasks.
 
Despite their empirical success in modern machine learning applications, where the underlying objectives are typically nonconvex, convergence guarantees for Schedule-Free methods were available only in the convex setting.
\citet{Defa24} introduced the method and provided an accompanying convergence-rate analysis in the convex setting.
There have been some attempts to extend convergence guarantees to the nonconvex regime, for instance by applying an online-to-nonconvex conversion framework \citep{Ahn25} or by directly constructing a tailored Lyapunov potential \citep{Brow25}. 
However, these analyses have notable limitations as we detail in \Cref{sec:related}, such as focusing on specific parameter choices or requiring a randomization of the parameters such as $\beta$. 
As a result, a comprehensive convergence analysis of Schedule-Free methods in the nonconvex setting remains largely unexplored.

\subsection{Our Contributions}
To address this gap, we analyze Schedule-Free methods under smooth and possibly nonconvex objectives through the worst-case convergence of the gradient location iterates, strict-saddle avoidance of the dynamics, and the worst-case behavior of the evaluation iterates commonly used in practice. 
Our contributions are as follows.

\begin{itemize}
    \item 
    As a starting point, we consider the continuous-time limiting ODE of \sfgd. 
    We construct its Lyapunov function, and establish that the solution of the \sfode{} attains an $O(1/T)$ decay rate in terms of the squared gradient norm $\|\nabla f\|^2$. 
    
    \item 
    Based on the analyses of the continuous-time limiting ODE, we establish convergence rate analyses for discrete-time Schedule-Free methods. 
    These results show that $\{\vy_k\}_{k \geq 0}$ computed by \sfgd{} and \sfsgd{} achieve the worst-case optimal rates known for first-order methods~\citep{Arje23, Carm21}. 
    In particular, we prove convergence rates for the methods \emph{as is}, unlike prior works that analyze specifically modified variants or require careful parameter tuning.  
    
    \item \looseness=-1
    We study \sfgd{} as a nonautonomous dynamical system and prove that, after an arbitrarily small one-time perturbation, it almost surely avoids strict saddle points.
    Thus, under the usual strict-saddle landscape assumptions, convergent Schedule-Free trajectories exhibit the same qualitative behavior as GD, converging only to local minimizers rather than to saddle points.

    \item
    We use the performance estimation problem (PEP) framework to study the averaged evaluation sequence $\{\vx_k\}_{k \geq 0}$ directly. 
    The PEP results show that the worst-case behavior of $\{\vx_k\}_{k \geq 0}$ can be worse than that of $\{\vy_k\}_{k \geq 0}$. 
    Nonetheless, we discuss how this gap can coexist with the strong empirical performance of $\{\vx_k\}_{k \geq 0}$.

\end{itemize}

\subsection{Problem Settings}
We consider the minimization problem \eqref{eqn:1} where the objective function $f :\rr^d \to \rr$ is differentiable and possibly nonconvex. 
For the analyses, we further impose a standard smoothness assumption. 
\begin{asmp}[{Smoothness}] \label{asmp:smooth}
   For some $L > 0$ given, $f$ is \emph{$L$-smooth}; that is, for any $\vx, \vy \in \rr^d$, \[
   \norm{\nabla f(\vx) - \nabla f(\vy)} \leq L\norm{\vx - \vy}.
   \]
\end{asmp}

To avoid pathological cases such as minimizing an affine function, we also impose the following.
\begin{asmp}[{Finite infimum}] \label{asmp:lower-bound}
    The objective function $f$ is bounded below by a constant $f^*$. 
\end{asmp}

We mainly focus on the case where the base optimizer in the Schedule-Free method is either GD, so that $\vg_k = \nabla f(\vy_k)$, or SGD, so that $\vg_k = \nabla f(\vy_k, \zeta_k)$, where $\nabla f(\,\cdot\,,\zeta)$ denotes a stochastic gradient oracle and $\{\zeta_k\}_{k \ge 1}$ is a sequence of random variables accounting for the stochasticity in the gradient evaluations.
For this stochastic gradient oracle we assume that it produces an unbiased estimator of $\nabla f$ with a bounded variance.
\begin{asmp}[{Stochastic gradients}] \label{asmp:stochastic}
   There exists some $\sigma^2 \geq 0$ such that, for any $\vx \in \rr^d$, the stochastic gradient oracle satisfies 
   \begin{multicols}{2} \begin{enumerate}[label=(\roman*)]
       \item $\expt[\nabla f(\vx, \zeta)] = \nabla f(\vx)$ ,
       \item $\expt\bigl[\norm{\nabla f(\vx, \zeta) - \nabla f(\vx)}^2\bigr] \leq \sigma^2$.
   \end{enumerate} \end{multicols}
\end{asmp}

\section{Related Work} \label{sec:related}

\paragraph{Convergence Analysis of Schedule-Free Methods}
\citet{Defa24} have provided convergence guarantees of \sfsgd{} when the objective function $f$ is convex, but the nonconvex convergence theory for Schedule-Free methods remains largely open.
\citet{Ahn25} were the first to attempt to attain convergence bounds for Schedule-Free methods in the nonconvex regime.
Based on an online-to-nonconvex conversion framework, their discussion also extends to the case where $f$ is \emph{nonsmooth}.
Concurrently, \citet{Brow25} have claimed a comparatively simpler and more direct proof of convergence via a construction of Lyapunov potentials tailored to nonconvex smooth objectives.

However, both works have limitations that prevent them from being considered to provide comprehensive convergence proofs for Schedule-Free methods in the nonconvex regime.
The analysis by \citet{Ahn25} establishes a worst-case optimal rate, but only for a specific choice of parameters and, moreover, requires
additional modifications such as making $\beta$ a random variable that is resampled every iteration.
Meanwhile, \citet{Brow25} proposed several convergence rates that depend on the choice of the averaging rates $\{c_{k+1}\}_{k \geq 0}$. 
However, the claimed rates are mostly suboptimal, especially being only $O(\frac{1}{\log T})$ at best where $T$ is the total number of iterations, under the standard choice of $c_{k+1} = \frac{1}{k+1}$. 
Moreover, their results unfortunately appear to hold only when $\beta = 1$, as the proof relies on a premise that, to our understanding, does not hold for $\beta <1$.  
See \Cref{appx:rltwrk} for further details on these limitations.
Our analyses of the convergence rates, on the other hand, apply to a wide range of $\beta$ and study Schedule-Free methods as is, without any modifications. 

\paragraph{Gradient Location Sequence vs. Evaluation Sequence}

The sequence $\{\vy_k\}_{k \geq 0}$ obtained from a Schedule-Free method is called the \emph{gradient location} sequence, whereas the sequence $\{\vx_k\}_{k \geq 0}$ is called the \emph{evaluation} sequence.  
It is clear that the former is named as such because of \eqref{eqn:zk}. 
The latter is named because the original work by \citet{Defa24} on Schedule-Free methods advocated using $\{\vx_k\}_{k \geq 0}$ as the ``main'' iterate to be used as the trained parameters in constructing a model.
Since then, treating $\{\vx_k\}_{k \geq 0}$ as the primary output sequence has largely become standard practice.

Nonetheless, \citet{Song25} observed that $\{\vy_k\}_{k \geq 0}$ is not merely an auxiliary sequence. 
In particular, they showed that $\{\vy_k\}_{k \geq 0}$ exhibits \emph{edge of stability}~\citep{Cohe21} behavior, a well-documented large learning rate phenomenon in neural network training, and that, on the so-called \emph{river valley} loss landscapes~\citep{Wen25}, it follows the low loss path, often referred to as the \emph{river}, more closely than $\{\vx_k\}_{k \geq 0}$.
They also reported a language model training example where evaluating at $\{\vy_k\}_{k \geq 0}$ yields better performance than evaluating at $\{\vx_k\}_{k \geq 0}$. 
These observations motivate studying $\{\vy_k\}_{k \geq 0}$ as an optimization sequence in its own right, along with the practically widely used $\{\vx_k\}_{k\geq 0}$.
 
\paragraph{Performance Estimation Problem} 
The performance estimation problem (PEP) framework \citep{Dror14} provides a systematic toolbox for numerically computing worst-case convergence guarantees for first-order methods over a prescribed function class.
In many cases, the resulting worst-case bound can be cast as a semidefinite programming problem, shifting much of the analysis from manual inequality chaining to the solution of a single convex optimization problem.
PEP has been used to obtain sharp guarantees for a wide range of methods and problem structures, from vanilla gradient descent to proximal methods and operator splitting schemes \citep{Ryu20, Abba22, Tayl18}.

\section{Continuous-Time Limiting ODE of Schedule-Free Methods} \label{sec:sfode}

Continuous-time systems, most representatively expressed in the form of ODEs, are often easier to analyze than their discrete-time counterparts. 
Moreover, the insights obtained by studying the continuous-time limiting ODE can often be transferred to the discrete-time system, sometimes even yielding a convergence proof. 

In this section, we show that this is also the case for Schedule-Free methods. 
As a first step in this direction, this section derives the continuous-time limiting ODE of \sfgd, where we take $\vg_k = \nabla f(\vy_k)$ in the update rule specified in \eqref{eqn:SF}, and analyzes its convergence properties. 
To this end, we focus on the standard choice of averaging rates $c_{k+1} = \frac{1}{k+1}$ and a constant learning rate $\gamma_k = \gamma$. 

In order to derive the continuous-time limiting ODE, we should take the limit as $\gamma \to 0$, so let us assume that $\gamma \ll 1$. 
Then \eqref{eqn:zk} yields \[
\dot{\vz} \approx \frac{\vz_{k+1} - \vz_{k}}{\gamma} = -\vg_k = -\nabla f(\vy_k),
\] leading to $\dot{\vz} = -\nabla f(\vy)$.

Next, let us consider the update rule \eqref{eqn:xk} for $\vx_k$. 
Because it involves an averaging rate $c_{k+1}$ that changes at every iteration, we introduce a time-dependent function $a(t)$ to represent this variation in continuous time. 
More precisely, we wish to view the update rule \eqref{eqn:xk} as a discretization of the ODE governing the evolution of $\vx(t)$, which involves an auxiliary function $a(t)$ accounting for the time-dependent averaging rate. 
To this end, let us rearrange the terms in \eqref{eqn:xk} and reformulate it into \begin{equation} \label{eqn:reformulated-xk}
    \vx_{k+1} = \vx_k + \gamma \cdot \frac{c_{k+1}}{\gamma(1-c_{k+1})} (\vz_{k+1} - \vx_{k+1}). 
\end{equation} One can observe that, if there were a function $a(t)$ such that $a((k+1)\gamma) = \frac{c_{k+1}}{\gamma(1-c_{k+1})}$, it would be possible to interpret the equation \eqref{eqn:reformulated-xk} as an implicit Euler step of the ODE \begin{equation} \label{eqn:x-ode}
\dot{\vx} (t) = a(t) (\vz (t) - \vx(t)).
\end{equation}
However, this would require expressing $\frac{c_{k+1}}{\gamma(1-c_{k+1})} = \frac{1}{k\gamma}$ as a function of $(k+1)\gamma$, which is not directly possible.  
So, instead, invoking that $1 - c_{k+1} \approx 1$ when $k \gg 1$, we seek a function $a(t)$ such that $a((k+1)\gamma) = \frac{c_{k+1}}{\gamma} = \frac{1}{(k+1)\gamma}$, and interpret \eqref{eqn:reformulated-xk} as an \emph{approximate} implicit Euler step of~\eqref{eqn:x-ode}. 
This, in fact, is easily achievable by simply setting $a(t) = \frac{1}{t}$. 

From these observations, we can construct the continuous-time limiting ODE of \sfgd{} as follows. 

\begin{defn} We define the \emph{Schedule-Free ODE} (\sfode) to be the continuous-time limiting ODE of \sfgd{} in the limit as $\gamma \to 0$, which can be formulated as a system of differential equations \begin{subequations} \label{eqn:SF-ODE}\begin{align}
    \vy &= (1-\beta) \vz + \beta \vx , \label{eqn:ode-y} \\
    \dot{\vz} &= - \nabla f(\vy) , \label{eqn:ode-z} \\
    \dot{\vx} &= \frac{1}{t} (\vz - \vx).  \label{eqn:ode-x}
\end{align}\end{subequations}  
\end{defn}
Despite the singularity of \eqref{eqn:ode-x} at $t = 0$, it turns out  for any given initial point $\vx_0 = \vz_0 \in \rr^d$, one can always find a function triple $(\vx, \vy, \vz)$ which becomes a solution of the \sfode{} on the interval $(0, T]$ for some $T > 0$ with $\lim_{t\to 0+} \bigl(\vx(t), \vy(t), \vz(t)\bigr) = (\vx_0, \vx_0, \vx_0)$. 
See \Cref{appx:ode-study} for further details. 

\sfode{} turns out to admit a Lyapunov function of a relatively simple form, namely, a weighted sum of $f(\vy)$ and $\norm{\vz - \vx}^2$, as shown in the following proposition.

\begin{prop} \label[prop]{prop:Lyapunov}
 Consider a function $\gL: (0, \infty) \to \rr$ defined as \begin{equation} \label{eqn:conti-lyap}
    \gL(t) = f(\vy(t)) - f^* + \frac{\beta}{2 t} \norm{\vz(t) - \vx(t)}^2.
    \end{equation} Then it holds that $\dot{\gL} \leq -(1-\beta) \norm{\nabla f (\vy)}^2 \leq 0$, and thus $\gL$ is a Lyapunov function of \sfode. 
\end{prop}

The bound on $\dot{\gL}$ derived in \Cref{prop:Lyapunov} can be further used to characterize the behavior of the solution trajectories as $t \to \infty$.  
In particular, we can show that $\norm{\nabla f(\vy(t))}^2 = O(1/T)$, as follows.

\begin{cor} \label[cor]{cor:asymps-of-sfode}
    For any $T > 0$, it holds that \[
    \min_{0 \leq t \leq T}\; \norm{\nabla f(\vy(t))}^2 \leq \frac{f(\vx_0) - f^*}{(1-\beta) T}.
    \]
\end{cor}

\section{Convergence Guarantees for \sfggd} \label{sec:conv} 

Motivated by the analyses of the \sfode{} in \Cref{sec:sfode}, to study the convergence of \sfgd{} and \sfsgd, it is natural to seek a discrete-time Lyapunov potential for them that mirrors the Lyapunov function~\eqref{eqn:SF-ODE} of \sfode. 
In this section, we explicitly construct such a potential and use it to obtain the corresponding convergence rate guarantees. 

While the Lyapunov potentials we construct are motivated by the Lyapunov function of the \sfode, which is the continuous-time limiting ODE of \sfgd{} with the specific choice $c_{k+1} = \frac{1}{k+1}$ of the averaging rates and a constant learning rate $\gamma_k = \gamma$, it turns out that the Lyapunov potentials in the discrete setting are much more flexible, allowing a unified analysis for any sequence of learning rates and averaging rates that satisfy the following conditions. 

\begin{asmp}\label[asmp]{asmp:lrn-rate}
    For a fixed constant $\gamma$ and the number of \emph{warmup steps} $\warmstep$ satisfying $1\leq \warmstep \leq T$, the learning rates are set as \begin{equation} \label{eqn:lr-assumption}
    \gamma_k = \gamma \min\left\{1, \frac{k+1}{\warmstep}\right\}.
    \end{equation} 
\end{asmp}
 
\looseness=-1
We write
$\warmratio \coloneqq \frac{\warmstep-1}{T}$, so that $\warmratio$ is the fraction of the optimization steps allocated to the warmup phase. 
Since $\gamma_k=\gamma$ already at $k=\warmstep-1$, whether to count the $(\warmstep-1)$th step as part of warmup is a stylistic choice.
We exclude it when defining $\warmratio$, and thus, we have $0 \le \warmratio < 1$.
In particular, when $\warmstep=1$ so that $\gamma_k = \gamma$ for all $k \geq 0$, there is essentially no warmup phase and $\warmratio = 0$. 
In practice, warmup is often $5\%$ of the run~\citep{Defa24}, so we view $\warmratio$ as a fixed small constant independent of $T$.

\begin{asmp}\label[asmp]{asmp:avg-rate}
    For $\gamma_k$ satisfying \Cref{asmp:lrn-rate}, there exists a constant $C > 0$ such that \begin{equation} \label{eqn:main-song-avg}
        c_{k+1} = \min\left\{\frac{\gamma_k^2}{\sum_{i=0}^k \gamma_i^2} \cdot (1-\beta) C ,\, 1\right\}.
    \end{equation}
\end{asmp} 

When $\beta<1$ and $C=\frac{1}{1-\beta}$, the rule \eqref{eqn:main-song-avg} reduces to the averaging rates suggested by \citet{Defa24}. 
Notice that taking $\warmstep=1$ so that there is essentially no warmup phase is equivalent to choosing a constant learning rate $\gamma_k = \gamma$, and yields the standard averaging rates $c_{k+1} = \frac{1}{k+1}$. 
Introducing the additional hyperparameter $C$ and setting the averaging rates as in \eqref{eqn:main-song-avg} follows the strategy proposed by \citet{Song25}, which is presented as a technique that empirically improves the
robustness to the parameter $\beta$ and performs better under large batch sizes, compared to when $C$ is fixed as $\frac{1}{1-\beta}$.

One noteworthy consequence of \Cref{asmp:avg-rate} is that $\{c_{k+1}\}_{k \geq 0}$ is then monotonically decreasing.
Hence, we may define $k_0$ to be the largest integer such that $c_{k+1} = 1$ for all $k \leq k_0$. 
If such a nonnegative integer $k$ does not exist, we naturally set $k_0 = -1$. 
Further discussions and additional details on \Cref{asmp:avg-rate} can be found in \Cref{appx:decr-avg-rates}.

\subsection{Convergence Analysis via the Design of Lyapunov Potentials}
The Lyapunov function \eqref{eqn:conti-lyap} of the \sfode{} provides strong evidence for imposing the ansatz that \sfgd{} has a Lyapunov potential of the form \begin{equation} \label{eqn:discrete_lyap_potential}
    V_k = f(\vy_k) - f^* + \alpha_k \norm{\vz_k - \vx_k}^2,
\end{equation} which, for a carefully chosen sequence of real numbers $\{\alpha_k\}_{k \ge 0}$, decreases monotonically. 
For \sfsgd, the Lyapunov potential might not decrease monotonically in the strict sense because additional terms are introduced by the noise in the stochastic gradients, but we expect these additional terms to be small enough to be well controlled by the decrease of the other terms. 

In the proof of \Cref{prop:Lyapunov}, a clean derivation that \eqref{eqn:conti-lyap} is a Lyapunov function of the \sfode{} relies crucially on the inner product term $\inprod{\dot{\vz}}{\dot{\vx}}$ canceling out, which is achieved by deliberately weighing $\norm{\vz(t) - \vx(t)}^2$.
Rewriting the inner product term using \eqref{eqn:SF-ODE} as $\inprod{\dot\vz}{\dot\vx} = -\frac{1}{t} \inprod{\nabla f(\vy)}{\vz - \vx}$, one expects that the term $\inprod{\nabla f(\vy_k)}{\vz_k - \vx_k}$, possibly up to a shift of some indices, will appear in the Lyapunov analysis of the discrete-time systems involving $V_k$.
The coefficients $\alpha_k$ should thus be chosen so as to cancel these inner product terms, ideally with ensuring that the remaining terms are monotonically decreasing, up to a small additive noise term when using stochastic gradients.

\Cref{prop:lyapunov-descent} shows that such a choice of $\alpha_k$ indeed exists. 
To provide a unified view, we establish the result in the stochastic setting, showing that the Lyapunov potentials are decreasing in expectation, up to an additive term induced by the noise in the gradients.
In order to recover the results for the deterministic setting, one can simply set $\sigma^2 = 0$ and ignore taking the expectations. 
\begin{prop}\label[prop]{prop:lyapunov-descent}
Suppose that \Cref{asmp:lower-bound,asmp:smooth,asmp:stochastic} hold. 
For $\beta \in [0, 1)$,
let $\{\vx_k\}_{k\geq 0}$, $\{\vy_k\}_{k\geq 0}$, and $\{\vz_k\}_{k\geq 0}$ be the iterates computed by \sfsgd{} with the learning rates $\{\gamma_k\}_{k \geq 0}$ satisfying \eqref{eqn:lr-assumption}, and the averaging rates $\{c_{k+1}\}_{k \geq 0}$ satisfying \eqref{eqn:main-song-avg}, where $k_0$ denotes the largest integer such that $c_{k+1} = 1$ for all $k \leq k_0$. 
Define a sequence of nonnegative real numbers $\{\alpha_k\}_{k\geq 0}$ as \[ 
    \alpha_k = \frac{\beta c_k (1-\gamma_{k-1} L (1 - \beta + \beta c_k))}{2\gamma_{k-1}(1-c_{k})^2}
\] for all $k = k_0+2, k_0+3, \dots$ and \[
\alpha_0 = \alpha_1 = \dots = \alpha_{k_0+1} = \frac{L \beta^2 c_{k_0+2}^2}{2} + \alpha_{k_0+2} (1-c_{k_0+2})^2.
\] Then, the Lyapunov potentials defined as in \eqref{eqn:discrete_lyap_potential} satisfy the inequality \[
    \expt[V_{k+1}] - \expt[V_k] \leq - \frac{\gamma_k(1-\beta)}{2} \expt \bigl[\norm{\nabla f(\vy_k)}^2 \bigr] + \frac{\gamma_k \sigma^2}{2} \bigl(\beta (1 + \gamma_k L (1-\beta)) c_{k+1} + \gamma_k L (1-\beta)^2 \bigr)
\] for all nonnegative integers $k \geq k_0 + 1$, and \[ 
    \expt[V_{k+1}] - \expt[V_k] \leq - \frac{\gamma_k }{2} \expt \bigl[\norm{\nabla f(\vy_k)}^2 \bigr] + \frac{\gamma_k^2 L \sigma^2}{2} 
\] for all nonnegative integers $k \leq k_0$. 
\end{prop}

For the proof, see \Cref{appx:lyap-descent}.
It should be noted that \citet{Brow25} have also considered a sequence of Lyapunov potentials whose overall form is quite similar to our $V_k$ in \eqref{eqn:discrete_lyap_potential}, although their Lyapunov potentials involve $f(\vx_k)$ instead of $f(\vy_k)$, with different coefficients $\alpha_k$.

\Cref{prop:lyapunov-descent} can be converted into a convergence bound for \sfsgd, as follows.

\begin{thm}\label{thm:master-convergence-thm}
    Suppose that \Cref{asmp:lower-bound,asmp:smooth,asmp:stochastic} hold. 
    For $\gamma \in (0, \frac{1}{L}]$ and $\beta \in [0, 1)$, let $\{\vx_k\}_{k\geq 0}$, $\{\vy_k\}_{k\geq 0}$, and $\{\vz_k\}_{k\geq 0}$ be the iterates computed by \sfsgd{} with the warmup length $\warmstep$ and the learning rates $\{\gamma_k\}_{k \geq 0}$ satisfying \eqref{eqn:lr-assumption}, and the averaging rates $\{c_{k+1}\}_{k \geq 0}$ satisfying \eqref{eqn:main-song-avg}.
    Then for any $T \geq 1$, letting $\warmratio = \frac{\warmstep-1}{T}$, we have the bound \begin{equation}\label{eqn:master-bound}  
        \min_{k=0, \dots, T-1} \expt \bigl[\norm{\nabla f(\vy_k)}^2 \bigr] \leq \frac{2(f(\vx_0) - f^*)}{\gamma (1-\beta) (1-\warmratio)T} + \frac{\gamma L (1-\beta) \sigma^2}{1-\warmratio} + \frac{3\beta(2-\beta) C \sigma^2 (1+\log T)}{ (1-\warmratio) T}. 
    \end{equation}
\end{thm} 

Notice that the convergence bound \eqref{eqn:master-bound} does not require any assumptions on $\beta$ other than $\beta < 1$, and also incorporates the warmup phase in the learning rates as used in practice.

\subsection{Rate-Optimal Convergence of \sfgd{} and \sfsgd}
From \Cref{thm:master-convergence-thm}, we can immediately deduce the convergence guarantee for \sfgd{} by setting $\sigma^2 = 0$.
As the method is deterministic, we can also drop the expectation from the left hand side of~\eqref{eqn:master-bound}, and obtain the following. 

\begin{cor}\label[cor]{cor:sfgd-rate}
    Suppose that \Cref{asmp:lower-bound,asmp:smooth} hold. 
    For any $\gamma \in (0, \frac{1}{L}]$ and $\beta \in [0, 1)$, let $\{\vx_k\}_{k \geq 0}$, $\{\vy_k\}_{k \geq 0}$, and $\{\vz_k\}_{k \geq 0}$ be the iterates computed by \sfgd{} with the warmup length $\warmstep$ and the learning rates $\{\gamma_k\}_{k \geq 0}$ satisfying \eqref{eqn:lr-assumption}, and the averaging rates $\{c_{k+1}\}_{k \geq 0}$ satisfying \eqref{eqn:main-song-avg}. 
    Then for any $T \geq 1$, letting $\warmratio = \frac{\warmstep-1}{T}$, it holds that \begin{equation} \label{eqn:sfgd-1t}
        \min_{k=0, \dots, T-1}\, \norm{\nabla f(\vy_k)}^2 \leq \frac{2(f(\vx_0) - f^*)}{\gamma (1-\beta) (1-\warmratio) T}.
    \end{equation}
\end{cor}

That is, \sfgd{} enjoys a convergence rate of $O(1/T)$ in the nonconvex smooth setting. 
Equivalently, for a given target accuracy $\epsilon > 0$, \sfgd{} can find a point $\vy$ such that $\norm{\nabla f(\vy)} \leq \epsilon$ within $O(\epsilon^{-2})$ gradient evaluations. 
\citet{Carm21} have shown that, when the objective function is nonconvex and only $L$-smooth, any first-order method must perform at least $\Omega(\epsilon^{-2})$ gradient evaluations to find a point $\vy$ such that $\norm{\nabla f(\vy)} \leq \epsilon$.
Hence, \eqref{eqn:sfgd-1t} implies that \sfgd{} attains the optimal convergence rate within the class of deterministic first-order methods.

For \sfsgd{} where $\sigma^2 > 0$, we should take into account the second and third terms on the right hand side of \eqref{eqn:master-bound}, as they are now nonzero.
Following the learning rate choice strategy of Theorem~1 in the original Schedule-Free paper~\citep{Defa24}, we take $\gamma \asymp {1}/{\sqrt{T}}$ to balance the $O(\frac{1}{\gamma T})$ and the $O(\gamma)$ terms in terms of their orders with respect to $T$, to obtain the following.  

\begin{cor}\label[cor]{cor:sfsgd-rate} 
    Suppose that \Cref{asmp:lower-bound,asmp:smooth,asmp:stochastic} hold. 
    Say we run \sfsgd{} for $T \geq 1$ iterations with $\beta \in [0, 1)$, the warmup length $\warmstep$ and the learning rates $\{\gamma_k\}_{k \geq 0}$ satisfying \eqref{eqn:lr-assumption}, and the averaging rates $\{c_{k+1}\}_{k \geq 0}$ satisfying \eqref{eqn:main-song-avg}, to obtain the iterates $\{\vx_k\}_{k\geq 0}$, $\{\vy_k\}_{k\geq 0}$, and $\{\vz_k\}_{k\geq 0}$.
    Then, choosing $\gamma = \min \{ \frac{1}{L}, \frac{1}{(1-\beta) \sqrt{T}}\}$, for $\warmratio = \frac{\warmstep-1}{T}$ and $A \coloneqq \max\{1, \frac{L}{1-\beta}\}$ it holds that \begin{equation} \label{eqn:sfsgd-rate}
        \min_{k=0, \dots, {T-1}}\, \expt\bigl[\norm{\nabla f(\vy_k)}^2 \bigr] \leq \frac{2 A(f(\vx_0) - f^*) +  L \sigma^2}{(1-\warmratio)\sqrt{T}} +  \frac{3\beta(2-\beta) C \sigma^2 (1+\log T)}{ (1-\warmratio) T}.
    \end{equation}
Moreover, if additionally $T \geq \frac{L^2}{(1-\beta)^2}$, then the convergence bound \eqref{eqn:sfsgd-rate} holds with $A = 1$.
\end{cor}

The proof of \Cref{cor:sfsgd-rate} is in \Cref{sec:sfsgd-rate}.
From \eqref{eqn:sfsgd-rate}, we deduce that \sfsgd{} achieves a convergence rate of $O(1/\sqrt{T})$. 
Equivalently, for any target accuracy $\epsilon>0$, \sfsgd{} finds a point $\vy$ satisfying $\expt{\norm{\nabla f(\vy)}} \leq \epsilon$ using $O(\epsilon^{-4})$ oracle queries.  
\citet{Arje23} have shown that, for nonconvex smooth objectives and stochastic gradient oracles returning unbiased estimators with bounded variance, $\Omega(\epsilon^{-4})$ queries are necessary in the worst case to find such a point.
Thus, \Cref{cor:sfsgd-rate} shows that \sfsgd{} is rate-optimal within this class of stochastic first-order methods.

\section{Schedule-Free Methods Avoid Strict Saddle Points} \label{sec:saddle}

It is now widely known that GD with a randomly chosen initial point avoids strict saddle points~\citep{Lee19}. 
A natural question is whether an analogous statement can be established for Schedule-Free methods. 
Indeed, under mild conditions where avoiding strict saddles implies convergence to minimizers, such a result would provide a principled reason to expect convergent Schedule-Free trajectories to approach minimizers rather than remain near saddle points.

To this end, we formulate Schedule-Free methods as a dynamical system. 
While the update rule for Schedule-Free methods is most commonly expressed in terms of three sequences $(\vx_k, \vy_k, \vz_k)$, one can use \eqref{eqn:yk} to eliminate one of them, yielding an equivalent \mbox{two-sequence} formulation.  
For instance, as we detail in \Cref{ssec:2seq}, as long as $\beta \in [0, 1)$, one can eliminate $\vz_k$ and obtain
\begin{subequations} \label{eqn:2SF}
\begin{align} 
    \vy_{k+1} &= \vy_{k} + \frac{\beta c_{k+1}}{1-\beta}(\vy_k - \vx_k) - \gamma_k (1-\beta+\beta c_{k+1})\vg_k, \label{eqn:2yk} \\
    \vx_{k+1} &= \vx_k + \frac{c_{k+1}}{1-\beta}(\vy_k - \vx_k) - \gamma_k c_{k+1} \vg_k. \label{eqn:2xk}
\end{align}
\end{subequations}
For \sfgd{}, where $\vg_k$ only depends on $\vy_k$, 
the right hand sides of both \eqref{eqn:2yk} and \eqref{eqn:2xk} are functions of $(\vy_k, \vx_k)$. 
So, letting $h_k$ denote the map from $(\vy_k,\vx_k)$ to the right hand side of \eqref{eqn:2SF}, \sfgd{} can be viewed as a \emph{nonautonomous} dynamical system $(\vy_{k+1},\vx_{k+1}) = h_k(\vy_k,\vx_k)$ in $\rr^d \times \rr^d$. 

A simple computation reveals that the fixed points of the dynamical system are precisely the points $(\vx^*,\vx^*)$ where $\vx^*$ is a stationary point of $f$. 

\begin{prop}\label[prop]{prop:fixed-points}
    For any integer $k \geq 0$, suppose that $\gamma_k \neq 0$ and $c_{k+1} \neq 0$. Then, a point $(\vy^*, \vx^*) \in \rr^d \times \rr^d$ is a fixed point of $h_k$ if and only if $\vy^* = \vx^*$ and $\nabla f(\vx^*) = \vzero$.
\end{prop}
In particular, notice that there are no fixed points at which only the gradient location coordinate is stationary while the evaluation coordinate remains nonstationary. 
Whenever convergence to an equilibrium occurs, the limiting values of $\vy_k$ and $\vx_k$ must coincide at a stationary point of $f$.

Recall that, for a twice continuously differentiable function $f$, a stationary point $\vx^*$ is called a strict saddle point if $\lambda_{\min}(\nabla^2 f(\vx^*))<0$.
For clarity of presentation, we temporarily assume for the remainder of this section that $\warmstep=1$, so that $\gamma_k=\gamma$ and $c_{k+1}=\frac{1}{k+1}$.
The general case is conceptually similar and is deferred to \Cref{sec:foo}.

\citet{Lee19} show that, for GD and related first-order methods run with suitably chosen constant learning rates, the set of initial points converging to strict saddles has Lebesgue measure zero.
An analogue of this result would yield a Lebesgue null exceptional set $\gW\subset\rr^d\times\rr^d$ that any trajectory $\{(\vy_k,\vx_k)\}_{k \geq 0}$ converging to a strict saddle point must hit. 
However, this is not enough for almost sure avoidance under the initialization of \sfgd. 
Since $c_1=1$, the update rules \eqref{eqn:xk} and \eqref{eqn:yk} force $\vy_1=\vx_1$. 
Thus, the distribution of $(\vy_1,\vx_1)$ is supported on $\Updelta=\{(\vx,\vx) : \vx\in\rr^d\}$, which is a Lebesgue measure zero subset of $\rr^d\times\rr^d$. 
Hence, while $\gW$ is a Lebesgue null set, it could still be hit with positive probability under this degenerate distribution, and a direct analogue of the almost sure avoidance conclusion of \citet{Lee19} does not follow.

To address this degeneracy, we introduce a one-time perturbation of the iterates, following the approach of \citet{Onei19} for proving that accelerated methods avoid strict saddle points.
After computing $\vz_1$ and $\vx_1$, we replace $\vz_1$ by $\vz_1 + \vxi$, where $\vxi$ is drawn from a distribution that is absolutely continuous with respect to Lebesgue measure.
This distribution may be chosen to have arbitrarily small variance.
Moreover, when the distribution has mean zero, the convergence rates obtained in \Cref{sec:conv} can be preserved, as we show in \Cref{thm:kick-rate-preserved}.
The absolute continuity of this perturbation yields the following avoidance result.

\begin{thm}\label{thm:one-time-perturb}
    \looseness=-1
    Consider \sfgd{} with a \emph{one-time perturbation} applied after $\vz_1$ and $\vx_1$ are computed, where $\vz_1$ is replaced by $\vz_1+\vxi$. 
    Suppose that the initial point $\vx_0=\vz_0$ and the perturbation $\vxi$ are drawn independently from distributions that are absolutely continuous with respect to Lebesgue measure on $\rr^d$. 
    Then, for $\beta \in [0, 1)$ and a twice continuously differentiable $L$-smooth $f$, as long as $\gamma < \frac{1}{L}$, the one-time perturbed \sfgd{} almost surely does not converge to any strict saddle point of $f$.
\end{thm}

Thus, the obstruction above is an artifact of the degeneracy caused by $c_1 = 1$, rather than evidence that the Schedule-Free iterates are attracted to strict saddle points.
Once this degeneracy is removed, \sfgd{} recovers the same qualitative strict-saddle avoidance behavior as GD.
Together with the small-variance and mean-zero observations above, this perturbation should be viewed as a negligible initialization-level modification of the method, rather than a substantive change to the algorithm.
Under the standard additional assumptions that turn avoidance of strict saddle points into convergence to minimizers, strict saddle points that are not minimizers are therefore excluded from asymptotic limits of the perturbed Schedule-Free dynamics.

\section{Further Discussions on the Evaluation Iterates} \label{sec:why-xk}

As we have discussed in \Cref{sec:intro}, Schedule-Free can be understood as an interpolation between \mbox{Polyak--Ruppert} averaging and primal averaging. 
In fact, Schedule-Free itself is also an averaging scheme, where the evaluation iterates $\vx_k$ can be interpreted as a weighted average of the iterates $\vy_k$. 
Indeed, \citet{Morw24} have observed, for all $k \geq 0$, that \begin{equation} \label{eqn:morwani}
\vx_{k+1} = \frac{(1-\beta)(1-c_{k+1}) \vx_k + c_{k+1} \vy_{k+1}}{(1-\beta)(1-c_{k+1}) + c_{k+1}}.
\end{equation} 

It is by now folklore that averaged iterates are beneficial, offering stabilized optimization dynamics, reduced sensitivity to checkpoint choice, and often improved robustness and generalization \citep{Neu18,Izma18,Mora24}.
Thus, if the iterates $\{\vy_k\}_{k \geq 0}$ inherit the nice convergence properties of the base optimizer, \`{a} la the results in \Cref{sec:conv}, this folklore could serve as an intuitive explanation on the superior performance of Schedule-Free methods.

In contrast, convergence guarantees for averaged iterates in nonconvex optimization are generally weak.
For primal averaging, corresponding to $\beta = 1$ in \sfsgd, \citet{Defa20} showed that the \emph{joint} minimum of $\|\nabla f(\vx_k)\|^2$ and $\|\nabla f(\vz_k)\|^2$ over $k = 0, 1, \dots, T-1$ is $O(\frac{1}{\gamma T}) + O(\gamma)$, matching our bound in \eqref{eqn:master-bound}. 
However, if the minimum of only $\|\nabla f (\vx_k)\|^2$ is considered as the optimality measure, then according to \citet{Brow25}, the convergence rate appears to be no better than $O(\frac{1}{\log T})$.
To the best of our knowledge, no comparable result is known for $\beta < 1$, and nonconvex SGD theory often avoids using averaged iterates as the output \citep{Khal23, Ghad13}.
Motivated by this gap, we use PEP to numerically compute worst-case upper bounds on $\min_{0 \leq k < T}\|\nabla f(\vx_k)\|^2$ for the iterates generated by \sfgd.
The results are organized and plotted in \Cref{fig:x05}.

\begin{figure}
    \centering 
    \includegraphics[width=0.5\textwidth]{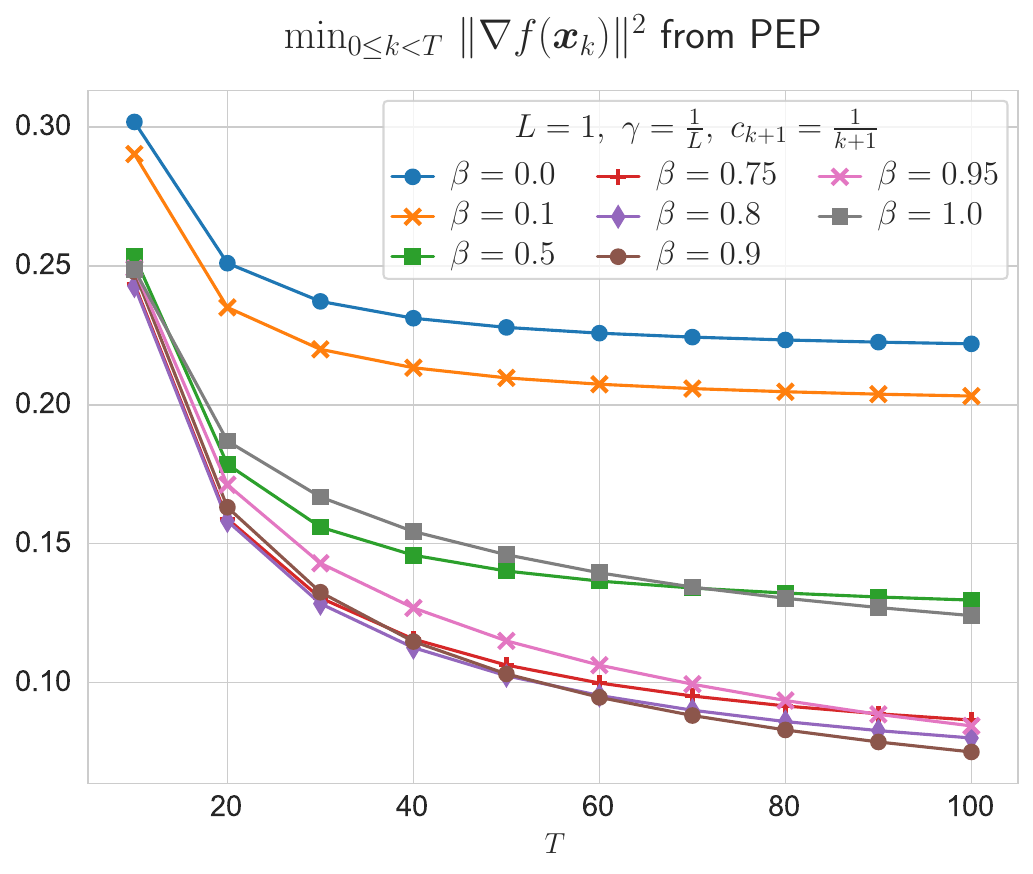}
    \caption{PEP results on the convergence of the evaluation iterates $\vx_k$ from \sfgd, with varying $\beta$.}
    \label{fig:x05}
\end{figure}

The PEP results first confirm that $\vx_k$ can be a pessimistic output sequence in the worst case.
When $\beta = 0$, the Polyak--Ruppert averages $\vx_k$ of the convergent iterates $\vy_k$ seem to fail to approach stationarity.
Also in the regime where $0 < \beta < 1$, the plotted results suggest that $\min_{0\leq k < T} \|\nabla f(\vx_k)\|^2$ does not achieve the ideal $O(1/T)$ worst-case rate that matches those we would have obtained if $\vy_k$ were considered instead.

Still, we can also observe an intriguing phenomenon that merits further discussion.
When $\beta$ is large enough and close to $1$, the certified worst-case bound from PEP appears to be \emph{strictly} improved compared to both endpoint cases $\beta = 0$ and $\beta = 1$.
At least up to $T = 100$, for values of $\beta$ from $0.75$ to $0.95$, the convergence plots lie consistently below the curve for $\beta = 1$. 
This suggests that Schedule-Free methods may improve over both \mbox{Polyak--Ruppert} averaging and primal averaging through an appropriate interpolation between them.
This observation is also consistent with the effectiveness of choosing $\beta \geq 0.9$ in practice.

\looseness=-1
Meanwhile, the pessimistic worst-case behavior of $\vx_k$ is not surprising in view of \eqref{eqn:morwani}, since $\{\vx_k\}_{k \geq 0}$ is an averaged sequence of $\{\vy_k\}_{k \geq 0}$. 
If $f$ were convex, then its desirable consequences such as Jensen's inequality would allow us to transfer the convergence of $\{\vy_k\}_{k \geq 0}$ into that of $\{\vx_k\}_{k \geq 0}$ more easily. 
But for general nonconvex objectives, we cannot expect such optimistic cases to be always true.

Nonetheless, empirical and theoretical evidence suggests that many modern learning problems exhibit \mbox{convex-like} regularity near or along optimization trajectories, including \mbox{PL\textsuperscript{\textasteriskcentered}-type} regularity~\citep{Liu22} and star convexity along optimization paths~\citep{Zhou19}.
Once a convergent Schedule-Free trajectory, as suggested by the discussion in \Cref{sec:saddle}, approaches a neighborhood of a local minimizer where such regularity holds, the averaging that defines $\vx_k$ is less likely to move the evaluation point outside the same well-behaved low-loss region traced by the gradient location iterates $\{\vy_k\}_{k \geq 0}$. 
In this sense, while strict-saddle avoidance does not by itself yield convergence guarantees for $\{\vx_k\}_{k \geq 0}$, it nonetheless offers a useful lens for understanding the empirical success of these averaged evaluation iterates despite the pessimistic worst-case PEP observations.

\section{Conclusion} \label{sec:conclusion}
In this paper, we studied Schedule-Free methods on smooth nonconvex objectives from the perspectives of worst-case convergence, strict saddle avoidance, and evaluation iterate behavior.
Guided by \sfode{}, the continuous-time limiting ODE of \sfgd{}, we developed Lyapunov analyses showing that the gradient location iterates $\vy_k$ of \sfgd{} and \sfsgd{} achieve the known worst-case optimal rates.
We also formulated \sfgd{} as a nonautonomous dynamical system and proved strict saddle avoidance under an arbitrarily small one-time perturbation.
Finally, we complemented these theoretical results with a PEP-based numerical study of the evaluation iterates $\vx_k$.

While these results clarify several theoretical aspects of Schedule-Free methods in the nonconvex setting, some natural questions remain open.
Although we establish worst-case rate-optimality, our constant factors may not yet be tight, and improving them would be an interesting direction for future work. 
Deriving analytical convergence guarantees for the evaluation iterates $\vx_k$ would also be a compelling direction for subsequent studies. 

\begin{ack}

This work was supported in part by the National Research Foundation of Korea (NRF) grants funded by the Korea government (MSIT) (Nos. RS-2019-NR040050 and RS-2025-00515159), and by the KAIST Grand Challenge 30 Project (KC30) funded by MSIT and KAIST (No. N11260046).
This research was also supported by the KAIST Jang Young Sil Fellow Program.
\end{ack}

\bibliographystyle{plainnat}
\bibliography{biblio.bib}

\newpage
\appendix
\onecolumn

\input{appx.tex}

\newpage
\input{checklist.tex}

\end{document}

%% file: appx.tex
\input{lemmata_bag}

\section{Further Related Work and Supplementary Remarks} \label{appx:rltwrk}

\subsection{On the Existing Nonconvex Smooth Convergence Analysis of Schedule-Free Methods} \label{appx:brown25}
Recently, \citet{Brow25} studied the convergence guarantees of \sfsgd{} when the objective function $f$ is smooth but possibly nonconvex. 
A notable aspect of their study is that they consider a wide range of averaging weights, including weights that decay to $0$ at various rates by taking $a > 0$ and setting $c_{k+1} = \frac{1}{(k+1)^a}$, as well as weights that do not even satisfy $c_{k+1} \to 0$, such as $c_{k+1} = \frac{k}{k+1}$. 

Yet their Theorem~1, which claims the convergence rates when $c_{k+1} = \frac{1}{(k+1)^a}$, relies on the existence of an integer $k^*$ such that the inequality \begin{equation}\label{eqn:faulty-ineq}
w_k \geq \frac{c_k}{2\eta}(1+2\eta L(1-\beta)) + \frac{\eta c_k L^2 (1-\beta)^2}{2} + \frac{L c_k^2}{2}(1+\eta)(1+2\eta L^2 (1-\beta)^2)
\end{equation} holds for all $k \geq k^*$, where $w_{k+1} = \frac{c_k}{2\eta(1-c_k)^2}$. 
However, it unfortunately turns out that such $k^*$ cannot exist if $\beta < 1$.  
Indeed, the left hand side of \eqref{eqn:faulty-ineq} is \begin{align*}
w_k = \frac{c_{k-1}}{2\eta (1-c_{k-1})^2} &= \frac{1}{2\eta}\cdot\frac{\frac{1}{(k-1)^a}}{(1-\frac{1}{(k-1)^a})^2} \\[1ex]
&= \frac{1}{2\eta}\cdot\frac{(k-1)^a}{((k-1)^a - 1)^2} = \frac{1}{2\eta k^a} + o\left( \frac{1}{k^a} \right),
\end{align*} while the right hand side is \[
\frac{1}{2\eta k^a}\left(1 + 2\eta L (1-\beta) + \eta^2 L^2 (1-\beta)^2\right) + O\left( \frac{1}{k^{2a}} \right),
\] where the big-$O$ and little-$o$ notations denote the asymptotic behavior when $k \to \infty$. Hence, it follows that \begin{align*}
    \MoveEqLeft w_k - \left(\frac{c_k}{2\eta}(1+2\eta L(1-\beta)) + \frac{\eta c_k L^2 (1-\beta)^2}{2} + \frac{L c_k^2}{2}(1+\eta)(1+2\eta L^2 (1-\beta)^2)\right) \\
    &= - \frac{ 2  L (1-\beta) + \eta  L^2 (1-\beta)^2}{2  k^a} + o\left( \frac{1}{k^a} \right).
\end{align*}
Therefore, for all sufficiently large $k$, the \emph{opposite} direction of the inequality \eqref{eqn:faulty-ineq} holds. 
In particular, there cannot exist an integer $k^*$ which makes \eqref{eqn:faulty-ineq} hold for all $k \geq k^*$.

\subsection{On the Existing Nonconvex Nonsmooth Convergence Analysis of Schedule-Free Methods}
Based on the framework which enables a conversion of the convergence guarantees for online optimization into those in (offline) nonconvex optimization, \citet{Ahn25} have shown that \mbox{\sfsgd} with a wisely chosen set of parameters achieves a worst-case optimal convergence rate in the nonconvex nonsmooth regime. 
Moreover, they show that under the additional assumption of smoothness of the objective function, the worst-case optimal convergence rate in that setting follows as a corollary of the nonsmooth worst-case optimal guarantee.

However, interpreting their convergence results requires care, due to some caveats. 
The exact method they analyze, which is Algorithm~5 (or equivalently Algorithm~7) in their paper, has the overall form that resembles the update rule of 
Schedule-Free methods, but there are some mismatches in the details. 
In particular, in the method they analyze, the parameter $\beta$ is not a fixed constant, but rather it is a random variable that is resampled at every iteration. 
In addition, while the original Schedule-Free scheme uses iteration-dependent averaging weights $c_{k+1}$, their convergence analysis replaces these with an iteration-independent constant. 
Finally, the convergence rates are with respect to the exponential moving average of $\vy_k$, instead of $\vy_k$ themselves or the already existing (but differently) averaged iterates $\vx_k$.  
Hence, while their results offer some insights regarding the performance of Schedule-Free methods, strictly speaking, their study is on a variant of a Schedule-Free method, not the Schedule-Free methods themselves.

\section{Dynamics of the \sfode}

In this section, we perform a Lyapunov analysis and study the asymptotic behavior of the solution trajectories of the \sfode.
For the clarity of the presentation, we temporarily defer to \Cref{appx:ode-study} the technical verifications showing that the singularity of the \sfode{} at $t = 0$ is \emph{removable}. 
Here, assuming that those details are already established, we proceed as if the \sfode{} has no singularities, and consequently, as if we can freely take the limit as $t \to 0+$ along the solution trajectories without worrying about whether they are well-defined or finite. 

\begin{prop} [\Cref{prop:Lyapunov}]
 Consider a function $\gL: (0, \infty) \to \rr$ defined as 
    \[
    \gL(t) = f(\vy(t)) - f^* + \frac{\beta }{2 t} \norm{\vz(t) - \vx(t)}^2.
    \] Then it holds that $\dot{\gL} \leq -(1-\beta) \norm{\nabla f (\vy)}^2 \leq 0$, and thus $\gL$ is a Lyapunov function of the \sfode. 
\end{prop}

\begin{proof}
It is immediate from \eqref{eqn:ode-y} that $\dot{\vy} = (1-\beta) \dot{\vz} + \beta \dot{\vx}$. 
    Differentiating $\gL$ with respect to $t$, we get \begin{align*}
        \dot{\gL} &= \inprod{\nabla f(\vy)}{\dot{\vy}} - \frac{\beta}{2 t^2} \norm{\vz - \vx}^2 + \frac\beta{t} \inprod{\dot{\vz} - \dot{\vx}}{\vz - \vx} \\ 
        &= \inprod{-\dot{\vz}}{(1-\beta) \dot{\vz} + \beta \dot{\vx}} - \frac{\beta}{2 t^2} \norm{\vz - \vx}^2 + \beta \inprod{\dot{\vz} - \dot{\vx}}{\dot{\vx}} \\ 
        &= -(1-\beta) \norm{\dot{\vz}}^2 - \beta \inprod{\dot{\vz}}{\dot{\vx}} - \frac{\beta}{2 t^2} \norm{\vz - \vx}^2 + \beta \inprod{\dot{\vz}}{\dot{\vx}} - \beta \norm{\dot{\vx}}^2 \\ 
        &= -(1-\beta) \norm{\dot{\vz}}^2 - \frac{3\beta}{2t^2}  \norm{\vz - \vx}^2. 
    \end{align*} 
    As $\dot{\vz} = -\nabla f(\vy)$ and $\beta \geq 0$, we indeed have $\dot{\gL} \leq -(1-\beta) \norm{\nabla f (\vy)}^2 \leq 0$.
\end{proof}

\begin{cor}[\Cref{cor:asymps-of-sfode}]
    For any $T > 0$, it holds that \[
    \min_{0 \leq t \leq T}\; \norm{\nabla f(\vy(t))}^2 \leq \frac{f(\vx_0) - f^*}{(1-\beta) T}.
    \]
\end{cor}
\begin{proof}
    By \Cref{prop:Lyapunov}, we have \[
    \gL(T) - \lim_{h \to 0+}\gL(h) = \int_0^T \dot{\gL}(t) \, \mathrm{d} t  \leq \int_0^T -(1-\beta) \norm{\dot{\vz}(t)}^2 \, \mathrm{d} t.
    \] By \Cref{asmp:lower-bound} we know that $\gL(T) \geq 0$, and by \Cref{prop:c1-extension} we have \begin{align*}
    \lim_{h \to 0+}\gL(h) &= f(\vx_0) - f^* + \frac{\beta}{2}  \lim_{h\to 0+} \inprod{\frac{\vz(h) - \vx(h)}{h}}{\vz(h) - \vx(h)} \\[1ex]
    &= f(\vx_0) - f^* + \frac{\beta}{2} \inprod{-\frac{\nabla f(\vx_0)}{2}}{\vzero} \\[1ex]
    &= f(\vx_0) - f^*.
    \end{align*} Therefore, it follows that \begin{align*}
        f(\vx_0) - f^* &= \lim_{h\to 0+} \gL (h) \\
        &\geq \gL(T) + (1-\beta) \int_0^T \norm{\dot{\vz}(t)}^2 \, \mathrm{d} t  \\
        &\geq (1-\beta) \int_0^T \norm{\nabla f(\vy(t))}^2 \, \mathrm{d} t \\
        &\geq (1-\beta) T \min_{0 \leq t \leq T}\; \norm{\nabla f(\vy(t))}^2.
    \end{align*} Dividing by $(1-\beta) T$ gives us the claimed bound. 
\end{proof}

\section{Further Discussions on the Parameter Choices} \label{appx:decr-avg-rates}

\citet{Defa24} proposed the choice of the averaging rates \begin{equation}\label{eqn:defazio-avg}
    c_{k+1} = \frac{\gamma_k^2}{\sum_{i=0}^k \gamma_i^2}
\end{equation} 
with $\gamma_k$ abiding by \Cref{asmp:lrn-rate}.
This is a generalization of the standard $c_{k+1} = \frac{1}{k+1}$, as when $\warmstep = 1$ we have $\gamma_k = \gamma$ for all $k \geq 0$, and indeed \eqref{eqn:defazio-avg} reduces back to  $c_{k+1} = \frac{1}{k+1}$.
Later, inspired by the observation that $\beta$ is overloaded in the sense that it determines the momentum applied to $\vy_k$ and at the same time the length of the implicit averaging window, introducing a \emph{decoupling parameter} $C > 0$ and setting \begin{equation}\label{eqn:song-avg}
    c_{k+1} = \min\left\{\frac{\gamma_k^2}{\sum_{i=0}^k \gamma_i^2} \cdot (1-\beta) C ,\, 1\right\}
\end{equation}
 has been proposed by \citet{Song25}. 
 This is exactly what \Cref{asmp:avg-rate} enforces. 
 Notice that \eqref{eqn:song-avg} reduces to \eqref{eqn:defazio-avg} when $C = \frac{1}{1-\beta}$.  
However, it is possible that $C > \frac{1}{1-\beta}$, so in \eqref{eqn:song-avg} an explicit clipping of the averaging rates is imposed, in order to maintain $0 \leq c_{k+1} \leq 1$.  

\subsection{Conditions for Monotonically Decreasing Averaging Rates}

For the sequence of learning rates $\{\gamma_k\}_{k \geq 0}$, let us define \begin{equation}\label{eqn:tilde-c}
    \tilde{c}_{k+1} = \frac{\gamma_k^2}{\sum_{i=0}^k \gamma_i^2}
    \end{equation} for each $k \geq 0$, so that the averaging rates, according to \Cref{asmp:avg-rate}, are set as \begin{equation}\label{eqn:tilde-c-to-c}
c_{k+1} = \min\{\tilde{c}_{k+1} (1-\beta) C ,\, 1\}.
\end{equation}  

The \emph{standard} choice of the averaging rates $c_{k+1} = \frac{1}{k+1}$ asserts $\{c_{k+1}\}_{k \geq 0}$ to be monotone decreasing. 
In fact, this monotone decreasing property of the averaging rates is guaranteed to hold as long as $c_{k+1}$ is set as \eqref{eqn:tilde-c-to-c} with a log-concave sequence of learning rates. 
Recall \Cref{defn:log-concave} for the definition of log-concavity.

\begin{prop} \label[prop]{prop:avg-rate-monotone}
    Suppose that the sequence of learning rates $\{\gamma_k\}_{k \geq 0}$ is log-concave. 
    Then the sequence of averaging rates, defined as \eqref{eqn:tilde-c-to-c}, is monotone decreasing. 
\end{prop}
\begin{proof}
    By \Cref{lem:log-concave-preserving}\ref{item:product}, $\{\gamma_k^2\}_{k \geq 0}$ is also log-concave,  so by \Cref{lem:log-concave-decr}, $\{\tilde{c}_{k+1}\}_{k \geq 0}$ is a monotone decreasing sequence. 
    As a consequence, for any $k \geq 1$, it holds that \[
        c_{k+1} = \min\{\tilde{c}_{k+1}(1-\beta) C, 1\} \leq \tilde{c}_{k+1}(1-\beta) C \leq \tilde{c}_k (1-\beta) C.
    \] In addition, we have $c_{k+1} \leq 1$, and thus \[
        c_{k+1} \leq \min\{\tilde{c}_k(1-\beta)C, 1\} = c_k.
    \] As this holds for all $k \geq 1$, it follows that $\{c_{k+1}\}_{k \geq 0}$ is monotone decreasing. 
\end{proof}

\subsection{Consequences of \titleref{Assumption}{asmp:lrn-rate}} 
Recall that \Cref{asmp:lrn-rate} specifies the choice of the learning rates to be \begin{equation} \label{eqn:asmp-4}
\gamma_k = \gamma \min\left\{1, \frac{k+1}{\warmstep}\right\}.
\end{equation}
An immediate observation one can make, yet useful in the forthcoming analyses, is that $\{\gamma_k\}_{k \geq 0}$ is a monotone increasing sequence. 
A slightly less obvious fact is that $\{\gamma_k\}_{k \geq 0}$ is log-concave. 
\begin{prop}\label[prop]{prop:asmp-4-is-logconcave}
    The sequence $\{\gamma_k\}_{k \geq 0}$ defined as in \eqref{eqn:asmp-4} is log-concave. 
\end{prop}
\begin{proof}
    By \Cref{lem:log-concave-preserving}\ref{item:constant} and \Cref{lem:log-concave-preserving}\ref{item:minimum}, it is immediate that $\{\min\{1, \frac{k+1}{\warmstep}\}\}_{k\geq 0}$ is log-concave. 
    In addition, \Cref{lem:log-concave-preserving}\ref{item:constant} tells us that the constant sequence $\{\gamma\}_{k \geq 0}$ is also log-concave. 
    Therefore, the sequence $\{\gamma_k\}_{k \geq 0}$ defined as in \eqref{eqn:asmp-4} is a pointwise product of two log-concave sequences, which we know from \Cref{lem:log-concave-preserving}\ref{item:product} that is also log-concave.
\end{proof}
The following is then an immediate corollary of the preceding observations. 
\begin{cor}\label[cor]{cor:avg-rate-monotone}
    Suppose that \Cref{asmp:avg-rate,asmp:lrn-rate} hold. 
    Then the sequence of averaging rates $\{c_{k+1}\}_{k \geq 0}$ is monotone decreasing. 
\end{cor}
\begin{proof}
    This is an immediate consequence of \Cref{prop:avg-rate-monotone,prop:asmp-4-is-logconcave}.
\end{proof}

Nonetheless, as \Cref{asmp:lrn-rate} fixes a specific explicit form of $\gamma_k$, under such an assumption it becomes possible to derive more quantitatively detailed statements. 
To begin with, as $\gamma_k = \frac{\gamma (k + 1)}{\warmstep}$ when $k \leq \warmstep - 1$, notice that \begin{align}
    \tilde{c}_{k+1} &= \frac{\gamma^2 (k+1)^2}{\warmstep^2} \cdot \frac{1}{\sum_{i=0}^k \frac{\gamma^2 (i+1)^2}{\warmstep^2}} \notag \\
    &= \frac{(k+1)^2}{\sum_{i=0}^k (i+1)^2} \notag \\
    &= \frac{6(k+1)}{(k+2)(2k+3)}.  \label{eqn:avg-rate-decay-rate-small}
\end{align} 
On the other hand, when $k \geq \warmstep - 1$, we have \begin{align} 
    \tilde{c}_{k+1} &= \gamma^2  \cdot \frac{1}{\sum_{i=0}^{\warmstep - 1} \frac{\gamma^2 (i+1)^2}{\warmstep^2} + \sum_{\warmstep}^k \gamma^2} \notag \\
    &= \frac{1}{\sum_{i=0}^{\warmstep - 1} \frac{(i+1)^2}{\warmstep^2} + k - \warmstep + 1} \notag \\
    &= \frac{1}{ \frac{(\warmstep+1)(2\warmstep+1)}{6\warmstep} + k - \warmstep + 1}. \label{eqn:avg-rate-decay-rate} 
\end{align}  
These formulae allow us to derive an upper bound of the averaging rates, as follows.

\begin{prop}
    Suppose that \Cref{asmp:avg-rate,asmp:lrn-rate} hold. 
    Then the averaging rates satisfy  
    \begin{equation} \label{eqn:avg-rate-upper-bound}
    c_{k+1} \leq 6(1-\beta) C \cdot \frac{(k+1)}{(k+2)(2k+3)}. 
    \end{equation}
\end{prop}
\begin{proof}
Consider the sequence $\{\tilde{c}_{k+1}\}_{k \geq 0}$ defined in \eqref{eqn:tilde-c}. 
We show a stronger inequality \begin{equation} \label{eqn:c-tilde-upper-bound}
\tilde{c}_{k+1} \leq \frac{6(k+1)}{(k+2)(2k+3)},
\end{equation} as then by \eqref{eqn:tilde-c-to-c} we would consequently have \eqref{eqn:avg-rate-upper-bound}.     
When $k \leq \warmstep - 1$, by \eqref{eqn:avg-rate-decay-rate-small} it is clear that \eqref{eqn:c-tilde-upper-bound} holds. 
Thus, it remains to show that \begin{equation} \label{eqn:avg-ub-tmp18}
    \frac{1}{ \frac{(\warmstep+1)(2\warmstep+1)}{6\warmstep} + k - \warmstep + 1} \leq \frac{6(k+1)}{(k+2)(2k+3)}
\end{equation} whenever $k > \warmstep - 1$, as the left hand side of \eqref{eqn:avg-ub-tmp18} is nothing but the explicit form of $\tilde{c}_{k+1}$ when $k \geq \warmstep - 1$, which we derived in \eqref{eqn:avg-rate-decay-rate}.

Let $j = k - \warmstep + 1$, so that \eqref{eqn:avg-ub-tmp18} is equivalent to \begin{equation} \label{eqn:avg-ub-tmp19}
    (j+\warmstep+1) (2j+2\warmstep+1) \leq 6(j+\warmstep)\left(j + \frac{(\warmstep+1)(2\warmstep+1)}{6\warmstep}\right).
\end{equation} Then, it suffices to show that this holds for all $j \geq 0$.
To this end, observe that \begin{align*}
     \MoveEqLeft 6(j+\warmstep)\left(j + \frac{(\warmstep+1)(2\warmstep+1)}{6\warmstep}\right) - (j+\warmstep+1) (2j+2\warmstep+1) \\
     &= (j+\warmstep)\left(6j + \frac{2\warmstep^2+3\warmstep+1}{\warmstep}\right) - (j+\warmstep) (2j+2\warmstep+1) - (2j+2\warmstep+1) \\
     &= (j+\warmstep)\left(4j + \frac{2\warmstep+1}{\warmstep}\right) - (2j+2\warmstep+1) \\
     &= 4j^2 + \left(4 \warmstep + 2 + \frac{1}{\warmstep}\right) j + 2\warmstep + 1 - (2j+2\warmstep+1) \\
     &= 4j^2 + \left(4 \warmstep + \frac{1}{\warmstep}\right) j .
\end{align*} From this it is immediate that \eqref{eqn:avg-ub-tmp19} holds for $j \geq 0$, and we are done.
\end{proof}

From \eqref{eqn:avg-rate-upper-bound} we get that $\lim_{k \to \infty} c_k = 0$. 
Hence, there exists some positive integer $k_0$ such that $k > k_0$ implies $c_{k+1} < 1$. 
Meanwhile, as $c_{1} = \min\{(1-\beta)C,\, 1\}$, if $C > \frac{1}{1-\beta}$ then $c_1 = 1$, and for the first few steps we will have $c_{k+1} = 1$. 
In that case, we may let $k_0$ be the largest integer such that $c_{k+1} = 1$ whenever $k \leq k_0$. 
But then, as we have shown in \Cref{lem:fallback}, we will have $\vx_k = \vz_k = \vy_k$ for all $k = 1, \dots, k_0+1$. 
In addition, from the update rule \eqref{eqn:SF} and the initialization $\vx_0 = \vz_0$, it follows that $\vy_0 = \vx_0$. 
Thus, using \eqref{eqn:zk}, we conclude that $\vy_{k+1} = \vy_k - \gamma_k \vg_k$ for all $k = 0, 1, \dots, k_0 + 1$. 
This means that the iterates that are generated for first $k_0 + 1$ steps are exactly the same as what we would get by running the base optimizer, and the Schedule-Free modification is essentially \emph{inactive} during that phase.  
Let us formally state this observation.  

\begin{prop} \label[prop]{prop:inactive-phase}
\renewcommand{\arraystretch}{1.2} %
\setlength{\arraycolsep}{0pt} %
    Suppose that \Cref{asmp:avg-rate,asmp:lrn-rate} hold. Then \begin{equation}\label{eqn:k-thres}  
    k_0 = \left\{
        \begin{array}{l@{\qquad}l}
        -1 & C < \dfrac{1}{1-\beta} \\[1em]
        \left\lfloor (1-\beta)C +\dfrac{4\warmstep^2-9\warmstep-1}{6\warmstep} \right\rfloor & C \geq \dfrac{(\warmstep + 1)(2\warmstep + 1)}{6(1-\beta)\warmstep} \\[1em]
        \left\lfloor \dfrac{-7+6(1-\beta)C +\sqrt{1-36(1-\beta)C + 36(1-\beta)^2 C^2}}{4} \right\rfloor & \text{otherwise}
        \end{array}
        \right. 
    \end{equation} is the integer such that $c_{k+1} = 1$ if and only if $k \leq k_0$. 
    Consequently, for all $k = 0, 1, \dots, k_0+1$ it holds that $\vx_k = \vz_k = \vy_k$ and $\vy_{k+1} = \vy_k - \gamma_k \vg_k$.  
\end{prop} 
\begin{proof}
    As previously discussed, the second half of the statement would follow from the first half. 
    Hence, it suffices to show that $c_{k+1} = 1$ if and only if $k \leq k_0$ for $k_0$ defined in \eqref{eqn:k-thres}.  
    \Cref{cor:avg-rate-monotone} ensures that $\{c_{k+1}\}_{k \geq 0}$ is monotone decreasing, so this is equivalent to $k_0$ being the largest integer $k$ such that $c_{k+1} = 1$.
    Because $c_1 = \min\{(1-\beta)C,\, 1\} < 1$ if and only if $C < \frac{1}{1-\beta}$, we immediately get that $k_0 = -1$ if $C < \frac{1}{1-\beta}$. 
    Now, assume that $C \geq \frac{1}{1-\beta}$, so that $c_1 = 1$, and in particular, there exists at least one nonnegative integer $k$ such that $c_{k+1} = 1$.
    Observe that \begin{equation} \label{eqn:appx-15}
        \tilde{c}_{\warmstep}  (1-\beta) C = \frac{6(1-\beta)C \warmstep}{(\warmstep+1)(2\warmstep+1)}.
    \end{equation} If the right hand side of \eqref{eqn:appx-15} is less than $1$, then $c_{\warmstep} < 1$ so $k_0 \leq \warmstep - 1$. 
    In particular, by \eqref{eqn:avg-rate-decay-rate-small}, $k_0$ is the largest integer such that \[
        \frac{6(1-\beta)C (k_0+1)}{(k_0+2)(2k_0+3)} \geq 1.
    \] 
    By simple algebra involving solving a quadratic inequality, we obtain that in this case \[
        k_0 = \left\lfloor \frac{-7+6(1-\beta)C +\sqrt{1-36(1-\beta)C + 36(1-\beta)^2 C^2}}{4} \right\rfloor.
    \] 
    On the other hand, if the right hand side of \eqref{eqn:appx-15} is greater than or equal to $1$, then $c_{\warmstep} = 1$ so $k_0 > \warmstep - 1$. 
    By \eqref{eqn:avg-rate-decay-rate}, $k_0$ is the largest integer such that \[
        \frac{(1-\beta) C}{ \frac{(\warmstep+1)(2\warmstep+1)}{6\warmstep} + k_0 - \warmstep + 1} \geq 1.
    \] 
    Again by simple algebra, with noting that the denominator of the left hand side is positive, we obtain \[
        k_0 = \left\lfloor (1-\beta)C +\frac{4\warmstep^2-9\warmstep-1}{6\warmstep} \right\rfloor 
    \] for this case.
\end{proof}

\Cref{prop:inactive-phase} tells us that, if $(1-\beta)C \geq 1$, then for roughly a constant multiple of $(1-\beta) C$ steps the Schedule-Free method will behave exactly the same as the base optimizer. 
Based on this observation, one can make a conjecture on why introducing this decoupling parameter $C$ should improve the empirical performance.  
In general, the base sequence $\{\vz_k\}_{k \geq 0}$ is a fast-moving sequence that directly receives the gradient updates as in \eqref{eqn:zk}, whereas the evaluation sequence $\{\vx_k\}_{k \geq 0}$ evolves more conservatively, being averaged iterates of $\{\vz_k\}_{k \geq 0}$.
But if $(1-\beta) C > 1$, then the ``fast'' variable~$\vz_k$ and the ``slow'' variable~$\vx_k$ can start to differ only after $k_0+1$ steps.  
Moreover, as the points $\vz_0, \dots, \vz_{k_0+1}$ are exactly those that would be computed by the base optimizer initialized at $\vz_0$, one can expect that the base optimizer would have driven both $\vz_{k_0+1}$ and $\vx_{k_0+1}$ down to a point with a lower loss that is in a more locally well-behaved region of the loss landscape than the randomly chosen initial point.
After that, it becomes less likely for $\vz_k$ to explore exotic regions with high loss or irregular loss landscapes, and its average $\vx_k$ would inherit those nice properties. 
We leave the verifications of this heuristic explanation as future directions of work.

\section{Lyapunov Descent Lemma for Schedule-Free Methods} \label{appx:lyap-descent}

The main result of this section is establishing the proof of \Cref{prop:appx-lyapunov-descent}, which we call the \emph{Lyapunov descent lemma}. 
\Cref{prop:lyapunov-descent} then follows from \Cref{prop:appx-lyapunov-descent}, as both share the same statements but the latter has a weaker set of assumptions which implies that of the former. 

In \Cref{prop:appx-lyapunov-descent}, we construct Lyapunov potentials for \sfsgd{} and show that they form a monotone decreasing sequence, much like in the celebrated \emph{descent lemma}, up to an additive noise term arising from the use of stochastic gradients.
The proof of it, though, involves verifying several technical inequalities.
So, to keep the main argument clear, we first establish those ancillary results before turning to the main proof itself.
To this end, we define some auxiliary sequences, whose motivations behind the definitions will soon be clear.  

In this section, we work with the learning rates and the averaging rates satisfying conditions that are substantially weaker than \Cref{asmp:avg-rate,asmp:lrn-rate}. 
We assume that the learning rates $\{\gamma_k\}_{k \geq 0}$ form a monotone increasing sequence of positive real numbers with $\gamma_k \leq \frac{1}{L}$ for all $k \geq 0$, and the averaging rates $\{c_{k+1}\}_{k \geq 0}$ are monotonically decreasing with $0 \leq c_{k+1} \leq 1$ for all $k \geq 0$ but not constantly $1$. 
Under these assumptions, we set $k_0$ to be the largest nonnegative integer such that $c_{k+1} = 1$ whenever $k \leq k_0$ if such $k$ exists, and otherwise $k_0 = -1$. 
In other words, for the learning rates we assume that \begin{equation} \label{eqn:incr-lr}
 0 < \gamma_0 \leq \gamma_1 \leq \gamma_2 \leq \dots \leq \frac{1}{L}
\end{equation} and for the averaging rates we assume that 
\begin{equation} \label{eqn:decr-ck}
    1 = c_1 = c_2 = \dots = c_{k_0+1} > c_{k_0+2} \geq c_{k_0+3} \geq \dots \geq 0.
\end{equation} Notice that \eqref{eqn:incr-lr} is consistent with \Cref{asmp:lrn-rate}, and \eqref{eqn:decr-ck} holds as a consequence of \Cref{asmp:lrn-rate,asmp:avg-rate} by the observations we made in \Cref{appx:decr-avg-rates}.

Now, let us define a sequence $\{\alpha_k\}_{k > k_0}$ as \begin{equation} \label{eqn:lyap-coeffs}
    \alpha_k = \frac{\beta c_k (1-\gamma_{k-1} L (1 - \beta + \beta c_k))}{2\gamma_{k-1}(1-c_{k})^2}
\end{equation} for all $k = k_0+2, k_0+3, \dots$, and \begin{equation} \label{eqn:35}
    \alpha_{k_0+1} = \frac{L \beta^2 c_{k_0+2}^2}{2} + \alpha_{k_0+2} (1-c_{k_0+2})^2.
\end{equation}
For $k > k_0$, we then define \begin{align}
    D_k &\coloneqq     \frac{L\beta^2 c_{k+1}^2}{2} + \alpha_{k+1} (1 - c_{k+1})^2 - \alpha_k, \label{eqn:Dk}\\
        G_k &\coloneqq     \gamma_k^2 \alpha_{k+1} (1 - c_{k+1})^2 + \frac{ \gamma_k^2  L (1 - \beta + \beta c_{k+1})^2}{2} - \gamma_k (1 - \beta + \beta c_{k+1}), \label{eqn:Gk} \\
        S_k &\coloneqq     \gamma_k^2 \left(\alpha_{k+1} (1 - c_{k+1})^2 + \frac{ L (1 - \beta + \beta c_{k+1})^2}{2}\right). \label{eqn:Sk}
\end{align}

\begin{lem}\label[lem]{lem:nonnegative-lyap-coeff}
    Suppose that $\beta \in [0, 1)$. Then $\alpha_k \geq 0$ for all $k > k_0$.
\end{lem}
\begin{proof}
    It is clear that $\alpha_{k_0+1} \geq 0$. For $k \geq k_0+2$, as $\frac{\beta c_k}{2\gamma_{k-1} (1-c_k)^2} \geq 0$, it suffices to show that $1-\gamma_{k-1} L (1 - \beta + \beta c_k) \geq 0$. 
    Indeed, as $0 < 1 - c_k \leq 1$, it holds that \begin{align*}
         1 - \beta (1 - c_k) \leq 1 \leq \frac{1}{\gamma_{k-1} L},
    \end{align*} from which it follows that $1-\gamma_{k-1} L (1 - \beta + \beta c_k)$ is nonnegative.
\end{proof}

\begin{lem} \label[lem]{lem:Gk-bound}
    Suppose that $\beta \in [0, 1)$. Then $G_k \leq -\frac{\gamma_k (1-\beta)}{2}$ for all $k > k_0$. 
\end{lem}
\begin{proof}
    Noting that $G_k = S_k- \gamma_k(1-\beta + \beta c_{k+1})$, let us first simplify the expression of $S_k$. 
    From the definition of $\alpha_{k+1}$, we get \begin{align*}
         S_k &= \gamma_k^2 \left(\alpha_{k+1} (1 - c_{k+1})^2 + \frac{ L (1 - \beta + \beta c_{k+1})^2}{2}\right) \\
         &= \gamma_k^2 \left(  \frac{\beta c_{k+1} (1-\gamma_k L (1 - \beta + \beta c_{k+1}))}{2\gamma_k} + \frac{ L (1 - \beta + \beta c_{k+1})^2}{2}\right) \\
         &= \frac{\gamma_k}{2}  \left(  \beta c_{k+1} (1-\gamma_k L (1 - \beta + \beta c_{k+1})) + \gamma_k L (1 - \beta + \beta c_{k+1})^2 \right) \\
         &= \frac{\gamma_k}{2}  \left(  \beta c_{k+1} (1 - \gamma_k L + \gamma_k L \beta - \gamma_k L \beta c_{k+1}) + \gamma_k L \left((1-\beta)^2 + 2  \beta (1-\beta) c_{k+1} +  \beta^2 c_{k+1}^2 \right)\right) \\
         &= \frac{\gamma_k}{2}  \left(  \beta c_{k+1} + \gamma_k L (1-\beta)^2 + \gamma_k L \beta (1-\beta) c_{k+1} \right).
    \end{align*} Hence, with the inequalities $\gamma_k L \leq 1$ and $0 \leq 1 - \beta \leq 1$, it follows that \begin{align*}
        G_k &= S_k- \gamma_k(1-\beta + \beta c_{k+1}) \\
        &= \frac{\gamma_k}{2}  \left(  \beta c_{k+1} + \gamma_k L (1-\beta)^2 + \gamma_k L \beta (1-\beta) c_{k+1} -2 (1-\beta + \beta c_{k+1}) \right) \\
        &= \frac{\gamma_k}{2}  \left( (1-\beta) (\gamma_k L (1-\beta) - 2) +  \beta c_{k+1} (\gamma_k L (1-\beta)  - 1) \right) \\
        &\leq \frac{\gamma_k (1-\beta) }{2}  (\gamma_k L (1-\beta) - 2) \\
        &\leq  -\frac{\gamma_k (1-\beta) }{2} 
    \end{align*}
    as claimed.
\end{proof}

\begin{lem} \label[lem]{lem:Dk-negative}
    Suppose that $\beta \in [0, 1)$. Then for all $k > k_0$, as long as $0 \leq c_{k+1} \leq c_k < 1$ and $0 < \gamma_{k-1} \leq \gamma_k \leq \frac{1}{L}$, it holds that $D_k \leq 0$.
\end{lem}
\begin{proof}
    By the definition of $\alpha_{k_0+1}$, it is clear that $D_{k_0+1} = 0$. Meanwhile, for $k \geq k_0+2$, observe that \begin{align*}
        \alpha_k &= \frac{\beta c_k (1-\gamma_{k-1} L (1 - \beta + \beta c_k))}{2\gamma_{k-1} (1-c_k)^2} \\
        &= \frac{\beta c_k}{2(1-c_k)^2} \left(\frac{1}{\gamma_{k-1}} - L (1 - \beta + \beta c_k)\right) \\
        &\geq \frac{\beta c_k}{2(1-c_k)^2} \left(\frac{1}{\gamma_{k}} - L (1 - \beta + \beta c_k)\right) \\
        &= \frac{\beta c_k (1-\gamma_k L (1 - \beta + \beta c_k))}{2\gamma_k (1-c_k)^2}.
    \end{align*} Hence, it follows that \begin{align*}
        D_k &= \frac{L\beta^2 c_{k+1}^2}{2} + \alpha_{k+1} (1 - c_{k+1})^2 - \alpha_k \\
        &\leq \frac{L\beta^2 c_{k+1}^2}{2} + \frac{\beta c_{k+1} (1-\gamma_k L (1 - \beta + \beta c_{k+1}))}{2\gamma_k} - \frac{\beta c_k (1-\gamma_k L (1 - \beta + \beta c_k))}{2\gamma_k(1-c_k)^2} \\ 
        &= \frac{\beta}{2\gamma_k}\left( \gamma_k L\beta  c_{k+1}^2 + c_{k+1} (1-\gamma_k L +\gamma_k L \beta -\gamma_k L \beta c_{k+1}) - \frac{c_k (1-\gamma_k L +\gamma_k L \beta -\gamma_k L \beta c_k)}{(1-c_k)^2} \right) \\
        &= \frac{\beta}{2\gamma_k}\left(  c_{k+1} (1-\gamma_k L +\gamma_k L \beta) - \frac{c_k (1-\gamma_k L) + \gamma_k L \beta c_k (1 - c_k)}{(1-c_k)^2} \right) \\
        &= \frac{\beta}{2\gamma_k}\left(  (1-\gamma_k L) \left(c_{k+1} - \frac{c_k}{(1-c_k)^2}\right) +  \gamma_k L \beta \left( c_{k+1} - \frac{c_k}{(1 - c_k)}\right) \right).
    \end{align*} As $\frac{1}{1-c_k} \geq 1$, we have \[
    c_{k+1} \leq c_k \leq \frac{c_k}{1-c_k} \leq \frac{c_k}{(1-c_k)^2},
    \] from which it is immediate that $D_k \leq 0$.
\end{proof}

Notice that \Cref{lem:nonnegative-lyap-coeff,lem:Gk-bound} only relies on that $0 \leq c_{k+1} < 1$.
On the other hand, \Cref{lem:Dk-negative} does require the monotone decreasing property of $\{c_k\}$, although it would suffice to have a slightly weaker condition \[
c_{k+1} \leq \frac{c_k}{1 - c_k}.
\]

Nonetheless, with these observations, we are now ready to prove the following \emph{Lyapunov descent lemma}.

\begin{lem} \label[lem]{prop:appx-lyapunov-descent}
Suppose that \Cref{asmp:lower-bound,asmp:smooth,asmp:stochastic} hold. 
For $\beta \in [0, 1)$,
let $\{\vx_k\}_{k\geq 0}$, $\{\vy_k\}_{k\geq 0}$, and $\{\vz_k\}_{k\geq 0}$ be the iterates computed by \sfsgd{} with
the learning rates $\{\gamma_k\}_{k \geq 0}$ satisfying \eqref{eqn:incr-lr}, and the averaging rates $\{c_{k+1}\}_{k \geq 0}$ satisfying \eqref{eqn:decr-ck} for some $k_0 \geq -1$.  
Define a sequence of nonnegative real numbers $\{\alpha_k\}_{k\geq 0}$ as \[ 
    \alpha_k = \frac{\beta c_k (1-\gamma_{k-1} L (1 - \beta + \beta c_k))}{2\gamma_{k-1}(1-c_{k})^2}
\] for all $k = k_0+2, k_0+3, \dots$ and \[
    \alpha_0 = \alpha_1 = \dots = \alpha_{k_0+1} = \frac{L \beta^2 c_{k_0+2}^2}{2} + \alpha_{k_0+2} (1-c_{k_0+2})^2.
\] Then, the Lyapunov potentials defined as \begin{equation}\label{eqn:lyapunov-potentials}
        V_k = f(\vy_k) - f^* + \alpha_k  \norm{\vz_k - \vx_k}^2
\end{equation} satisfy the inequality \begin{equation} \label{eqn:lyapunov-descent} 
    \expt[V_{k+1}] - \expt[V_k] \leq  
    - \frac{\gamma_k(1-\beta)}{2} \expt \bigl[\norm{\nabla f(\vy_k)}^2 \bigr] 
    + \frac{\gamma_k \sigma^2}{2} \bigl(\beta (1 + \gamma_k L (1-\beta)) c_{k+1} + \gamma_k L (1-\beta)^2 \bigr)
\end{equation}
for all nonnegative integers $k \geq k_0 + 1$, and \begin{equation} \label{eqn:simple-descent}
    \expt[V_{k+1}] - \expt[V_k] \leq 
    - \frac{\gamma_k }{2} \expt \bigl[\norm{\nabla f(\vy_k)}^2 \bigr]  
    + \frac{\gamma_k^2 L \sigma^2}{2} 
\end{equation}
for all nonnegative integers $k \leq k_0$.
\end{lem}

\begin{proof}
Let us first consider when $k \geq k_0 + 1$.
    By \Cref{lem:smooth,lem:y-diff,lem:zx-diff}, we have \begin{align*}
        \MoveEqLeft V_{k+1} - V_k \\
        &= f(\vy_{k+1}) - f(\vy_k) + \alpha_{k+1}\norm{\vz_{k+1} - \vx_{k+1}}^2 - \alpha_k \norm{\vz_k - \vx_k}^2 \\[5pt] 
        &\leq \inprod{\vy_{k+1} - \vy_k}{\nabla f(\vy_k)} + \frac{L}{2} \norm{\vy_{k+1}- \vy_k}^2 + \alpha_{k+1}\norm{\vz_{k+1} - \vx_{k+1}}^2 - \alpha_k \norm{\vz_k - \vx_k}^2 \\[5pt]
        &= \inprod{ \beta c_{k+1} (\vz_k - \vx_k) - \gamma_k (1 - \beta + \beta c_{k+1}) \vg_k}{\nabla f(\vy_k)} \\*
            &\phantom{=} \qquad + \frac{L}{2} \norm{\beta c_{k+1} (\vz_k - \vx_k) - \gamma_k (1 - \beta + \beta c_{k+1}) \vg_k}^2 \\*
            &\phantom{=} \qquad + \alpha_{k+1} (1 - c_{k+1})^2 \norm{\vz_k - \vx_k - \gamma_k \vg_k}^2 - \alpha_k \norm{\vz_k - \vx_k}^2  \\[5pt]
        &=  \beta c_{k+1} \inprod{\vz_k - \vx_k}{\nabla f(\vy_k)} - \gamma_k (1 - \beta + \beta c_{k+1})  \inprod{ \vg_k}{\nabla f(\vy_k)} \\* 
            &\phantom{=} \qquad + \frac{L\beta^2 c_{k+1}^2}{2} \norm{ \vz_k - \vx_k}^2  + \frac{\gamma_k^2 L (1 - \beta + \beta c_{k+1})^2}{2} \norm{ \vg_k}^2 \\*
            &\phantom{=} \qquad -  \gamma_k L \beta  c_{k+1} (1 - \beta + \beta c_{k+1}) \inprod{ \vz_k - \vx_k} {\vg_k} - 2\gamma_k  \alpha_{k+1} (1 - c_{k+1})^2 \inprod{\vz_k - \vx_k}{\vg_k} \\*
            &\phantom{=} \qquad + \alpha_{k+1} (1 - c_{k+1})^2 \norm{\vz_k - \vx_k }^2  + \gamma_k^2 \alpha_{k+1} (1 - c_{k+1})^2 \norm{ \vg_k}^2 - \alpha_k \norm{\vz_k - \vx_k}^2 \\[5pt] 
        &= \left(\frac{L\beta^2 c_{k+1}^2}{2} + \alpha_{k+1} (1 - c_{k+1})^2 - \alpha_k \right) \norm{\vz_k - \vx_k}^2 \\*
            &\phantom{=} \qquad + \beta c_{k+1} \inprod{\vz_k - \vx_k}{\nabla f(\vy_k)} \\*
            &\phantom{=} \qquad - \bigl( \gamma_k L \beta  c_{k+1} (1 - \beta + \beta c_{k+1}) + 2\gamma_k  \alpha_{k+1} (1 - c_{k+1})^2 \bigr) \inprod{\vz_k - \vx_k}{\vg_k} \\*
            &\phantom{=} \qquad + \gamma_k^2 \left( \alpha_{k+1} (1 - c_{k+1})^2 + \frac{L (1 - \beta + \beta c_{k+1})^2}{2} \right) \norm{\vg_k}^2 \\*
            &\phantom{=} \qquad - \gamma_k (1 - \beta + \beta c_{k+1})  \inprod{ \vg_k}{\nabla f(\vy_k)}.
    \end{align*}
    Let us take the conditional expectation conditioned on $\vz_k$ and $\vx_k$. 
    Then, by \Cref{asmp:stochastic} and its consequence \Cref{lem:noise-bound}, we obtain
    \begin{align*}
        \MoveEqLeft[1.5] \expt[V_{k+1} | \vz_k, \vx_k] - V_k \\
        &\leq \left(\frac{L\beta^2 c_{k+1}^2}{2} + \alpha_{k+1} (1 - c_{k+1})^2 - \alpha_k \right) \norm{\vz_k - \vx_k}^2 \\*
            &\phantom{=} \qquad + \bigl( \beta c_{k+1}  -  \gamma_k L \beta  c_{k+1} (1 - \beta + \beta c_{k+1}) - 2\gamma_k  \alpha_{k+1} (1 - c_{k+1})^2 \bigr)\inprod{\vz_k - \vx_k}{\nabla f(\vy_k)} \\*
            &\phantom{=} \qquad +\left(  \gamma_k^2 \alpha_{k+1} (1 - c_{k+1})^2 + \frac{ \gamma_k^2  L (1 - \beta + \beta c_{k+1})^2}{2} - \gamma_k (1 - \beta + \beta c_{k+1}) \right) \norm{\nabla f(\vy_k)}^2 \\* 
            &\phantom{=} \qquad + \gamma_k^2 \left( \alpha_{k+1} (1 - c_{k+1})^2 + \frac{L (1 - \beta + \beta c_{k+1})^2}{2} \right) \sigma^2.
    \end{align*}
    By the definition \eqref{eqn:lyap-coeffs} of $\alpha_{k+1}$, the coefficient of the inner product term $\inprod{\vz_k - \vx_k}{\nabla f(\vy_k)}$ is \[
        \beta c_{k+1}  -  \gamma_k L \beta  c_{k+1} (1 - \beta + \beta c_{k+1}) - 2\gamma_k  \alpha_{k+1} (1 - c_{k+1})^2 = 0.
    \] Hence, such a choice of $\alpha_k$ gives us the inequality \begin{equation} \label{eqn:pre-bound-lyapunov-descent}
        \expt[V_{k+1} | \vz_k, \vx_k] - V_k \leq D_k \norm{\vz_k - \vx_k}^2  + G_k \norm{\nabla f(\vy_k)}^2 + S_k \sigma^2
    \end{equation} for $D_k$, $G_k$, and $S_k$ defined exactly in \eqref{eqn:Dk}, \eqref{eqn:Gk}, and \eqref{eqn:Sk}, respectively. 
    In \Cref{lem:Dk-negative} we have shown that $D_k \leq 0$, and in \Cref{lem:Gk-bound} we have shown that $G_k \leq -\frac{\gamma_k (1-\beta)}2$.
    Moreover, continuing from the identity regarding $S_k$ established in the proof of \Cref{lem:Gk-bound}, we can observe that \begin{align*}
    S_k &= \frac{\gamma_k}{2}  \left(  \beta c_{k+1} + \gamma_k L (1-\beta)^2 + \gamma_k L \beta (1-\beta) c_{k+1} \right) \\ 
     &= \frac{\gamma_k}{2}  \left(  \beta (1 + \gamma_k L (1-\beta)) c_{k+1} + \gamma_k L (1-\beta)^2 \right).
    \end{align*}
    Applying these results to \eqref{eqn:pre-bound-lyapunov-descent}, we obtain \[
    \expt[V_{k+1} | \vz_k, \vx_k] - V_k \leq -\frac{\gamma_k(1-\beta)}{2} \norm{\nabla f(\vy_k)}^2 + \frac{\gamma_k \sigma^2}{2} \bigl(\beta (1 + \gamma_k L (1-\beta)) c_{k+1} + \gamma_k L (1-\beta)^2 \bigr).
    \] Taking the expectation with respect to $\vz_k$ and $\vx_k$, by the law of total expectation, we conclude that the claimed inequality \eqref{eqn:lyapunov-descent} holds.

    It remains to consider when $k \leq k_0$. 
    By \Cref{lem:fallback}, we have $\vx_k = \vz_k$ and $\vx_{k+1} = \vz_{k+1}$, and thus \[
    V_{k+1} - V_k = f(\vy_{k+1}) - f(\vy_k).
    \] Then, as \Cref{lem:y-diff} asserts that $\vy_{k+1} - \vy_k = -\gamma_k \vg_k$, the standard arguments for the classical descent lemma apply in obtaining \eqref{eqn:simple-descent}.
    That is, by \Cref{lem:smooth}, we have \begin{align*}
        V_{k+1} - V_k &\leq \inprod{\vy_{k+1} - \vy_k}{\nabla f(\vy_k)} + \frac{L}{2}\norm{\vy_{k+1} - \vy_k}^2 \\
                  &= -\gamma_k \inprod{\vg_k}{\nabla f(\vy_k)} + \frac{\gamma_k^2 L}{2} \norm{\vg_k}^2.
    \end{align*} Taking the conditional expectation with respect to $\vy_k$ and applying \Cref{lem:noise-bound}, we obtain \begin{equation} \label{eqn:initial-descent}
        \expt[V_{k+1} | \vy_k ] - \expt[V_k | \vy_k ] \leq \left(-\gamma_k + \frac{\gamma_k^2 L}{2}\right)\norm{\nabla f(\vy_k)}^2 + \frac{\gamma_k^2 L}{2} \sigma^2.
    \end{equation}
    Here, as $\gamma_k \leq \frac{1}{L}$, it holds that $-\gamma_k + \frac{\gamma_k^2 L}2 \leq -\frac{\gamma_k}2$, so we conclude that \eqref{eqn:initial-descent} implies \eqref{eqn:simple-descent}.  
\end{proof}

\section{Convergence Rate of \sfsgd{}} \label{sec:sfsgd-rate}

\begin{thm}[\Cref{thm:master-convergence-thm}]
    Suppose that \Cref{asmp:lower-bound,asmp:smooth,asmp:stochastic} hold. 
    For $\gamma \in (0, \frac{1}{L}]$ and $\beta \in [0, 1)$, let $\{\vx_k\}_{k\geq 0}$, $\{\vy_k\}_{k\geq 0}$, and $\{\vz_k\}_{k\geq 0}$ be the iterates computed by \sfsgd{} with
    the warmup length $\warmstep$ and the learning rates $\{\gamma_k\}_{k \geq 0}$ satisfying~\eqref{eqn:lr-assumption}, and the averaging rates $\{c_{k+1}\}_{k \geq 0}$ satisfying~\eqref{eqn:main-song-avg}.
    Then for any $T \geq 1$, letting $\warmratio = \frac{\warmstep-1}{T}$, we have the bound  \[
        \min_{k=0, \dots, T-1} \expt \bigl[\norm{\nabla f(\vy_k)}^2 \bigr] \leq \frac{2(f(\vx_0) - f^*)}{\gamma (1-\beta) (1-\warmratio)T} + \frac{\gamma L (1-\beta) \sigma^2}{1-\warmratio} + \frac{3\beta(2-\beta) C \sigma^2 (1+\log T)}{ (1-\warmratio) T}. 
    \]
\end{thm}
\begin{proof}
For notational convenience, for each $k \geq 0$, let us define \[
    (a_k, b_k) \coloneqq \begin{cases}
        \left(\dfrac{\gamma_k}2, \dfrac{\gamma_k^2L \sigma^2}{2}\right) & \quad k \leq k_0 \\[1em]
        \left(\dfrac{\gamma_k (1-\beta)}{2} , \dfrac{\gamma_k \sigma^2}{2} \bigl(\beta (1 + \gamma_k L (1-\beta)) c_{k+1} + \gamma_k L (1-\beta)^2 \bigr) \right)& \quad k \geq k_0 + 1
    \end{cases}
\] so that, by the Lyapunov descent lemma (\Cref{prop:appx-lyapunov-descent}), the Lyapunov potentials defined in \eqref{eqn:lyapunov-potentials} satisfy \[
    a_k \expt \bigl[\norm{\nabla f(\vy_k)}^2 \bigr] \leq \expt[V_k] - \expt[V_{k+1}] + b_k.
\]
    By telescopic summation, we obtain \begin{align*}
 \MoveEqLeft \sum_{k=0}^{T-1} a_k \expt \bigl[\norm{\nabla f(\vy_k)}^2 \bigr] \\ 
 &\leq V_0 - \expt[V_T] + \sum_{k=0}^{T-1} b_k \\ 
 &= V_0 - \expt[V_T] + \sum_{k=0}^{k_0} \frac{\gamma_k^2L \sigma^2}{2} + \sum_{k=k_0+1}^{T-1} \frac{\gamma_k \sigma^2}{2} \bigl(\beta (1 + \gamma_k L (1-\beta)) c_{k+1} + \gamma_k L (1-\beta)^2 \bigr)  \\ 
&\leq V_0 - \expt[V_T] + \frac{L \sigma^2}{2} \sum_{k=0}^{k_0} \gamma_k^2  + \sum_{k=k_0+1}^{T-1} \frac{\gamma_k \sigma^2}{2} \bigl(\beta (2-\beta) c_{k+1} + \gamma_k L (1-\beta)^2 \bigr)  \\ 
&= V_0 - \expt[V_T] + \frac{L \sigma^2}{2} \sum_{k=0}^{k_0} \gamma_k^2  + \frac{\beta (2-\beta) \sigma^2}{2} \sum_{k=k_0+1}^{T-1} \gamma_k c_{k+1} + \frac{L (1-\beta)^2\sigma^2}{2} \sum_{k=k_0+1}^{T-1}  \gamma_k^2  \\ 
&= V_0 - \expt[V_T] + \frac{L ((1-\beta)^2 + \beta(2-\beta))\sigma^2}{2} \sum_{k=0}^{k_0} \gamma_k^2  \\ 
&\phantom{= V_0 - \expt[V_T]} \qquad  + \frac{\beta (2-\beta) \sigma^2}{2} \sum_{k=k_0+1}^{T-1} \gamma_k c_{k+1} + \frac{L (1-\beta)^2\sigma^2}{2} \sum_{k=k_0+1}^{T-1}  \gamma_k^2  \\ 
&= V_0 - \expt[V_T] + \frac{\beta (2-\beta) \sigma^2}{2} \left(\sum_{k=0}^{k_0} \gamma_k^2 L +  \sum_{k=k_0+1}^{T-1} \gamma_k c_{k+1} \right) + \frac{L (1-\beta)^2\sigma^2}{2} \sum_{k=0}^{T-1}  \gamma_k^2  .
    \end{align*} 
Up to this point, we only used that the learning rates $\{\gamma_k\}_{k \geq 0}$ satisfy \eqref{eqn:incr-lr} and the averaging rates $\{c_{k+1}\}_{k \geq 0}$ satisfy \eqref{eqn:decr-ck}. 
But from here, in order to derive explicit convergence bounds, we invoke \Cref{asmp:lrn-rate,,asmp:avg-rate}.

Consider the sequence $\{\tilde{c}_{k+1}\}_{k \geq 0}$ as in \eqref{eqn:tilde-c}, and recall that $k_0$ is the largest integer such that $c_{k+1} = 1$ for all $k \leq k_0$.
Hence, from \eqref{eqn:tilde-c-to-c} which tells us that $c_{k+1} = \min\{\tilde{c}_{k+1} (1-\beta) C ,\, 1\}$, we get that if $k > k_0$ then $c_{k+1} = \tilde{c}_{k+1}(1-\beta) C$, and on the other hand, if $k \leq k_0$ then $\tilde{c}_{k+1}(1-\beta) C \geq 1$ as otherwise we cannot have $c_{k+1} = 1$. 
This implies that $\tilde{c}_{k+1} (1-\beta) C \geq \gamma_k L$ whenever $k \leq k_0$, so we obtain that \begin{align*}
\sum_{k=0}^{k_0} \gamma_k^2 L +  \sum_{k=k_0+1}^{T-1} \gamma_k c_{k+1} &\leq \sum_{k=0}^{T-1} \gamma_k \tilde{c}_{k+1}(1-\beta) C \\
&\leq \sum_{k=0}^{T-1} \frac{6\gamma_k (1-\beta) C (k+1)}{(k+2)(2k+3)} \\
&\leq \gamma (1-\beta) C \sum_{k=0}^{T-1} \frac{6(k+1)}{(k+2)(2k+3)} \\
&\leq \gamma (1-\beta) C \sum_{k=0}^{T-1} \frac{3}{k+1} \\
&\leq 3\gamma (1-\beta) C (1+\log T)
\end{align*}
     where the second line is by the upper bound of $\tilde{c}_{k+1}$ we obtained in \eqref{eqn:c-tilde-upper-bound}, the third line is by \Cref{asmp:lrn-rate}, the fourth line is by a simple inequality $(k+1)(2k+2) \leq (k+2)(2k+3)$ which holds for all nonnegative integers $k \geq 0$, and the fifth line is by \Cref{lem:harmonic-bound}.
    
    Meanwhile, as the initial points are chosen so that $\vz_0 = \vx_0$, we have $\vy_0 = (1-\beta) \vz_0 + \beta \vx_0 = \vx_0$. 
    This implies that \[
    V_0 = f(\vy_0)- f^* + \alpha_0\norm{\vz_0 - \vx_0}^2  = f(\vx_0) - f^*.
    \] As \Cref{asmp:lower-bound} and \Cref{lem:nonnegative-lyap-coeff} together assert that $V_T \geq 0$, we furthermore have \[
    V_0 - \expt[V_T] \leq f(\vx_0) - f^*.
    \] 
    Hence, using these bounds and once again applying \Cref{asmp:lrn-rate} of asserting $\gamma_k \leq \gamma$, from the telescopic sum we get \begin{equation} \label{eqn:master-bound-prototype}
     \sum_{k=0}^{T-1} a_k \expt \bigl[\norm{\nabla f(\vy_k)}^2 \bigr]  \leq f(\vx_0) - f^* + \frac{3\gamma \beta(1-\beta)(2-\beta) C \sigma^2 (1+\log T)}{2} +\frac{\gamma^2 L(1-\beta)^2\sigma^2 T}{2}. 
    \end{equation}

    We now seek a lower bound of the left hand side of \eqref{eqn:master-bound-prototype}. 
    By the definition of $a_k$, for all $k \geq 0$ we have $a_k \geq \frac{\gamma_k(1-\beta)}2 \geq 0$. 
    Also, by \Cref{asmp:lrn-rate} we have $\gamma_k = \gamma$ for all $k \geq \warmstep - 1 = \warmratio T$.
    It thus holds that \begin{align*}
        \sum_{k=0}^{T-1} a_k \expt \bigl[\norm{\nabla f(\vy_k)}^2 \bigr] &\geq \sum_{k=0}^{T-1} \frac{\gamma_k(1-\beta)}2\expt \bigl[\norm{\nabla f(\vy_k)}^2 \bigr]  \\
        &\geq \sum_{k=\warmstep-1}^{T-1} \frac{\gamma (1-\beta)}2\expt \bigl[\norm{\nabla f(\vy_k)}^2 \bigr] \\
        &\geq \frac{\gamma (1-\beta) (T - \warmstep + 1)}2 \min_{k = \warmstep-1, \warmstep, \dots, T-1}\expt \bigl[\norm{\nabla f(\vy_k)}^2 \bigr] \\
        &\geq \frac{\gamma (1-\beta) (1-\warmratio) T}2 \min_{k = \warmstep-1, \warmstep, \dots, T-1}\expt \bigl[\norm{\nabla f(\vy_k)}^2 \bigr] .
    \end{align*}  
    Plugging this lower bound back into \eqref{eqn:master-bound-prototype} and dividing both sides by $\frac{\gamma (1-\beta) (1-\warmratio) T}2$, we obtain  \[
    \min_{k = \warmstep-1, \dots, T-1} \expt \bigl[\norm{\nabla f(\vy_k)}^2 \bigr] \leq \frac{2(f(\vx_0) - f^*)}{\gamma (1-\beta) (1-\warmratio)T} + \frac{\gamma L (1-\beta) \sigma^2}{1-\warmratio} + \frac{3\beta(2-\beta) C \sigma^2 (1+\log T)}{ (1-\warmratio) T}.
    \] The claimed bound now follows from the simple fact that \[
        \min_{k = \warmstep-1, \warmstep, \dots, T-1}\expt \bigl[\norm{\nabla f(\vy_k)}^2 \bigr] \geq  \min_{k = 0, \dots, T-1}\expt \bigl[\norm{\nabla f(\vy_k)}^2 \bigr],
    \] as the minimum decreases when it is taken over a larger set. 
\end{proof}

\begin{cor}[\Cref{cor:sfsgd-rate}]
    Suppose that \Cref{asmp:lower-bound,asmp:smooth,asmp:stochastic} hold. 
    Say we run \sfsgd{} for $T \geq 1$ iterations with $\beta \in [0, 1)$, the warmup length $\warmstep$ and the learning rates $\{\gamma_k\}_{k \geq 0}$ satisfying \eqref{eqn:lr-assumption}, and the averaging rates $\{c_{k+1}\}_{k \geq 0}$ satisfying \eqref{eqn:main-song-avg}, to obtain the iterates $\{\vx_k\}_{k\geq 0}$, $\{\vy_k\}_{k\geq 0}$, and $\{\vz_k\}_{k\geq 0}$.
    Then, choosing $\gamma = \min \{ \frac{1}{L}, \frac{1}{(1-\beta) \sqrt{T}}\}$, for $\warmratio = \frac{\warmstep-1}{T}$ and $A \coloneqq \max\{1, \frac{L}{1-\beta}\}$ it holds that \begin{equation} \label{eqn:appx-sfsgd-rate}
    \min_{k=0, \dots, {T-1}}\, \expt\bigl[\norm{\nabla f(\vy_k)}^2 \bigr] \leq \frac{2 A(f(\vx_0) - f^*) +  L \sigma^2}{(1-\warmratio)\sqrt{T}} +  \frac{3\beta(2-\beta) C \sigma^2 (1+\log T)}{ (1-\warmratio) T} .
\end{equation}
Moreover, if additionally $T \geq \frac{L^2}{(1-\beta)^2}$, then the convergence bound \eqref{eqn:appx-sfsgd-rate} holds with $A = 1$.
\end{cor}
\begin{proof}
    The value of $\gamma$ depends on the choice of $T$. 
    Accordingly, we divide into two cases. 
    \begin{enumerate} [label=(\roman*)]
        \item $\frac{1}{(1-\beta)\sqrt{T}} \leq \frac{1}{L}$ 
        
        In this case we would have $\gamma = \frac{1}{(1-\beta) \sqrt{T}}$.
        Substituting this value of $\gamma$ into \eqref{eqn:master-bound} we get \begin{equation} \label{eqn:large-T-bound}
            \min_{k=0, \dots, T-1}\, \expt\bigl[\norm{\nabla f(\vy_k)}^2 \bigr] \leq \frac{2(f(\vx_0) - f^*) +  L \sigma^2}{(1-\warmratio)\sqrt{T}} +\frac{3\beta(2-\beta) C \sigma^2 (1+\log T)}{ (1-\warmratio) T}.
        \end{equation}
        \item $\frac{1}{(1-\beta)\sqrt{T}} > \frac{1}{L}$ 

        In this case we have $\gamma = \frac{1}{L}$. 
        Consider the three terms in the right hand side of \eqref{eqn:master-bound}.
        The third term of the right hand side of \eqref{eqn:master-bound} does not depend on $\gamma$, hence remains as is. 
        Meanwhile, as $\gamma = \frac{1}L$, the first term can be bounded as \[ 
            \frac{2(f(\vx_0) - f^*)}{\gamma (1-\beta) (1-\warmratio)T} = \frac{2L(f(\vx_0) - f^*)}{(1-\beta) (1-\warmratio)T} \leq \frac{2L(f(\vx_0) - f^*)}{(1-\beta) (1-\warmratio)\sqrt{T}} 
            .\]
        Moreover, for the second term, as $\gamma = \frac{1}{L} <\frac{1}{(1-\beta)\sqrt{T}}$ we get that \[
            \frac{\gamma L (1-\beta) \sigma^2}{1-\warmratio} \leq \frac{L \sigma^2 }{(1-\warmratio)\sqrt{T}}.
        \] Thus, in total, we have the bound \[
            \min_{k=0, \dots, T-1}\, \expt\bigl[\norm{\nabla f(\vy_k)}^2 \bigr] \leq \frac{\frac{2L}{1-\beta}(f(\vx_0) - f^*) +  L \sigma^2}{(1-\warmratio)\sqrt{T}} +  \frac{3\beta(2-\beta) C \sigma^2 (1+\log T)}{ (1-\warmratio) T}.
        \]
    \end{enumerate} 
    Hence, for $A \coloneqq \max\{1, \frac{L}{1-\beta}\}$, we see that in both cases the bound \[
        \min_{k=0, \dots, T-1}\, \expt\bigl[\norm{\nabla f(\vy_k)}^2 \bigr] \leq \frac{2 A(f(\vx_0) - f^*) +  L \sigma^2}{(1-\warmratio)\sqrt{T}} +  \frac{3\beta(2-\beta) C \sigma^2 (1+\log T)}{ (1-\warmratio) T}.
    \] holds. Moreover, if $T \geq \frac{L^2}{(1-\beta)^2}$ then only case (i) is valid between the two cases above. 
    Hence in that case, we conclude that the bound \eqref{eqn:large-T-bound} holds, which is exactly \eqref{eqn:appx-sfsgd-rate} with $A = 1$.
\end{proof}

\section{\sfgd{} as a Nonautonomous Dynamical System} \label{sec:foo}

\subsection{Two-sequence reformulation of the Schedule-Free Update Rule} \label{ssec:2seq}

Assuming $\beta \in [0, 1)$, here we show how the three-sequence update rule \eqref{eqn:SF} can be reformulated into a two-sequence form \eqref{eqn:2SF}. 
First, we use \eqref{eqn:zk} to eliminate $\vz_{k+1}$ from \eqref{eqn:xk} which leads to \[
\vx_{k+1} = (1-c_{k+1}) \vx_k + c_{k+1}(\vz_k - \gamma_k \vg_k).
\]
As \eqref{eqn:yk} implies that $\vz_k = \frac{1}{1-\beta}(\vy_k - \beta \vx_k)$, we substitute this into the above, and get  \begin{align*}
    \vx_{k+1} &= (1-c_{k+1}) \vx_k + c_{k+1}\left(\frac{\vy_k - \beta \vx_k}{1-\beta} - \gamma_k \vg_k\right) \\
              &= \vx_k + c_{k+1}\left(\frac{\vy_k - \beta \vx_k}{1-\beta} - \vx_k -  \gamma_k \vg_k\right) \\
              &= \vx_k + \frac{c_{k+1}}{1-\beta} (\vy_k - \vx_k) -   \gamma_k c_{k+1} \vg_k 
\end{align*} which is exactly \eqref{eqn:2xk}. Meanwhile, similarly eliminating $\vz_{k+1}$ and $\vz_k$ from \eqref{eqn:zk}, we obtain \[
\frac{\vy_{k+1} - \beta \vx_{k+1}}{1-\beta} = \frac{\vy_{k} - \beta \vx_{k}}{1-\beta} - \gamma_k \vg_k . 
\] Multiplying by $1-\beta$ and substituting \eqref{eqn:2xk} gives us \begin{align*}
    \vy_{k+1} &= \vy_{k} + \beta (\vx_{k+1} - \vx_{k}) - \gamma_k (1-\beta)\vg_k \\
              &= \vy_{k} + \frac{\beta c_{k+1}}{1-\beta} (\vy_k - \vx_k) - \gamma_k (1-\beta+\beta c_{k+1})\vg_k
\end{align*} which is exactly \eqref{eqn:2yk}.  

\subsection{\sfgd{} as a Nonautonomous Dynamical System}

As explained in \Cref{sec:saddle} and derived in \Cref{ssec:2seq}, given that $\beta \in [0,1)$, \sfgd{} can be viewed as a nonautonomous dynamical system with maps $h_k:\rr^d \times \rr^d \to \rr^d \times \rr^d$ defined as \begin{equation} \label{eqn:appx-nonauto-system}
h_k(\vy, \vx) = \begin{bmatrix} \vy + \frac{\beta c_{k+1}}{1-\beta}(\vy - \vx) - \gamma_k (1-\beta+\beta c_{k+1})\nabla f(\vy) \\[3pt] 
                                \vx + \frac{c_{k+1}}{1-\beta}(\vy - \vx) - \gamma_k c_{k+1} \nabla f(\vy) \end{bmatrix}, \quad k = 0, 1, \dots.
\end{equation}
For convenience, we will refer to the sequence of maps \eqref{eqn:appx-nonauto-system} simply as the nonautonomous dynamical system corresponding to \sfgd.

\begin{prop}[\Cref{prop:fixed-points}]
    For any integer $k \geq 0$, suppose that $\gamma_k \neq 0$ and $c_{k+1} \neq 0$. Then, a point $(\vy^*, \vx^*) \in \rr^d \times \rr^d$ is a fixed point of $h_k$ if and only if $\vy^* = \vx^*$ and $\nabla f(\vx^*) = \vzero$.
\end{prop}
\begin{proof}
    The ``if'' direction is straightforward from the definition of $h_k$.
    To show the reverse direction, suppose that $(\vy^*, \vx^*) \in \rr^d \times \rr^d$ is a fixed point of $h_k$.
    Then, from \eqref{eqn:appx-nonauto-system}, it is immediate that $(\vy^*, \vx^*)$ should satisfy \begin{equation} \label{eqn:mv-system}
        \begin{bmatrix}
            \beta \id & (1-\beta + \beta c_{k+1}) \id \\[5pt]
            \id       & c_{k+1} \id
        \end{bmatrix} \begin{bmatrix}
            \frac{c_{k+1}}{1-\beta} (\vy^* - \vx^*) \\[5pt]
            -\gamma_k \nabla f(\vy^*)
        \end{bmatrix} = \begin{bmatrix}
            \vzero \\[5pt] \vzero
        \end{bmatrix}  
    \end{equation} where $\id$ denotes the $d\times d$ identity matrix. 
    As it holds that \begin{equation} \label{eqn:ul-decomposition}
        \begin{bmatrix}
            \beta \id & (1-\beta + \beta c_{k+1}) \id \\[5pt]
            \id       & c_{k+1} \id
        \end{bmatrix} = \begin{bmatrix}
            \id    & \frac{1-\beta +\beta c_{k+1}}{c_{k+1}} \id \\[5pt]
            \vzero & \id
        \end{bmatrix} \begin{bmatrix}
            -\frac{1-\beta}{c_{k+1}} \id & \vzero \\[5pt] 
            \id & {c_{k+1}} \id
        \end{bmatrix}
    \end{equation} where the two factors on the right hand side of \eqref{eqn:ul-decomposition} are invertible triangular matrices because all of their diagonal entries are nonzero, \eqref{eqn:mv-system} holds if and only if $\frac{c_{k+1}}{1-\beta} (\vy^* - \vx^*) =  -\gamma_k \nabla f(\vy^*) = \vzero$.
    But $c_{k+1} \neq 0$ and $\gamma_k \neq 0$, so we conclude that $\vy^* = \vx^*$ and $\nabla f(\vx^*) = \vzero$.
\end{proof}

In particular, because all $h_k$ share the same set of fixed points, it is clear that \[
\{(\vx^*, \vx^*) : \nabla f(\vx^*) = \vzero\}
\] is the set of fixed points of the nonautonomous dynamical system.

Let us now state some general notions regarding a nonautonomous dynamical system, introduced by \citet{Musa26}.
Following the standard notation, for $r \geq 0$, we denote by $B_r(\vzero)$ the open ball of radius $r$ centered at $\vzero$.

\begin{defn}[\citep{Musa26}, Definition~2.4] \label[defn]{defn:PH}
    Let $h: \rr^m \to \rr^m$ and $T : \rr^m \to \rr^m$ be maps on $\rr^m$, with $T$ linear. 
    Let $\mu$, $\lambda$, and $\epsilon$ be scalars such that $1 \leq \lambda < \mu$ and $\epsilon > 0$.
    Consider a direct sum decomposition $\rr^m = \Ecs \oplus \Eu$ with the norms $\norm{\cdot}_\text{cs}$ and $\norm{\cdot}_\text{u}$ given on the subspaces $\Ecs$ and $\Eu$, respectively. 
    We say that the pair $(h, T)$ is \emph{$(\mu, \lambda, \epsilon)$-pseudo-hyperbolic} on an open neighborhood $U$ of $\vzero$ with respect to the splitting $\rr^m = \Ecs \oplus \Eu$, and denote by \[
        (h, T) \in \PsHy(\mu, \lambda, \epsilon; U, \Ecs \oplus \Eu)
    \] if the following conditions hold: \begin{enumerate}[label=(\roman*)]
        \item $h(\vzero) = \vzero$. 
        \item If $\vu \in \Ecs$ then $T(\vu) \in \Ecs$ with $\norm{T(\vu)}_{\text{cs}} \leq \lambda \norm{\vu}_{\text{cs}}$, and if $\vv \in \Eu$ then $T(\vv) \in \Eu$ with $\norm{T(\vv)}_{\text{u}} \geq \mu \norm{\vv}_{\text{u}}$. 
        \item \label{item:almost-linear} $\Lip((h-T) \vert_U) \leq \epsilon < \frac{\mu -\lambda}4$, where $\Lip(\,\cdot\,\vert_U)$ denotes the Lipschitz constant when restricted to $U$, with respect to the max-norm on $\rr^m$ defined as \begin{equation} \label{eqn:max-norm} 
            \|\vw\|_{\text{max}} \coloneqq \max\{\|\vu\|_{\text{cs}}, \|\vv\|_{\text{u}}\} 
        \end{equation} for $\vw = \vu + \vv$ with $\vu \in \Ecs$ and $\vv \in \Eu$. 
    \end{enumerate}
\end{defn}

\begin{defn}[\citep{Musa26}, Definition~2.8] \label[defn]{defn:NPH}
    Suppose that the origin $\vzero$ is a fixed point of a nonautonomous dynamical system given by the maps $h_k :\rr^m \to \rr^m$, $k = 0, 1, \dots$. 
    Then, $\vzero$ is said to be a \emph{non-uniformly pseudo-hyperbolic (NPH) unstable fixed point} of the system if there exists a splitting $\rr^m = \Ecs \oplus \Eu$ with $\dim(\Eu) \geq 1$, a sequence of linear maps $T_k : \rr^m \to \rr^m$, $k=0,1, \dots$, a sequence of positive scalars $\{(\mu_k, \lambda_k, \epsilon_k)\}_{k \geq 0}$, a positive scalar $r > 0$, and a nonnegative integer $K_0 \geq 0$ such that \begin{enumerate}[label=(\roman*)]
        \item \label{item:NPH-i} $(h_k, T_k) \in \PsHy(\mu_k, \lambda_k, \frac{\epsilon_k}{4} ; B_r(\vzero), \Ecs \oplus \Eu)$ for all $k \geq K_0$,
        \item $\epsilon_k < \frac{\mu_k - \lambda_k}{4}$ for all $k \geq K_0$,
        \item \label{item:NPH-iii} $\sum_{k=K_0}^\infty \frac{\epsilon_k}{\mu_k - 2\epsilon_k} = \infty$.
    \end{enumerate}
\end{defn}

As is standard in dynamical systems, the definitions are stated for the case in which the origin is the fixed point, although they apply more generally. 
More precisely, if $\vw^*$ is a common fixed point of $h_0, h_1, \dots$, then the translated maps $\tilde{h}_k = h_k(\,\cdot\, + \vw^*) - \vw^*$ have $\vzero$ as a common fixed point.

Intuitively, the NPH unstable fixed points are a nonautonomous analogue of the ordinary unstable fixed points for autonomous dynamical systems. 
Indeed, as we discuss shortly, for the iterates of a nonautonomous dynamical system to remain within a neighborhood of an NPH unstable fixed point, they must all lie on a certain measure-zero set.
From this, one can show that NPH unstable fixed points are almost surely avoided by the dynamical system.

\citet{Lee19} showed that GD and several of its variants, which can be represented as autonomous dynamical systems, avoid strict saddle points of the objective function, in the sense that the set of initial points converging to any such point has Lebesgue measure zero.
In light of this, it is natural to ask whether a similar conclusion can be drawn for strict saddle points of $f$ when considering the nonautonomous dynamical system corresponding to \sfgd{}. 
In order to answer this question, let us consider \begin{equation} \label{eqn:fixed-pts}
\gQ^* \coloneqq \{(\vx^*, \vx^*) : \nabla f(\vx^*) = \vzero \text{ and } \lambda_{\min}(\nabla^2 f(\vx^*)) < 0\}
\end{equation} which is the set of fixed points of the system \eqref{eqn:appx-nonauto-system} that are associated with the strict saddle points of $f$.

One can show that the points in $\gQ^*$ are NPH unstable, under mild conditions on the learning rates and the averaging rates, as follows.
As the proof is quite long and consists mainly of technical computations, we defer it to \Cref{ssec:proving-saddles-are-NPH}.

\begin{lem}\label[lem]{lem:saddles-are-NPH}
    Let $f$ be an $L$-smooth and twice continuously differentiable function, and let $\vx^*$ be a strict saddle point of $f$, so that $(\vx^*,\vx^*) \in \gQ^*$ where $\gQ^*$ is defined as in \eqref{eqn:fixed-pts}.
    Suppose that there exist $\gamma_{\min}$ and $\gamma_{\max}$ such that $0 < \gamma_{\min} \leq \gamma_k \leq \gamma_{\max} < \frac{1}{L}$ for all $k$, and the averaging rates $\{c_{k+1}\}_{k \geq 0}$ satisfy $\lim_{k \to \infty} c_{k+1} = 0$.  
    Then, $(\vx^*, \vx^*)$ is an NPH unstable fixed point of the nonautonomous system \eqref{eqn:appx-nonauto-system}.
\end{lem} 

As we briefly noted above, NPH unstable points are akin to ordinary unstable points of autonomous dynamical systems, in that the iterates can stay close to such a point only if they all lie in a certain measure zero set.
A formal statement of this fact is as follows.

\begin{thm}[\citep{Musa26}, Theorem~2.14]\label{musat-theorem}
    Suppose that each map $h_k : \rr^m \to \rr^m$ in the nonautonomous dynamical system is continuously differentiable, and that its Jacobian $Dh_k$ has full rank except on a set of measure zero. 
    Then, the set of initial points whose trajectories converge to any NPH unstable fixed point has measure zero.
\end{thm}

To apply \Cref{musat-theorem}, one has to verify that $D h_k$ is of full rank almost surely for $h_k$ in \eqref{eqn:appx-nonauto-system}. 
As we show below, it in fact turns out that $D h_k$ is invertible everywhere, given that $c_{k+1}<1$. 

\begin{prop} \label[prop]{prop:hk-diff}
    For any $k \geq 0$, consider the map $h_k : \rr^d \times \rr^d \to \rr^d \times \rr^d$ defined in \eqref{eqn:appx-nonauto-system}.
    Suppose that $f$ is $L$-smooth.
    If $c_{k+1} < 1$ and $\gamma_k < \frac{1}{(1-\beta)L}$, then $D h_k$ is invertible. 
\end{prop}
\begin{proof}
    Fix any $k \geq 0$. 
    By a direct computation, we get \begin{align*}  
        D h_k(\vy, \vx) &= \begin{bmatrix} (1+\frac{\beta c_{k+1}}{1-\beta})\id - \gamma_k (1-\beta + \beta c_{k+1})\nabla^2 f(\vy)  & -\frac{\beta c_{k+1}}{1-\beta} \id \\[1ex] 
            \frac{c_{k+1}}{1-\beta} \id - \gamma_k c_{k+1} \nabla^2 f(\vy) & (1-\frac{c_{k+1}}{1-\beta}) \id  
        \end{bmatrix} \\[1ex] 
        &= \begin{bmatrix} (1-\beta + \beta c_{k+1})(\frac{1}{1-\beta}\id - \gamma_k \nabla^2 f(\vy))  & -\frac{\beta c_{k+1}}{1-\beta} \id \\[1ex] 
            c_{k+1} (\frac{1}{1-\beta} \id - \gamma_k \nabla^2 f(\vy)) & \frac{1-\beta - c_{k+1}}{1-\beta} \id  
        \end{bmatrix} . 
    \end{align*} For the moment, let us assume that $\beta$ is chosen so that $1-\beta - c_{k+1} \neq 0$, and for convenience, denote $\mM \coloneqq \frac{1}{1-\beta}\id - \gamma_k \nabla^2 f(\vy)$. 
    Then the lower right block $\frac{1-\beta - c_{k+1}}{1-\beta} \id$ of $D h_k$ is invertible, and the Schur complement $\mS$ with respect to that block is \begin{align*}
    \mS &= (1 - \beta + \beta c_{k+1})\mM + \frac{\beta c_{k+1}}{1-\beta} \left(\frac{1-\beta - c_{k+1}}{1-\beta}\right)^{-1} \left( c_{k+1}\mM \right) \\ 
    &= \left( 1 - \beta + \beta c_{k+1} + \frac{\beta c_{k+1}^2}{1-\beta - c_{k+1}} \right)\mM \\ 
    &= \frac{(1-\beta)^2 (1-c_{k+1})}{1-\beta - c_{k+1}} \mM. 
    \end{align*} 
    Using the formula for the determinant of a block matrix in terms of an invertible principal submatrix and the Schur complement with respect to that submatrix, we obtain \begin{align*}
    \det D h_k (\vy, \vx) &= \det \left( \left( \frac{1-\beta - c_{k+1}}{1-\beta}  \right) \id \right) \det \mS \\ 
    &= \left( \frac{1-\beta - c_{k+1}}{1-\beta} \right)^d \frac{(1-\beta)^{2d} (1-c_{k+1})^d}{(1-\beta - c_{k+1})^d} \det \mM \\ 
    &= (1-\beta)^d (1-c_{k+1})^d \det \mM \\ 
    &= (1-\beta)^d (1-c_{k+1})^d \det \left( \frac{1}{1-\beta} \id - \gamma_k \nabla^2 f(\vy) \right) \\ 
    &= (1-c_{k+1})^d \det \left( \id - \gamma_k (1-\beta) \nabla^2 f(\vy) \right).
    \end{align*} 
    
    While we have shown that this holds for $\beta$ satisfying $1-\beta -c_{k+1} \neq 0$, the determinant of a matrix is continuous with respect to its entries. Hence, by continuity, the identity holds for all $\beta$; that is, \begin{equation} \label{eqn:det-dhk}
        \det D h_k (\vy, \vx) = (1-c_{k+1})^d \det \left( \id - \gamma_k (1-\beta) \nabla^2 f(\vy) \right).
    \end{equation} As the $L$-smoothness of $f$ implies that $\norm{\gamma_k (1-\beta)\nabla^2 f(\vy)} \leq \gamma_k (1-\beta) L< 1$, no eigenvalue of $\id - \gamma_k (1-\beta) \nabla^2 f(\vy)$ can be zero. 
    Therefore, $\det D h_k \neq 0$, and thus $D h_k$ is invertible. 
\end{proof}

From \eqref{eqn:det-dhk}, we also see that if $c_{k+1}=1$, then $Dh_k$ becomes rank deficient.
This is exactly as expected, because for any $(\vy_k, \vx_k)$, if $c_{k+1} = 1$ then by \Cref{lem:fallback} we know that $(\vy_{k+1}, \vx_{k+1}) = h_k(\vy_k, \vx_k)$ lies on the \emph{diagonal} $\Updelta \coloneqq \{(\vx, \vx) : \vx \in \rr^d\}$, which is a $d$-dimensional linear subspace of $\rr^d \times \rr^d$. 
This implies that $\rank (Dh_k ) \leq d$.

At the same time, this is also the source of additional difficulties. 
As before, let $k_0$ be the largest nonnegative integer such that $c_{k_0+1}=1$.
A typical statement about the avoidance of saddle points is that the set of initial points that converge to such points has measure zero, \`a la \Cref{musat-theorem}.
However, regardless of how the initial points are chosen, whether following the standard practice of setting $\vx_0=\vz_0$ or allowing possibly $\vx_0\neq\vz_0$, we eventually obtain $ \vx_{k_0+1} = \vy_{k_0+1} = \vz_{k_0+1}$, and in particular, $(\vy_{k_0+1}, \vx_{k_0+1}) \in \Updelta$.
Thus, although $\Updelta$ is Lebesgue null as a subset of $\rr^d\times\rr^d$, it may contain subsets that are themselves Lebesgue null by virtue of being subsets of a Lebesgue measure zero set, yet nevertheless have positive probability under the distribution of $(\vy_{k_0+1},\vx_{k_0+1})$. 
Because of the existence of such sets, even if we obtain a Lebesgue measure zero set $\gW$ of initial points whose iterates converge to saddle points, we cannot directly conclude that the iterates almost surely avoid $\gW$.
Notice that the same degeneracy is still present when $k_0 = -1$, under the standard initialization choice $\vx_0 = \vz_0$, since this initialization alone implies $\vx_0 = \vy_0$.

Similar measure-theoretic issues arise in the analysis of accelerated methods \citep{Onei19, Dixi26}.
To circumvent this complication, we adopt the technique of \citet{Onei19}, which consists of applying a small perturbation at a single iteration.
More precisely, we propose the following. 
We still use the update rule \eqref{eqn:SF}, up to obtaining $\vx_{k_0+1} = \vz_{k_0+1} \eqqcolon \vw_{k_0+1}$, but once we reach that point and only at that point, we slightly perturb the iterate $\vz_{k_0+1}$ and set \begin{equation} \label{eqn:z-perturbation}
\vz_{k_0+1}^+ = \vz_{k_0+1} + \vxi
\end{equation} where $\vxi \in \rr^d$ is any mean zero random variable whose distribution is absolutely continuous with respect to the Lebesgue measure. 
For example, one may set $\vxi$ to follow a Gaussian distribution $\gN(\vzero, \varepsilon^2 \id)$ for arbitrarily small $\varepsilon > 0$.
Then, at the $(k_0 + 1)$th iteration, we use $\vz_{k_0+1}^+$ in place of $\vz_{k_0+1}$ and compute \begin{subequations} \label{eqn:post-kick} \begin{align}
    \vy_{k_0+1} &= (1-\beta) \vz_{k_0+1}^+ + \beta \vx_{k_0+1}, \label{eqn:kicked-yk} \\
    \vz_{k_0+2} &= \vz_{k_0+1}^+ - \gamma_{k_0+1} \vg_{k_0+1}, \label{eqn:kicked-zk} \\
    \vx_{k_0+2} &= (1-c_{k_0+2}) \vx_{k_0+1} + c_{k_0+2} \vz_{k_0+2}.  \label{eqn:kicked-xk} 
\end{align} \end{subequations}
After that, we again resume updating the iterates following the update rule \eqref{eqn:SF}.

As mentioned therein, \Cref{sec:saddle} and, in particular, \Cref{thm:one-time-perturb} discuss the special case $k_0 = 0$ for clarity of the presentation. 
The general statement which reflects \eqref{eqn:z-perturbation} and \eqref{eqn:post-kick} is as follows.

\begin{thm}
    Consider \sfgd{} with a \emph{one-time perturbation} applied after $\vz_{k_0+1}$ and $\vx_{k_0+1}$ are computed, where $\vz_{k_0+1}$ is replaced by $\vz_{k_0+1}+\vxi$. 
    Suppose that the initial point $\vx_0=\vz_0$ and the perturbation $\vxi$ are drawn independently from distributions that are absolutely continuous with respect to Lebesgue measure on $\rr^d$. 
    Then, for $\beta \in [0, 1)$ and a twice continuously differentiable $L$-smooth $f$, as long as $\gamma < \frac{1}{L}$, the one-time perturbed \sfgd{} almost surely does not converge to any strict saddle point of $f$.
\end{thm}
\begin{proof}
    As shown in \Cref{lem:fallback}, for $k = 0, 1, \dots, k_0$ we have $\vx_k = \vy_k = \vz_k$. 
    Hence, by the update rule \eqref{eqn:zk} it holds that $\vz_{k+1} = \vz_k - \gamma_k \nabla f(\vz_k)$ for all $k = 0, 1, \dots, k_0$. 
    That is, $\vz_0, \vz_1, \dots, \vz_{k_0+1}$ is exactly the sequence of iterates obtained by running gradient descent from the initial point $\vz_0$ with learning rates $\gamma_k$.  
    For each $k = 0, 1, \dots, k_0$, define $\varphi_k : \rr^d \to \rr^d$ by \[
        \varphi_k (\vz) \coloneqq \vz - \gamma_k \nabla f(\vz)
    \] so that $\vz_{k+1} = \varphi_k(\vz_k)$.
    
    Now, consider the nonautonomous dynamical system $(\vy_{k+1}, \vx_{k+1}) = h_k (\vy_k, \vx_k)$, but temporarily as if it is initialized at time $k = k_0+1$. 
    By \Cref{lem:saddles-are-NPH}, \Cref{musat-theorem}, and \Cref{prop:hk-diff}, we can then conclude that there exists a Lebesgue measure zero set $\gW$ of points such that, for the trajectory of the system to converge to any of the points in $\gQ^*$, it is necessary to have $(\vy_{k_0+1}, \vx_{k_0+1}) \in \gW$. 

    Thus, if we can show that the joint distribution of $(\vy_{k_0+1}, \vx_{k_0+1})$ is absolutely continuous with respect to the Lebesgue measure on $\rr^d \times \rr^d$, we will be able to conclude that \sfgd{} almost surely does not converge to any strict saddle point of $f$. 
    To this end, let $\lambda$ denote the Lebesgue measure on $\rr^d$, and $\mu_0$ denote the distribution which $\vz_0$ is drawn from.
    For each $k = 0, 1, \dots, k_0$, let \[ 
    \mu_{k+1} = (\varphi_k \circ \dots \circ \varphi_1 \circ \varphi_0)_\sharp \mu_0 
    \] where $(\cdot)_\sharp$ denotes the pushforward, so that $\vz_{k+1} \sim \mu_{k+1}$.
    Notice that \[
    D\varphi_k (\vz) = \id - \gamma_k \nabla^2 f(\vz),
    \] and because the $L$-smoothness of $f$ implies $\norm{\gamma_k \nabla^2 f(\vz)} \leq \gamma_k L < 1$, it follows that $\det D\varphi_k \neq 0$.
    In particular, by the chain rule it holds that \[ 
        \det D (\varphi_k \circ \dots \circ \varphi_1 \circ \varphi_0) = \prod_{j=0}^k \det D \varphi_j \neq 0,
    \] so by \Cref{lem:diffeo}, $\mu_1, \dots, \mu_{k_0+1}$ are all absolutely continuous with respect to $\lambda$.

    Let $\nu$ denote the distribution which $\vxi$ is drawn from, so that $(\vz_{k_0+1}, \vxi) \sim \mu_{k_0+1} \otimes \nu$.
    Notice that \eqref{eqn:kicked-yk} is equivalent to \[
    \vy_{k_0+1} = \vz_{k_0+1} + (1-\beta) \vxi.
    \] Hence, for a linear transformation $T$ defined as \[
    T(\vu, \vv) = (\vu + (1-\beta)\vv, \vu)
    \] it holds that $(\vy_{k_0 + 1}, \vx_{k_0 + 1}) = T(\vz_{k_0+1}, \vxi)$.
    That is, $(\vy_{k_0 + 1}, \vx_{k_0 + 1}) \sim T_\sharp (\mu_{k_0+1} \otimes \nu)$. 
    Here, as the inverse of $T$ exists and is \[
    T^{-1}(\vp, \vq) = \left(\vq, \frac{1}{1-\beta}(\vp - \vq)\right) 
    \] we know that $\det DT \neq 0$. 
    Applying \Cref{lem:diffeo} once again, it follows that the joint distribution of $(\vy_{k_0+1}, \vx_{k_0+1})$ is indeed absolutely continuous with respect to the Lebesgue measure on $\rr^d \times \rr^d$. 
    This concludes the proof.
\end{proof}

\subsection{Convergence of the One-Time Perturbed \sfsgd}

As mentioned in \Cref{sec:saddle}, if we additionally assume that $\vxi$ is drawn from a mean zero distribution, then we can show that the one-time perturbation does not alter the overall convergence rate, provided that the second moment of $\vxi$ is chosen small enough to satisfy $\expt[\|\vxi\|^2] \leq \gamma_{k_0+1}$.
More precisely, we have the following. 

\begin{thm}\label{thm:kick-rate-preserved} 
    Consider \sfsgd{} with a \emph{one-time perturbation} applied after $\vz_{k_0+1}$ and $\vx_{k_0+1}$ are computed, where $\vz_{k_0+1}$ is replaced by $\vz_{k_0+1}+\vxi$ for $\vxi$ independently drawn from some mean zero distribution $\nu$. 
    Suppose that $\nu$ has a finite second moment, so that there exists some $\rho^2 \geq 0$ such that $\expt[\|\vxi\|^2] \leq \rho^2$.
    Then, under the setting identical to that of \Cref{thm:master-convergence-thm}, it holds that \begin{equation} \label{eqn:perturbed-rate}\begin{aligned}
    \min_{k = 0, \dots, T-1} \expt \bigl[\norm{\nabla f(\vy_k)}^2 \bigr] &\leq \frac{2(f(\vx_0) - f^*) + (1-\beta)^2 L \rho^2  + \beta \rho^2 / \gamma_{k_0+1} }{\gamma (1-\beta) (1-\warmratio)T} \\
        &\phantom{\leq} \qquad  + \frac{\gamma L (1-\beta) \sigma^2}{1-\warmratio} + \frac{3\beta(2-\beta) C \sigma^2 (1+\log T)}{ (1-\warmratio) T} . 
    \end{aligned} \end{equation}
\end{thm}
\begin{proof}
   To avoid confusion, let us reuse the notation $\vz_{k_0+1}^+ = \vz_{k_0+1} + \vxi$ introduced in \eqref{eqn:z-perturbation}. 
   That is, we replace $\vz_{k_0+1}$ by $\vz_{k_0+1}^+$.
    Most of the results established in \Cref{prop:appx-lyapunov-descent} continue to hold verbatim.
    Indeed, define the Lyapunov potentials exactly as in \eqref{eqn:lyapunov-potentials}, namely \[
        V_k = f(\vy_k) - f^* + \alpha_k  \norm{\vz_k - \vx_k}^2.
    \]
    Then for $k < k_0$, nothing is changed by the perturbation, so \eqref{eqn:simple-descent} still holds as is. 
    Also, after computing $\vz_{k_0+2}$ and $\vx_{k_0+2}$ using \eqref{eqn:kicked-zk} and \eqref{eqn:kicked-xk}, we return to the ordinary Schedule-Free update rules, so for $k \geq k_0+2$ the inequality \eqref{eqn:lyapunov-descent} still holds. 

    To study the consequences of the one-time perturbation, let us define \[
        V_{k_0+1}^+ = f(\vy_{k_0+1}) - f^* + \alpha_{k_0+1} \norm{\vz_{k_0+1}^+ - \vx_{k_0+1}}^2.
    \] If there had been no perturbation, \Cref{lem:fallback} would give $\vz_{k_0+1} = \vx_{k_0+1}$, and $\vy_{k_0+1}$ would coincide with this common value as well.
    To reflect this, let us define \[
        V_{k_0+1}^- = f(\vz_{k_0+ 1} ) - f^*.
    \] Then \eqref{eqn:simple-descent} applies to $V_{k_0+1}^- - V_{k_0}$, yielding \begin{equation}\label{eqn:pre-kick-descent}
    \expt[V_{k_0+1}^-] - \expt[V_{k_0}] \leq  - \frac{\gamma_{k_0} }{2} \expt \bigl[\norm{\nabla f(\vy_{k_0})}^2 \bigr]   + \frac{\gamma_{k_0}^2 L \sigma^2}{2} .
    \end{equation} In addition, observe that the validity of \eqref{eqn:lyapunov-descent} depends only on the identity $\vy_k = (1-\beta)\vz_k + \beta \vx_k$ and on the sequences of learning rates $\{\gamma_k\}_{k \geq 0}$ and averaging rates $\{c_{k+1}\}_{k \geq 0}$ satisfying the assumptions of \Cref{lem:Dk-negative,lem:Gk-bound}.
    Hence, \eqref{eqn:lyapunov-descent} applies to $V_{k_0+2} - V_{k_0+1}^+$, yielding \begin{equation} \label{eqn:post-kick-descent} \begin{aligned} 
    \expt[V_{k_0+2}] - \expt[V_{k_0+1}^+] &\leq  - \frac{\gamma_{k_0+1} (1-\beta)}{2} \expt \bigl[\norm{\nabla f(\vy_{k_0+1})}^2 \bigr] \\ 
    &\phantom{\leq} \qquad + \frac{\gamma_{k_0+1} \sigma^2}{2} \bigl(\beta (1 + \gamma_{k_0+1} L (1-\beta)) c_{k_0+2} + \gamma_{k_0+1} L (1-\beta)^2 \bigr).
    \end{aligned} \end{equation}
    Now let us examine what happens between $V_{k_0+1}^+$ and $V_{k_0+1}^-$.
    Using \eqref{eqn:z-perturbation} and \eqref{eqn:kicked-yk}, we obtain \begin{align*}
    V_{k_0+1}^+ - V_{k_0+1}^- &= f\bigl((1-\beta) \vz_{k_0+1}^+ + \beta \vx_{k_0+1}\bigr) - f(\vz_{k_0+1}) + \alpha_{k_0+1} \norm{\vz_{k_0+1}^+ - \vx_{k_0+1}}^2 \\ 
    &= f\bigl(\vz_{k_0+1} + (1-\beta) \vxi\bigr) - f(\vz_{k_0+1}) + \alpha_{k_0+1} \norm{\vxi}^2 \\
    &\leq \inprod{(1-\beta) \vxi}{\nabla f(\vz_{k_0+1})} + \frac{L}{2}\norm{(1-\beta) \vxi}^2 + \alpha_{k_0+1} \norm{\vxi}^2
    \end{align*} where the inequality follows from \Cref{lem:smooth}. 
    Taking conditional expectation given $\vz_{k_0+1}$, by the assumption that $\vxi$ is drawn from a mean zero distribution $\nu$ and is independent of $\vz_{k_0+1}$, we get \begin{equation} \label{eqn:55}
    \expt\left[V_{k_0+1}^+ \middle| \vz_{k_0+1}\right] - \expt\left[V_{k_0+1}^- \middle| \vz_{k_0+1}\right] \leq \left(\frac{(1-\beta)^2 L}{2} + \alpha_{k_0+1}\right) \expt\bigl[\norm{\vxi}^2\bigr].  
\end{equation}
    By \eqref{eqn:lyap-coeffs} and \eqref{eqn:35}, it holds that \begin{align*} 
        \alpha_{k_0+1} &= \frac{L \beta^2 c_{k_0+2}^2}{2} + \alpha_{k_0+2} (1-c_{k_0+2})^2 \\
        &= \frac{L \beta^2 c_{k_0+2}^2}{2} + \frac{\beta c_{k_0+2} (1-\gamma_{k_0+1} L (1 - \beta + \beta c_{k_0+2}))}{2\gamma_{k_0+1}} \\
        &= \frac{\beta c_{k_0+2} (1-\gamma_{k_0+1} L (1 - \beta) ) }{2\gamma_{k_0+1}} . 
    \end{align*} As $c_{k_0+2} \leq 1$ by \eqref{eqn:main-song-avg}, from the above we further get that \[
    \alpha_{k_0+1} \leq \frac{\beta}{2\gamma_{k_0+1}}.
    \] Hence, applying the law of total expectation on \eqref{eqn:55}, we obtain \begin{equation} \label{eqn:56}
    \expt\left[V_{k_0+1}^+\right] - \expt\left[V_{k_0+1}^-\right] \leq \left(\frac{(1-\beta)^2 L}{2} + \frac{\beta}{2\gamma_{k_0+1}}\right) \rho^2.  
    \end{equation}

    Therefore, by writing \[
    \expt[V_{k_0+2}] - \expt[V_{k_0}] = \left(\expt[V_{k_0+2}] - \expt\left[V_{k_0+1}^+\right]\right) + \left(\expt\left[V_{k_0+1}^+\right] - \expt\left[V_{k_0+1}^-\right]\right) + \left(\expt\left[V_{k_0+1}^-\right] - \expt[V_{k_0}]\right)
    \] and then observing the inequalities \eqref{eqn:pre-kick-descent}, \eqref{eqn:post-kick-descent}, and \eqref{eqn:56}, we conclude that the bound on $\expt[V_{k_0+2}] - \expt[V_{k_0}]$ for the one-time perturbed \sfsgd{} differs from the corresponding bound for the original \sfsgd{} only by the additional contribution from the right-hand side of \eqref{eqn:56}.
    We can then repeat the proof of \Cref{thm:master-convergence-thm} to obtain a counterpart of \eqref{eqn:master-bound-prototype}, which now reads \begin{align*} 
     \sum_{k=0}^{T-1} a_k \expt \bigl[\norm{\nabla f(\vy_k)}^2 \bigr]  &\leq f(\vx_0) - f^* + \frac{3\gamma \beta(1-\beta)(2-\beta) C \sigma^2 (1+\log T)}{2} \\ 
     &\phantom{\leq}\qquad +\frac{\gamma^2 L(1-\beta)^2\sigma^2 T}{2} + \left(\frac{(1-\beta)^2 L}{2} + \frac{\beta}{2\gamma_{k_0+1}}\right) \rho^2
    \end{align*} for $a_k$ defined in the beginning of the proof of \Cref{thm:master-convergence-thm}.
    Proceeding further exactly as in \Cref{thm:master-convergence-thm}, we obtain the claimed inequality. 
\end{proof}

Notice that we had the freedom to choose $\vxi$ to have an arbitrarily small second moment. 
In case we set $\expt[\|\vxi\|^2] \leq \gamma_{k_0+1}$, we may take $\rho^2 = \gamma_{k_0+1}$ in \eqref{eqn:perturbed-rate}, which makes the numerator of the first term in the right hand side be bounded as\begin{align*}
2(f(\vx_0) - f^*) + (1-\beta)^2 L \rho^2  + \frac{\beta \rho^2}{\gamma_{k_0+1}} &= 2(f(\vx_0) - f^*) + (1-\beta)^2 \gamma_{k_0+1} L  + \beta \\
&\leq 2(f(\vx_0) - f^*) + (1-\beta)^2 + \beta.
\end{align*}
Then one gets that the convergence rates are unaffected by the one-time perturbation.

\subsection{Proof of \titleref{Lemma}{lem:saddles-are-NPH}} \label{ssec:proving-saddles-are-NPH}

It is well known that, in a finite dimensional space, all norms are equivalent. 
Let us derive the explicit factors that arise when bounding the max-norm defined in \eqref{eqn:max-norm} by the $\ell_2$-norm and vice versa, given that the splitting $\rr^m = \Ecs \oplus \Eu$ is an orthogonal sum. 

\begin{prop}\label[prop]{prop:norm-equivalent}
    Under the notations used in \Cref{defn:PH}, let both $\Ecs$ and $\Eu$ be equipped with the $\ell_2$-norm $\norm{\cdot}$ inherited from $\rr^m$. 
    Suppose that $\Eu = \Ecs^\bot$. 
    Then for any $\vw \in \rr^m$, the max-norm defined as in \eqref{eqn:max-norm} satisfies \[
    \norm{\vw}_{\max} \leq \norm{\vw} \leq \sqrt{2} \norm{\vw}_{\max} .
    \]
\end{prop}
\begin{proof}
    Let $\vw = \vu + \vv$ with $\vu \in \Ecs$ and $\vv \in \Eu$. 
    As $\Eu = \Ecs^\bot$, we have $\inprod{\vu}{\vv} = 0$, hence \begin{align*}
        \norm{\vw}^2 = \norm{\vu + \vv}^2 &= \norm{\vu}^2 + \norm{\vv}^2 + 2\inprod{\vu}{\vv} \\ 
        &= \norm{\vu}^2 + \norm{\vv}^2.
    \end{align*} 
    On one hand, it immediately follows that \[
       \norm{\vw}^2 = \norm{\vu}^2 + \norm{\vv}^2 \leq 2 \left(\max\{\norm{\vu}, \norm{\vv}\}\right)^2 = 2 \norm{\vw}_{\max}^2.
    \] On the other hand, it is also straightforward that \[
       \norm{\vw}^2 = \norm{\vu}^2 + \norm{\vv}^2 \geq \max\bigl\{\|\vu\|^2, \|\vv\|^2\bigr\} = \left(\max\{\norm{\vu}, \norm{\vv}\}\right)^2 =  \norm{\vw}_{\max}^2.
    \] Taking the square roots leads to the desired inequality.
\end{proof}

\begin{lem}[\Cref{lem:saddles-are-NPH}]
    Let $f$ be an $L$-smooth and twice continuously differentiable function, and let $\vx^*$ be a strict saddle point of $f$, so that $(\vx^*,\vx^*) \in \gQ^*$ where $\gQ^*$ is defined as in \eqref{eqn:fixed-pts}.
    Suppose that there exist $\gamma_{\min}$ and $\gamma_{\max}$ such that $0 < \gamma_{\min} \leq \gamma_k \leq \gamma_{\max} < \frac{1}{L}$ for all $k$, and the averaging rates $\{c_{k+1}\}_{k \geq 0}$ satisfy $\lim_{k \to \infty} c_{k+1} = 0$.  
    Then, $(\vx^*, \vx^*)$ is an NPH unstable fixed point of the nonautonomous system \eqref{eqn:appx-nonauto-system}.
\end{lem}
\begin{proof}
    Upon a suitable translation, we may assume without loss of generality that $\vx^*=\vzero$.
    Let $\mH \coloneqq \nabla^2 f(\vzero) \in \rr^{d \times d}$. 
    By the strict saddle property, there exists an eigenvector of $\mH$ associated with a negative eigenvalue.
    Let $V$ be the span of all the eigenvectors of $\mH$ that are associated with a negative eigenvalue, and $W$ the span of all eigenvectors of $\mH$ associated with a nonnegative eigenvalue.
    As $\mH$ is symmetric hence orthogonally diagonalizable, it holds that $W = V^\bot$. 
    Thus, by setting $\Eu = V \times \{\vzero\}$ and $\Ecs = W \times \rr^d$, we have $\rr^{d} \times \rr^d = \Ecs \oplus \Eu$, with $\dim(\Eu) \geq 1$. 

    For each $k \geq 0$, define a linear map $T_k : \rr^{d} \times \rr^d \to \rr^{d} \times \rr^d$ by \[
        T_k(\vy, \vx) = \begin{bmatrix}
            \id - \gamma_k (1-\beta) \mH & \vzero \\ 
            \vzero                       & \id 
        \end{bmatrix} \begin{bmatrix}
            \vy \\ \vx
        \end{bmatrix} .
    \] We wish to show that, for each $k \geq 0$, there exists some $\mu_k$, $\lambda_k$, and $\epsilon_k$ with  $1\leq \lambda_k <\mu_k$ and $\epsilon_k > 0$ such that $(h_k, T_k)$ is $(\mu_k, \lambda_k, \frac{\epsilon_k}4)$-pseudo-hyperbolic on some open ball $B_r(\vzero)$ with respect to the splitting $\rr^d \times \rr^d = \Ecs \oplus \Eu$.

    Let us equip both $\Ecs$ and $\Eu$ with the standard $\ell_2$ norm inherited from $\mathbb{R}^d \times \mathbb{R}^d$. 
    Any vector in $\Eu$ is of the form $(\vv, \vzero)$ for some $\vv \in V$. 
    As $V$ is $\mH$-invariant, for any $k \geq 0$ we have \[
    T_k(\vv, \vzero) = \begin{bmatrix}
        \vv - \gamma_k (1-\beta) \mH \vv \\ \vzero
    \end{bmatrix} \in \Eu.
    \] Moreover, let $\mathrm{id}_V : V \to V$ denote the identity operator on $V$, and $\mathsf{H}_V : V \to V$ denote the linear operator corresponding to the restriction of $\mH$ to $V$. 
    Then $-\mathsf{H}_V$ is positive definite, by the definition of $V$. 
    In particular, if we set \begin{equation} \label{eqn:tmp-theta}
    \theta \coloneqq \min\{-\nu : \nu < 0 \text{ and $\nu$ is an eigenvalue of $\mH$}\}
    \end{equation} then the smallest eigenvalue of $-\mathsf{H}_V$ is also $\theta$. 
    It follows that the smallest eigenvalue of $\mathrm{id}_V + \gamma_k (1-\beta) (-\mathsf{H}_V)$ is $1+\gamma_k (1-\beta) \theta$, and hence \begin{equation} \label{eqn:unstable-expansion} \begin{aligned}
        \norm{T_k(\vv, \vzero)} = \norm{\vv - \gamma_k (1-\beta) \mH \vv} &= \norm{(\mathrm{id}_V + \gamma_k (1-\beta) (-\mathsf{H}_V)) \vv}  \\ 
        &\geq (1+\gamma_k (1-\beta) \theta)  \norm{\vv} \\ 
        &\geq (1+\gamma_{\min} (1-\beta) \theta)  \norm{\vv}. 
    \end{aligned} \end{equation}
    Meanwhile, any vector in $\Ecs$ is of the form $(\vw, \vu)$ for some $\vw \in W$.
    As $W$ is $\mH$-invariant, for any $k \geq 0$ we have \[
    T_k(\vw, \vu) = \begin{bmatrix}
        \vw - \gamma_k (1-\beta) \mH \vw \\ \vu
    \end{bmatrix} \in \Ecs.
    \] This time, let $\mathrm{id}_W : W \to W$ denote the identity operator on $W$, and $\mathsf{H}_W : W \to W$ denote the linear operator corresponding to the restriction of $\mH$ to $W$. 
    Then, by the definition of $W$, we know that $-\mathsf{H}_W$ is negative semidefinite, in particular having only nonpositive eigenvalues. 
    In addition, because the $L$-smoothness of $f$ implies that $\norm{\mH} \leq L$, we also have $\norm{\mathsf{H}_W} \leq L$, and thus the eigenvalues of $-\mathsf{H}_W$ are bounded below by $-L$.
    Hence, the eigenvalues of $\mathrm{id}_W + \gamma_k (1-\beta) (-\mathsf{H}_W)$ are bounded above by $1$, and bounded below by $1-\gamma_k (1-\beta)L \geq 1-(1-\beta) = \beta$. 
    It follows that \begin{equation} \label{eqn:center-stable-nonexpansion} \begin{aligned}
        \norm{T_k(\vw, \vu)}^2 &= \norm{\vw - \gamma_k (1-\beta) \mH \vw}^2 + \norm{\vu}^2  \\ 
        &= \norm{(\mathrm{id}_W + \gamma_k (1-\beta) (-\mathsf{H}_W)) \vw}^2 + \norm{\vu}^2 \\ 
        &\leq \norm{\vw}^2 + \norm{\vu}^2 \\
        &= \norm{(\vw, \vu)}^2. 
    \end{aligned} \end{equation}
    
    For $\theta$ defined as in \eqref{eqn:tmp-theta}, let us set $\mu = 1+\gamma_{\min}(1-\beta) \theta$. 
    We claim that there exists some $r > 0$ and a nonnegative integer $K_0$ which satisfy $\Lip((h_k - T_k)\vert_{B_{r}(\vzero)}) \leq \frac{\mu - 1}{20}$ in the sense of \mbox{\Cref{defn:PH}\ref{item:almost-linear}} whenever $k \geq K_0$.  
    If this is the case, then together with \eqref{eqn:unstable-expansion} and \eqref{eqn:center-stable-nonexpansion}, we will be able to conclude that $(h_k, T_k) \in \PsHy(\mu, 1, \frac{\mu - 1}{20}; B_r(\vzero), \Ecs\oplus\Eu)$ for all $k \geq K_0$.
    To this end, let us first observe that \begin{align*}
     h_k(\vy, \vx) - T_k(\vy, \vx) &= \begin{bmatrix} 
    \frac{\beta c_{k+1}}{1-\beta}(\vy - \vx) + \gamma_k(1-\beta)\bigl(\mH \vy - \nabla f(\vy) \bigr) - \gamma_k \beta c_{k+1} \nabla f(\vy) \\
    \frac{c_{k+1}}{1-\beta}(\vy - \vx) - \gamma_k c_{k+1} \nabla f(\vy) {}
    \end{bmatrix} \\[1ex]
    &= \gamma_k (1-\beta) \begin{bmatrix}
        \mH \vy - \nabla f(\vy) \\ \vzero
    \end{bmatrix} + \frac{c_{k+1}}{1-\beta} \begin{bmatrix}
        \beta (\vy - \vx) \\ \vy - \vx
    \end{bmatrix}  - \gamma_k  c_{k+1} \begin{bmatrix}
        \beta \nabla f(\vy) \\ \nabla f(\vy)
    \end{bmatrix} . 
    \end{align*} We now study the Lipschitz constant of each term in the last line above separately. 
    Notice that $\Eu = \Ecs^\bot$ by construction, so \Cref{prop:norm-equivalent} applies. 
    
    For any $\vy_1, \vy_2 \in \rr^d$, by the fundamental theorem of calculus, we have \begin{equation}\label{eqn:grad-FTC}\begin{aligned}
        \mH \vy_1 - \nabla f(\vy_1) - \bigl(\mH \vy_2 - \nabla f(\vy_2)\bigr) &= \mH(\vy_1 - \vy_2) + \nabla f(\vy_2) - \nabla f(\vy_1) \\ 
        &= \int_0^1 \Bigl( \nabla^2 f\bigl(\vy_1 + t(\vy_2-\vy_1)\bigr) - \mH \Bigr) (\vy_2-\vy_1) \, \mathrm{d} t   \\ 
        &= \int_0^1 \Bigl( \nabla^2 f\bigl(\vy_1 + t(\vy_2-\vy_1)\bigr) - \nabla^2 f(\vzero) \Bigr) \mathrm{d} t \cdot (\vy_2-\vy_1).
    \end{aligned}\end{equation} 
    As $f$ is twice continuously differentiable, $\nabla^2 f$ is continuous, so there exists some $\delta > 0$ such that $\norm{\vy} < \delta$ implies $\norm{\nabla^2 f(\vy) - \nabla^2 f(\vzero)} < \frac{\mu - 1}{140 \gamma_{\max} (1-\beta)}$. For such $\delta$, as having both $\norm{\vy_1} < \delta $ and $\norm{\vy_2} < \delta$ implies that \[ 
    \norm{\vy_1 + t(\vy_2 - \vy_1)} = \norm{(1-t)\vy_1 + t \vy_2} < \delta
    \] whenever $0\leq t\leq 1$, by using \eqref{eqn:grad-FTC} and \Cref{prop:norm-equivalent} we obtain \begin{equation} \label{eqn:PH-part1} \begin{aligned}
       \MoveEqLeft \norm{\gamma_k (1-\beta) \begin{bmatrix}
        \mH \vy_1 - \nabla f(\vy_1) \\ \vzero
    \end{bmatrix} - \gamma_k (1-\beta)  \begin{bmatrix}
        \mH \vy_2 - \nabla f(\vy_2) \\ \vzero
    \end{bmatrix}}_{\max} \\
    &\leq \gamma_k (1-\beta)  \norm{\begin{bmatrix}
        \mH \vy_1 - \nabla f(\vy_1) \\ \vzero
    \end{bmatrix} - \begin{bmatrix}
        \mH \vy_2 - \nabla f(\vy_2) \\ \vzero
    \end{bmatrix}} \\
    &= \gamma_k (1-\beta)  \norm{\mH \vy_1 - \nabla f(\vy_1) - \bigl(\mH \vy_2 - \nabla f(\vy_2)\bigr)} \\
    &\leq \gamma_k (1-\beta) \norm{ \int_0^1 \Bigl( \nabla^2 f\bigl(\vy_1 + t(\vy_2-\vy_1)\bigr) - \nabla^2 f(\vzero) \Bigr) \mathrm{d} t} \; \norm{\vy_2 - \vy_1} \\
    &< \frac{\gamma_k(\mu - 1)}{140 \gamma_{\max}} \norm{\vy_2 - \vy_1} \\
    &\leq \frac{\mu - 1}{140} \norm{\vy_2 - \vy_1}. 
    \end{aligned} \end{equation}

    Meanwhile, because we assume $\lim_{k \to \infty} c_{k+1} = 0$, there exists some nonnegative integer $K_0$ such that $0 \leq c_{k+1} \leq  \frac{(\mu-1)(1-\beta)}{140}$ whenever $k \geq K_0$.
    Then for any $\vy_1, \vy_2 \in \rr^d$, as long as $k \geq K_0$, it follows that \begin{align*}
     \norm{\gamma_k c_{k+1} \begin{bmatrix}
        \beta \nabla f(\vy_1) \\ \nabla f(\vy_1)
    \end{bmatrix} - \gamma_k c_{k+1} \begin{bmatrix}
        \beta \nabla f(\vy_2) \\ \nabla f(\vy_2)
    \end{bmatrix}}_{\max} &\leq \gamma_k c_{k+1} \norm{\begin{bmatrix}
        \beta (\nabla f(\vy_1) - \nabla f(\vy_2)) \\ \nabla f(\vy_1) - \nabla f(\vy_2)
    \end{bmatrix}}  \\
    &= \gamma_k c_{k+1} \sqrt{1+\beta^2} \ \bigl\| \nabla f(\vy_1) - \nabla f(\vy_2) \bigr\| \\
    &\leq \gamma_k L c_{k+1} \sqrt{1+\beta^2} \ \norm{ \vy_1  - \vy_2  } \\
    &\leq \frac{\gamma_k L (\mu-1) (1-\beta) \sqrt{1+\beta^2}}{140} \ \norm{ \vy_1  - \vy_2 } \\
    &\leq \frac{(\mu-1) (1-\beta) \sqrt{1+\beta^2}}{140} \ \norm{ \vy_1  - \vy_2 }.   
    \end{align*}
    Here, as $0\leq \beta < 1$, we have a simple inequality \[ 
        (1-\beta)\sqrt{1+\beta^2} \leq (1-\beta) \sqrt{1+2\beta+\beta^2} = (1-\beta)(1+\beta) = 1-\beta^2 \leq 1. 
    \] Thus, we eventually have
    \begin{equation} \label{eqn:PH-part2}  
     \norm{\gamma_k c_{k+1} \begin{bmatrix}
        \beta \nabla f(\vy_1) \\ \nabla f(\vy_1)
    \end{bmatrix} - \gamma_k c_{k+1} \begin{bmatrix}
        \beta \nabla f(\vy_2) \\ \nabla f(\vy_2)
    \end{bmatrix}}_{\max} \leq \frac{\mu-1}{140} \ \norm{ \vy_1  - \vy_2 }.   
    \end{equation}
    For the same $K_0$ as above, for any $\vy_1, \vy_2, \vx_1, \vx_2 \in \rr^d$, one can show, whenever $k \geq K_0$, that 
    \begin{equation} \label{eqn:PH-part3} \begin{aligned}
        \norm{\frac{c_{k+1}}{1-\beta} \begin{bmatrix}
        \beta (\vy_1 - \vx_1) \\ \vy_1 - \vx_1
    \end{bmatrix} - \frac{c_{k+1}}{1-\beta} \begin{bmatrix}
        \beta (\vy_2 - \vx_2) \\ \vy_2 - \vx_2
    \end{bmatrix}}_{\max} &\leq \frac{c_{k+1}}{1-\beta} \norm{\begin{bmatrix}
        \beta (\vy_1 - \vy_2 + \vx_2 - \vx_1) \\ \vy_1 - \vy_2 + \vx_2 - \vx_1
    \end{bmatrix}  }  \\
    &= \frac{c_{k+1} \sqrt{1+\beta^2}}{1-\beta} \norm{ \vy_1 - \vy_2 + \vx_2 - \vx_1 } \\
    &\leq \frac{c_{k+1} \sqrt{1+\beta^2}}{1-\beta} \bigl( \norm{ \vy_1 - \vy_2} +\norm{\vx_2 - \vx_1}\bigr) \\
    &\leq \frac{(\mu - 1) \sqrt{2}}{140} \bigl( \norm{ \vy_1 - \vy_2} +\norm{\vx_1 - \vx_2 }\bigr). 
    \end{aligned} \end{equation}
    
    Also, for $\delta$ and $K_0$ defined as above, let $r = \frac{\delta}{\sqrt{2}}$. 
    Then for any $(\vy, \vx) \in B_r(\vzero)$, by \Cref{prop:norm-equivalent}, we have \[
        \norm{\vy} = \left\lVert \begin{bmatrix}
            \vy \\[1ex] \vzero
        \end{bmatrix} \right\rVert \leq \sqrt{2} \left\lVert \begin{bmatrix}
            \vy \\[1ex] \vzero
        \end{bmatrix} \right\rVert_{\max} < \sqrt{2} r = \delta .
    \]  
    Therefore, for any $(\vy_1, \vx_1)$ and $(\vy_2, \vx_2) $ in $B_r(\vzero)$, whenever $k \geq K_0$, from the three inequalities \eqref{eqn:PH-part1}, \eqref{eqn:PH-part2}, and \eqref{eqn:PH-part3}, we obtain that \begin{align*}
      \MoveEqLeft \norm{ h_k(\vy_1, \vx_1) - T_k(\vy_1, \vx_1) - \bigl(h_k(\vy_2, \vx_2) - T_k(\vy_2, \vx_2)\bigr)}_{\max} \\ 
     &\leq  \norm{\gamma_k (1-\beta) \begin{bmatrix}
        \mH \vy_1 - \nabla f(\vy_1) \\ \vzero
    \end{bmatrix} - \gamma_k (1-\beta)  \begin{bmatrix}
        \mH \vy_2 - \nabla f(\vy_2) \\ \vzero
    \end{bmatrix}}_{\max}  \\ 
    &\phantom{\leq} \qquad + \norm{\gamma_k c_{k+1} \begin{bmatrix}
        \beta \nabla f(\vy_1) \\ \nabla f(\vy_1)
    \end{bmatrix} - \gamma_k c_{k+1} \begin{bmatrix}
        \beta \nabla f(\vy_2) \\ \nabla f(\vy_2)
    \end{bmatrix}}_{\max} \\ 
    &\phantom{\leq} \qquad + \norm{\frac{c_{k+1}}{1-\beta} \begin{bmatrix}
        \beta (\vy_1 - \vx_1) \\ \vy_1 - \vx_1
    \end{bmatrix} - \frac{c_{k+1}}{1-\beta} \begin{bmatrix}
        \beta (\vy_2 - \vx_2) \\ \vy_2 - \vx_2
    \end{bmatrix}}_{\max} \\
    &\leq  \frac{ 2 (\mu-1)}{140} \norm{ \vy_1  - \vy_2  } + \frac{(\mu - 1) \sqrt{2}}{140} \bigl( \norm{ \vy_1 - \vy_2} +\norm{\vx_1 - \vx_2 }\bigr)  \\
    &\leq  \frac{ (\mu-1) (2+\sqrt{2})}{140} \bigl( \norm{ \vy_1 - \vy_2} +\norm{\vx_1 - \vx_2 }\bigr) \\
    &\leq  \frac{ (\mu-1) (2\sqrt{2}+2)}{140} \sqrt{\norm{ \vy_1 - \vy_2}^2 +\norm{\vx_1 - \vx_2 }^2} \\
    &=  \frac{ (\mu-1) (2\sqrt{2}+2)}{140} \norm{ (\vy_1, \vx_1) - (\vy_2, \vx_2)} \\
    &\leq  \frac{ (\mu-1) (4+2\sqrt{2})}{140} \norm{ (\vy_1, \vx_1) - (\vy_2, \vx_2)}_{\max} ,
    \end{align*} where the fourth inequality follows from the Cauchy--Schwarz inequality, and the last inequality follows from \Cref{prop:norm-equivalent}.
    As $2\sqrt{2} = \sqrt{8} < \sqrt{9} = 3$, we conclude that \[
    \norm{ h_k(\vy_1, \vx_1) - T_k(\vy_1, \vx_1) - \bigl(h_k(\vy_2, \vx_2) - T_k(\vy_2, \vx_2)\bigr)}_{\max} \leq \frac{\mu - 1}{20} \norm{ (\vy_1, \vx_1) - (\vy_2, \vx_2)}_{\max}.
    \] That is, whenever $k \geq K_0$, the map $h_k - T_k$ is $\frac{\mu - 1}{20}$-Lipschitz on $B_r(\vzero)$ with respect to the \mbox{max-norm}. 
    It follows indeed that $(h_k, T_k) \in \PsHy(\mu, 1, \frac{\mu-1}{20}; B_r(\vzero), \Ecs\oplus\Eu)$ for all $k \geq K_0$.

    Finally, for each $k \geq 0$, let $\mu_k = \mu$, $\lambda_k = 1$, and $\epsilon_k = \frac{\mu-1}{5}$.
    It is clear that the conditions \ref{item:NPH-i}--\ref{item:NPH-iii} in \Cref{defn:NPH} are all satisfied. 
    Therefore, $\vzero$ is an NPH unstable fixed point of the system \eqref{eqn:appx-nonauto-system}.
\end{proof}

\input{pep.tex}

\input{ode-study.tex}

%% file: lemmata_bag.tex
\tableofcontents

\section{Useful Lemmata} \label{appx:useful-lemmata}

 \subsection{Basic Inequalities}
Let us begin by recalling some standard facts that are consequences of the assumptions we impose. 
\begin{lem}\label[lem]{lem:smooth}
Suppose that $f: \rr^d \to \rr$ is $L$-smooth. 
Then, for any $\vx, \vy \in \rr^d$, it holds that \[
    f(\vx) \leq f(\vy) + \inprod{\vx - \vy}{\nabla f(\vy)} + \frac{L}2 \norm{\vx- \vy}^2. 
\] 
\end{lem} 
\begin{proof}
Consider a function $h : \rr \to \rr$ defined as \[  
h(t) = f((1-t)\vy + t \vx) + (1-t) \inprod{\vx - \vy}{\nabla f(\vy)}.
\] 
Differentiating and then applying the Cauchy--Schwarz inequality, we obtain that \begin{align*}
h'(t) &= \inprod{\nabla f((1-t)\vy + t\vx)}{\vx - \vy} - \inprod{\vx - \vy}{\nabla f(\vy)} \\
&= \inprod{\nabla f((1-t)\vy + t\vx) - \nabla f(\vy)}{\vx - \vy} \\
&\leq \norm{\nabla f((1-t)\vy + t\vx) - \nabla f(\vy)} \norm{\vx - \vy}.
\end{align*} By the $L$-smoothness of $f$, it further follows from the last line above that \begin{align*}
h'(t) &\leq L \norm{(1-t)\vy + t\vx - \vy} \norm{\vx - \vy} \\
&= L t \norm{\vx - \vy}^2.
\end{align*} Therefore, by the fundamental theorem of calculus, we have \begin{align*}
    f(\vx) - f(\vy) - \inprod{\vx - \vy}{\nabla f (\vy)} &= h(1) - h(0) \\
    &= \int_0^1 h'(t) \, \mathrm{d}t \\
    &\leq L \norm{\vx - \vy}^2 \int_0^1 t \, \mathrm{d}t \\
    &= \frac{L}{2} \norm{\vx - \vy}^2.
\end{align*} Rearranging the terms gives us the claimed inequality.
\end{proof}

\begin{lem}\label[lem]{lem:noise-bound}
Suppose that \Cref{asmp:stochastic} holds. 
Then, for any $\vx \in \rr^d$, we have \[
    \expt\bigl[\norm{\nabla f(\vx, \zeta)}^2 \, \big| \, \vx \bigr] \leq \norm{\nabla f(\vx)}^2 + \sigma^2 . 
\] 
\end{lem} 
\begin{proof}
    Observe that \begin{align*}
        \norm{\nabla f(\vx, \zeta)}^2 &= \norm{\nabla f(\vx) + (\nabla f(\vx, \zeta) - \nabla f(\vx))}^2 \\
        &= \norm{\nabla f(\vx)}^2 - 2\inprod{\nabla f(\vx)}{\nabla f(\vx, \zeta) - \nabla f (\vx)} + \norm{\nabla f(\vx, \zeta) - \nabla f(\vx)}^2.
    \end{align*} Taking the conditional expectation with respect to $\vx$, the conclusion is immediate from the equations in \Cref{asmp:stochastic}.
\end{proof}

The following is a simple result from elementary calculus, which we state here for completeness. 
\begin{lem} \label[lem]{lem:harmonic-bound}
    For any positive integer $K \geq 1$, it holds that \[
        \sum_{k=1}^K \frac{1}{k+1} \leq \log(K+1).
    \]
\end{lem} 
\begin{proof}
    For any positive integer $k \geq 1$, the inequality $\frac{1}{k+1} \leq \frac{1}{t}$ holds whenever $k \leq t \leq k+1$. It thus follows that \[
        \sum_{k=1}^K \frac{1}{k+1} = \sum_{k=1}^K \int_k^{k+1} \frac{1}{k+1} \, \mathrm{d} t \leq \sum_{k=1}^K \int_k^{k+1} \frac{1}{t} \, \mathrm{d} t = \int_1^{K+1} \frac{1}{t} \, \mathrm{d} t = \log(K+1). \qedhere
    \]
\end{proof}

\subsection{Local Diffeomorphisms and Absolute Continuity of Measures}

Recall that a map $\varphi : \rr^d \to \rr^d$ is a local diffeomorphism if its Jacobian $D\varphi$ is invertible everywhere. 

\begin{lem} \label[lem]{lem:diffeo}
    Let $\mu$ be a measure on $\rr^d$ which is absolutely continuous with respect to Lebesgue measure. 
    Suppose that $\varphi : \rr^d \to \rr^d$ is a continuously differentiable map such that $\det D \varphi \neq 0$. 
    Then the pushforward $\varphi_\sharp \mu$ is also absolutely continuous with respect to Lebesgue measure. 
\end{lem}

\begin{proof}
    Let $\lambda$ denote the Lebesgue measure on $\rr^d$. 
    By the inverse function theorem, for any $\vx \in \rr^d$, there exists an open neighborhood $U_\vx$ of $\vx$ such that $\varphi$, when restricted to $U_\vx$, is a diffeomorphism between $U_\vx$ and $\varphi(U_\vx)$.
    As $\{U_\vx : \vx \in \rr^d\}$ is an open cover of a second countable space $\rr^d$, by Lindel\"{o}f's lemma, it has a countable subcover. 
    Let $\vx_1, \vx_2, \dots$ be the sequence in $\rr^d$ such that, by denoting $U_i = U_{\vx_i}$, the collection $\{U_i : i = 1, 2, \dots\}$ becomes an open cover of $\rr^d$.

    Fix any $E \subset \rr^d$ such that $\lambda (E) = 0$. 
    For each $i = 1, 2, \dots$, let $V_i = \varphi(U_i)$ and $A_i = \varphi^{-1}(E) \cap U_i$. 
    Because $\varphi$ is a continuously differentiable diffeomorphism between $U_i$ and $V_i$, the change of variables formula yields \begin{equation} \label{eqn:lambda-null}
        \lambda(A_i) = \lambda\bigl((\varphi|_{U_i})^{-1}(E\cap V_i)\bigr) = \int_{E\cap V_i} \abs{\det D(\varphi|_{U_i})^{-1}} \mathrm{d}\lambda = 0.
    \end{equation}
    Meanwhile, as $\{U_i: i = 1, 2, \dots\}$ is an open cover of $\rr^d$, by the definition of pushforward measures we have \[
        \varphi_\sharp \mu(E) = \mu(\varphi^{-1}(E)) = \mu\left( \bigcup_{i}A_i \right) \leq \sum_{i} \mu ( A_i ).
    \]
    But here, as $\mu$ is absolutely continuous with respect to $\lambda$, \eqref{eqn:lambda-null} implies that $\mu(A_i) = 0$ for each $i = 1, 2, \dots$. 
    Hence, it follows that $\varphi_\sharp \mu(E) = 0$, and therefore, $\varphi_\sharp \mu$ is also absolutely continuous with respect to $\lambda$. 
\end{proof}

\subsection{Identities Between Schedule-Free Iterates}
The intertwined update rules of $\{\vx_k\}_{k \geq 0}$,  $\{\vy_k\}_{k \geq 0}$, and  $\{\vz_k\}_{k \geq 0
}$ in the Schedule-Free method give rise to several useful identities among these sequences.
We present two such identities here, which will be used in the convergence analyses.

\begin{lem}\label[lem]{lem:y-diff}
Let $\{\vx_k\}_{k \geq 0}$,  $\{\vy_k\}_{k \geq 0}$, and  $\{\vz_k\}_{k \geq 0}$ be the sequences computed by the update rule~\eqref{eqn:SF}. 
Then for any $k \geq 0$, it holds that \[
\vy_{k+1} - \vy_k = \beta c_{k+1}(\vz_k - \vx_k) - \gamma_k (1 - \beta + \beta c_{k+1}) \vg_k.
\] 
\end{lem} 
\begin{proof}
From the equations in \eqref{eqn:SF}, one can observe that \begin{align*}
        \vy_{k+1} &= (1-\beta) \vz_{k+1} + \beta \vx_{k+1} \\
                  &= (1-\beta) \vz_{k+1} + \beta ((1- c_{k+1})\vx_k + c_{k+1} \vz_{k+1}) \\
                  &= (1-\beta + \beta c_{k+1})\vz_{k+1} + \beta(1-c_{k+1})\vx_k \\
                  &= (1-\beta + \beta c_{k+1})(\vz_k - \gamma_k \vg_k) + \beta(1-c_{k+1})\vx_k \\
                  &= \vy_k - \gamma_k (1-\beta)\vg_k + \beta c_{k+1}(\vz_k - \gamma_k \vg_k)  - \beta c_{k+1} \vx_k.
    \end{align*} By rearranging the terms, the conclusion is immediate.
\end{proof}

\begin{lem}\label[lem]{lem:zx-diff}
Let $\{\vx_k\}_{k \geq 0}$ and  $\{\vz_k\}_{k \geq 0}$ be the sequences computed by the update rule \eqref{eqn:SF}. 
Then for any $k \geq 0$, it holds that \[
\vz_{k+1} - \vx_{k+1} = (1 - c_{k+1})(\vz_k - \vx_k - \gamma_k \vg_k) .
\] 
\end{lem} 
\begin{proof}
From the equations in \eqref{eqn:SF}, one can observe that \begin{align*}
        \vz_{k+1} - \vx_{k+1} &= \vz_{k+1} - ((1- c_{k+1})\vx_k + c_{k+1} \vz_{k+1}) \\ 
                              &= (1 - c_{k+1}) (\vz_{k+1} - \vx_k) \\
                              &= (1 - c_{k+1}) (\vz_k - \vx_k - \gamma_k \vg_k)  
    \end{align*} which is exactly the claimed identity. 
\end{proof}

\begin{lem}\label[lem]{lem:fallback}
    Let $\{\vx_k\}_{k \geq 0}$,  $\{\vy_k\}_{k \geq 0}$, and  $\{\vz_k\}_{k \geq 0}$ be the sequences computed by the update rule~\eqref{eqn:SF}. 
    For any $k \geq 0$, if $c_{k+1} = 1$, then $\vx_{k+1}=\vy_{k+1}=\vz_{k+1}$. 
\end{lem}
\begin{proof}
    This is an immediate consequence of the update rules. 
    Indeed, suppose that $c_{k+1} = 1$. 
    Then, from \eqref{eqn:xk} we get that $\vx_{k+1} = \vz_{k+1}$, and this, with \eqref{eqn:yk}, asserts that $\vy_{k+1} = \vx_{k+1}$.  
\end{proof}

\subsection{Log-Concave Sequences}

Let us begin by recalling the definition of log-concavity of a sequence. 
\begin{defn}\label[defn]{defn:log-concave}
    A sequence of positive reals $\{w_k\}_{k \geq 0}$ is said to be \emph{log-concave} if \begin{equation}  \label{eqn:defn-log-concave}
    \frac{w_k}{w_{k-1}} \geq \frac{w_{k+1}}{w_k},
    \end{equation} or equivalently $w_k^2 \geq w_{k-1}w_{k+1}$, for all $k \geq 1$. 
\end{defn}

\begin{lem}\label[lem]{lem:log-concave-preserving}
    Let $\{v_k\}_{k \geq 0}$ and $\{w_k\}_{k \geq 0}$ be log-concave sequences of positive reals. 
    Then, the following sequences are also log-concave. 
    \begin{enumerate}[label=(\roman*)]
        \item For any two constants $a$ and $b$, a linear sequence $\{a k + b\}$ is log-concave, within the range of $k$ such that $ak + b > 0$. \label{item:constant}
        \item The pointwise product $\{v_k w_k\}_{k \geq 0}$ is log-concave.  \label{item:product} 
        \item The pointwise minimum $\{\min\{v_k, w_k\}\}_{k \geq 0}$ is log-concave. \label{item:minimum}
    \end{enumerate}
\end{lem}
\begin{proof}
    For any $k$ such that both $ak+b$ and $a(k+1)+b$ are positive, observe that \[
    \frac{a(k+1) + b}{ak + b} = 1 + \frac{1}{k+\frac{b}{a}}
    \] is decreasing with respect to $k$. 
    Hence, considering the condition \eqref{eqn:defn-log-concave} for log-concavity, statement~\ref{item:constant} holds.

    It is also straightforward to verify \ref{item:product}, as the log-concavity of each sequence immediately implies the inequality \[ 
    v_{k}^2 w_k^2 \geq (v_{k-1}w_{k-1})(v_{k+1}w_{k+1}).
    \]

    To show that \ref{item:minimum} holds, let us temporarily fix any $k \geq 1$. 
    Without loss of generality, we may assume that $v_k \geq w_k$.
    Then, by the log-concavity of $\{w_k\}_{k \geq 0}$, it holds that \[
        \min\{v_k, w_k\}^2 = w_k^2 \geq w_{k-1} w_{k+1} \geq \min\{v_{k-1}, w_{k-1}\} \cdot \min\{v_{k+1}, w_{k+1}\}.
    \] As $k \geq 1$ was arbitrary, it follows that $\{\min\{v_k, w_k\}\}_{k \geq 0}$ is indeed log-concave.
\end{proof}

\begin{lem}\label[lem]{lem:log-concave-decr}
    Let $\{w_k\}_{k\geq 0}$ be log-concave. 
    Then $\{\tau_k\}_{k\geq 0}$ defined as \[ 
        \tau_k = \frac{w_k}{\sum_{j=0}^k w_j}
    \] 
is monotone decreasing. 
\end{lem}
\begin{proof}
    Observe that \begin{align}
        \tau_k - \tau_{k+1} &= \frac{w_k}{\sum_{j=0}^k w_j} - \frac{w_{k+1}}{w_{k+1} + \sum_{j=0}^k w_j} \notag \\ 
        &=\frac{1}{\left(w_{k+1} + \sum_{j=0}^k w_j\right) \sum_{j=0}^k w_j } \left(w_k w_{k+1} + (w_k - w_{k+1}) \sum_{j=0}^k w_j\right). \label{eqn:tmp-14} 
    \end{align}
    Hence, if $w_k \geq w_{k+1}$ then we immediately have $\tau_k \geq \tau_{k+1}$, so let us assume that $w_k < w_{k+1}$. 
    By the log-concavity of $\{w_k\}_{k\geq 0}$, the sequence of the ratios between adjacent terms $\{\frac{w_{k+1}}{w_k}\}_{k \geq 0}$ is decreasing. So for any $0 \leq j \leq k$, it holds that \[
    w_k =  w_j \prod_{i=j+1}^{k} \frac{w_i}{w_{i-1}} \geq w_j \left( \frac{w_{k+1}}{w_k} \right)^{k-j}
    \] and thus we also have \begin{align*}
    \sum_{j=0}^k w_j \leq \sum_{j=0}^k w_k \left(\frac{w_k}{w_{k+1}}\right)^{k-j} &\leq w_k \sum_{m=0}^\infty \left(\frac{w_k}{w_{k+1}}\right)^m \\[1ex]
    &= \frac{w_k}{1-\frac{w_k}{w_{k+1}}} \\[1ex]
    &= \frac{w_kw_{k+1}}{w_{k+1}-{w_k}}. 
    \end{align*} Plugging this into \eqref{eqn:tmp-14} while noting that we are considering the case where $w_k - w_{k+1} < 0$, we get \[
    \tau_k - \tau_{k+1} = \frac{1}{\left(w_{k+1} + \sum_{j=0}^k w_j\right) \sum_{j=0}^k w_j } \left(w_k w_{k+1} + (w_k - w_{k+1})\frac{w_kw_{k+1}}{w_{k+1}-{w_k}}\right) \geq 0.
    \] Therefore, we have $\tau_k \geq \tau_{k+1}$ in either of the cases, so $\{\tau_k\}_{k \geq 0}$ is monotone decreasing. 
\end{proof}

%% file: pep.tex
\section{PEP Experiments} \label{sec:pep-experiments}

All PEP results in this work are obtained using the Python package \texttt{PEPit} \citep{PEPit}.
In formulating a PEP instance with an objective function $f$ that is possibly nonconvex yet bounded below, we invoke the following theorem. 

\begin{thm}[{\citep{Dror20}, Theorem~7}]
    For a fixed $L > 0$, let $\{(\vx_i, \vg_i, f_i)\}_{i=0, \dots, T} \subset \rr^d \times \rr^d \times \rr$ be a set of triplets that satisfy \[
      \frac{1}{2L} \norm{\vg_i - \vg_j}^2 - \frac{L}{4} \norm{\vx_i - \vx_j - \frac{1}{L} (\vg_i - \vg_j)} \leq f_i - f_j - \inprod{\vg_j}{\vx_i - \vx_j}
    \] for all $0 \leq i, j \leq T$. Then, there exists an $L$-smooth function $f$ such that $f(\vx_i) = f_i$ and $\nabla f(\vx_i) = \vg_i$ for all $i =0, \dots, T$, which in addition satisfies \[
\min_{\vx} \; f(\vx) = \min_{i \in \{0, 1, \dots, T\}} \; f_i - \frac{1}{2L} \norm{\vg_i}^2.
\]
\end{thm}

Based on this fact, we enforce the constraints of the form $f(\vx) - \frac{1}{2L} \norm{\nabla f(\vx)}^2 \geq f^*$ instead of the na\"{i}ve $f(\vx) \geq f^*$ as a formulation of \Cref{asmp:lower-bound}. 
We note that a similar strategy was adopted by \citet{Abba22}.
In particular, the results plotted in \Cref{fig:x05} are obtained by solving
\begin{equation} \label{eqn:pep-xk} \begin{aligned}
    \max \ &{} \min_{0\leq k\leq T} \; \norm{\nabla f(\vx_k)}^2 \\
    \text{s.t. } &{} \ \text{$f$ is $L$-smooth}, \\
                 &{} \ \vx_0 = \vy_0 = \vz_0, \\
                 &{} \ f(\vx_0) - f^* \leq 1, \\
                 &{} \ \text{$\vx_k$, $\vy_k$, and $\vz_k$ computed by \sfgd}, \quad \forall k = 1, \dots, T, \\  
                 &{} \ f(\vx_k) - \frac{1}{2L} \norm{\nabla f(\vx_k)}^2 \geq f^*, \quad \forall k = 0, 1, \dots, T, \\
                 &{} \ f(\vy_k) - \frac{1}{2L} \norm{\nabla f(\vy_k)}^2 \geq f^*, \quad \forall k = 0, 1, \dots, T-1
\end{aligned} \end{equation} with $L = 1$ and $\gamma = \frac{1}{L}$.  
Notice that there is no loss of generality in setting $L = 1$, because for an $L$-smooth function $f$ we can always alternatively consider $\frac{f}{L}$ in place of $f$.
Also, to avoid unnecessary complexities, we consider the case where there is no warmup phase, \textit{i.e.}, $\warmstep = 1$, so that $\warmratio = 0$ and the averaging rates become the standard $c_{k+1} = \frac{1}{k+1}$.
As both the update rule of \sfgd{} and the optimality measure do not involve $f(\vz_k)$ nor $\nabla f(\vz_k)$, the lower bound constraints regarding $\vz_k$ are not included in the PEP instance \eqref{eqn:pep-xk}. 

As a complementary verification of the convergence rates we established in \Cref{sec:conv}, we also numerically computed the tight worst-case upper bounds on $\min_{0 \leq k < T} \|\nabla f(\vy_k)\|^2$ for the iterates generated by \sfgd, under the assumptions that $f$ is $1$-smooth and bounded below. 
Such results are obtained by solving \begin{align*} 
    \max \ &{} \min_{0\leq k\leq T} \; \norm{\nabla f(\vy_k)}^2 \\
    \text{s.t. } &{} \ \text{$f$ is $L$-smooth}, \\
                 &{} \ \vx_0 = \vy_0 = \vz_0, \\
                 &{} \ f(\vx_0) - f^* \leq 1, \\
                 &{} \ \text{$\vx_k$, $\vy_k$, and $\vz_k$ computed by \sfgd}, \quad \forall k = 1, \dots, T, \\  
                 &{} \ f(\vy_k) - \frac{1}{2L} \norm{\nabla f(\vy_k)}^2 \geq f^*, \quad \forall k = 0, 1, \dots, T-1. 
\end{align*}
As now the only points where the gradient of $f$ should be computed are $\vy_k$, we may remove the lower bound constraints regarding $\vx_k$. 

\begin{figure}
  \centering
    \includegraphics[width=0.5\textwidth]{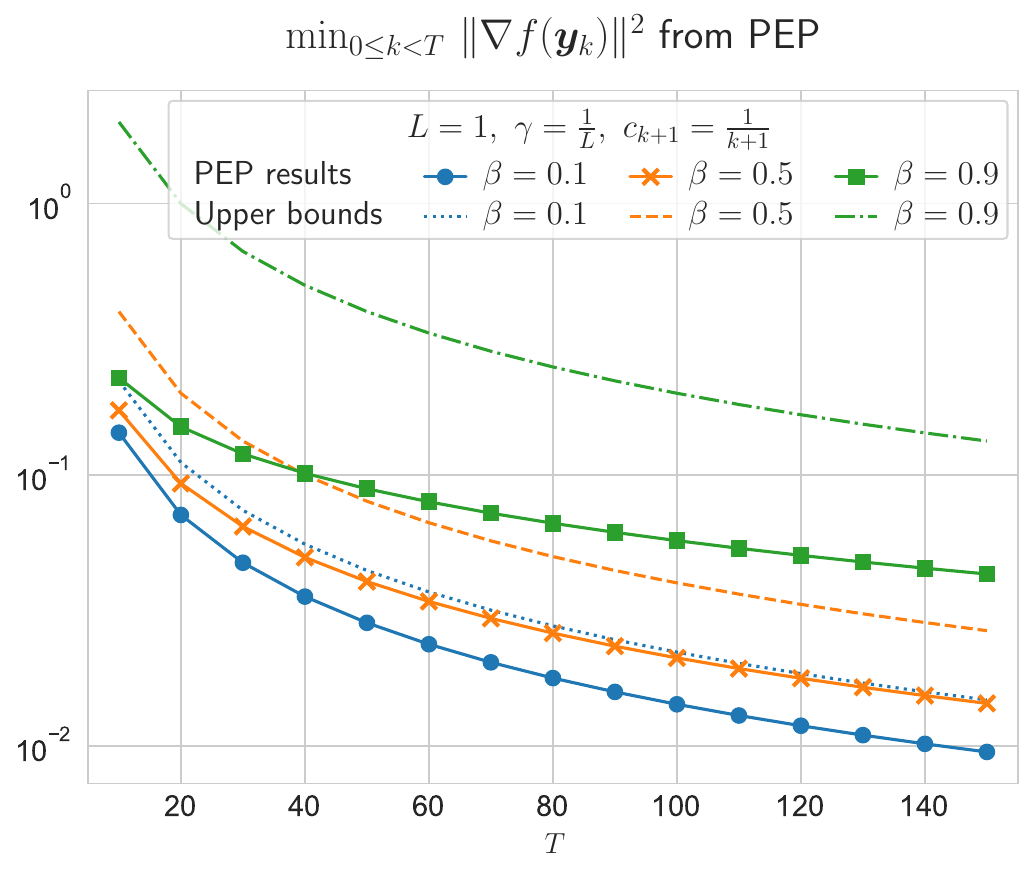}
    \caption{
      The convergence plots regarding the gradient location iterates $\vy_k$, with varying only $\beta$ among the parameters.
      The solid lines with markers are the plots for the PEP results. 
      The other dashed and dotted lines are the plots of the theoretical upper bounds obtained in \Cref{cor:sfgd-rate}.
    }
    \label{fig:y05}
\end{figure}

The results are plotted in \Cref{fig:y05}. 
Because we impose the initial condition $f(\vx_0) - f^* \leq 1$, the theoretical upper bound from \Cref{cor:sfgd-rate} reads $\frac{2}{(1-\beta) T}$.
The plots show that the tight worst-case performances are indeed bounded by this quantity, validating our theoretical upper bound.

In our computing environment, equipped with an Intel\textregistered{} Xeon\textregistered{} Gold 6226R 32-core CPU and 384~GB of RAM, the computations for \Cref{fig:x05} took approximately 6 hours for each choice of $\beta$, while those for \Cref{fig:y05} took approximately 4 hours for each choice of $\beta$.

%% file: ode-study.tex
\section{Regularity of \sfode{} near \texorpdfstring{$t = 0$}{t=0}} \label{appx:ode-study}

Following the standard notation, for an interval $I \subset \rr$, we let $C(I, \rr^d)$ be the set of all continuous functions $I \to \rr^d$. 
In cases where $I$ is a closed interval, we equip $C(I, \rr^d)$ with the uniform norm $\norm{\vw}_\infty = \sup_{t \in I} \norm{\vw(t)}$ so that it becomes a Banach space. 

The equation \eqref{eqn:ode-x} has a factor $\frac{1}{t}$, which could potentially make the solutions, or even the ODE itself, ill-defined at $t = 0$. 
However, it turns out that this singularity at $t = 0$ is removable, in the sense that we can always find a solution whose value at $t = 0$ can be assigned by a continuous extension to be exactly the desired initial point.
The goal of this section is to establish this fact. 

As a first step, let $T$ be any positive real number, and consider the ODE \eqref{eqn:SF-ODE} and its solution $(\vx, \vy, \vz)$ on the half-open interval $(0, T]$. 
We show that whether the solutions can be continuously extended to the closed interval $[0, T]$ can be reduced into determining whether a certain integral equation has a well-posed solution.  

For convenience, let us define an operator $\gI : C([0, T], \rr^d) \to C([0, T], \rr^d)$ as \begin{equation} \label{eqn:opI}
\gI(\vw) (t) = \begin{cases} \displaystyle{\frac{1}{t} \int_0^t s \vw(s) \, \mathrm{d} s}  & t > 0, \\[1ex]
\vzero & t = 0. \end{cases}
\end{equation} Notice that, by the continuity of $\vw$, we indeed have $\lim_{t \to 0+} \gI(\vw)(t) = \vzero$. 

\begin{prop}\label[prop]{prop:sf-ie}
    Let $\vx$, $\vy$, and $\vz$ be functions in $C([0, T], \rr^d)$.
    Then, for any given $\vx_0 \in \rr^d$, the following are equivalent. 
    \begin{enumerate}[label=(\roman*)]
        \item \label{item:g1a} The triple $(\vx, \vy, \vz)$ is a solution to the \sfode{} on the half-open interval $(0, T]$, and $\vx(0) = \vy(0) = \vz(0) = \vx_0$. 
        \item \label{item:g1b} The function $\vy$ is a solution of the integral equation \begin{equation} \label{eqn:y-ie}       
            \vy(t) = \vx_0 - \int_0^t \nabla f(\vy(s)) \, \mathrm{d}s + \beta \gI(\nabla f \circ \vy) (t), \qquad \forall t \in [0, T].
        \end{equation} Moreover, for such $\vy$ we have \begin{equation} \label{eqn:x-ie}  
            \vx(t) = \vx_0 - \int_0^t \nabla f(\vy(s)) \, \mathrm{d}s + \gI(\nabla f \circ \vy) (t), \qquad \forall t \in [0, T]
        \end{equation} and $\vz = \frac{1}{1-\beta}(\vy - \beta \vx)$. 
    \end{enumerate}
\end{prop}
\begin{proof}
    \ref{item:g1a}~$\Longrightarrow$~\ref{item:g1b} \hspace{1.5ex plus 0.5ex minus .2ex} Whenever $s > 0$, from \eqref{eqn:ode-z} and \eqref{eqn:ode-x} we have \begin{equation} \label{eqn:second-moment-derivative}\begin{aligned}
    \frac{\mathrm{d}}{\mathrm{d}s}\, s^2 \dot{\vx}(s) &= \frac{\mathrm{d}}{\mathrm{d}s}\, s \bigl(\vz(s) - \vx(s)\bigr) \\ 
     &= \vz(s) - \vx(s) + s\dot{\vz}(s) - s\dot{\vx}(s) \\
    &= s\dot{\vz}(s) \\
    &= - s \nabla f(\vy(s)).
    \end{aligned}\end{equation} Thus, by integrating on the interval $(0, u]$ we get \[
    u^2 \dot{\vx}(u) = - \int_{0}^{u} s \nabla f(\vy(s)) \, \mathrm{d}s.
    \] Dividing both sides by $u^2$ and then integrating leads to \begin{equation} \label{eqn:to-be-fubinied}
        \vx(t) - \vx(0) = -\int_0^t \int_{0}^{u} \frac{s}{u^2} \nabla f(\vy(s)) \, \mathrm{d}s \mathrm{d}u.
    \end{equation} 
    As $s \mapsto \nabla f(\vy(s))$ is continuous, it is bounded on the compact set $[0, T]$. 
    Let $M$ be an upper bound of $\norm{\nabla f(\vy)}$ on the interval $[0, T]$. 
    Then one can observe that the function $(s, u) \mapsto \frac{s}{u^2} \nabla f(\vy(s))$ is absolutely integrable on the triangle $\{(s, u) : 0 < s \leq u \leq t\}$ for any $t > 0$, from that \[
    \int_0^t \int_s^t \norm{\frac{s}{u^2} \nabla f(\vy(s))} \, \mathrm{d}u \mathrm{d}s \leq \int_0^t M s \left(\int_s^t \frac{1}{u^2} \, \mathrm{d}u \right) \mathrm{d}s = \int_0^t M  \left(1 - \frac{s}{t}\right) \mathrm{d}s = \frac{Mt}{2} < \infty. 
    \] Thus, we can apply the Fubini--Tonelli theorem on the right hand side of \eqref{eqn:to-be-fubinied} to get  \begin{equation} \label{eqn:fubinied}\begin{aligned}
        \vx(t) - \vx(0) &= - \int_{0}^{t} \int_s^t \frac{s}{u^2} \nabla f(\vy(s)) \, \mathrm{d}u \mathrm{d}s \\ 
        &= - \int_{0}^{t} \left( 1-\frac{s}t \right) \nabla f(\vy(s)) \,\mathrm{d}s \\ 
        &= - \int_{0}^{t} \nabla f(\vy(s)) \,\mathrm{d}s + \frac{1}{t} \int_{0}^{t} s \nabla f(\vy(s)) \,\mathrm{d}s .
    \end{aligned}\end{equation}
    Now, notice that \eqref{eqn:ode-y} implies $\vy - \vx = (1-\beta) (\vz - \vx)$. 
    Hence, from \eqref{eqn:second-moment-derivative} we also get \[
    \frac{\mathrm{d}}{\mathrm{d}s}\, s \bigl(\vy(s) - \vx(s)\bigr) = -s (1-\beta) \nabla f(\vy(s)). 
    \] Integrating both sides, we obtain \[
    t \vy(t) - t \vx(t) = -(1-\beta) \int_{0}^{t} s \nabla f(\vy(s))\, \mathrm{d}s.
    \] Dividing both sides by $t$ and then substituting the resulting formula of $\vx$ into \eqref{eqn:fubinied}, we conclude that \[
    \vy(t) = \vx_0 - \int_{0}^{t} \nabla f(\vy(s)) \,\mathrm{d}s + \frac{\beta}{t} \int_{0}^{t} s \nabla f(\vy(s)) \,\mathrm{d}s 
    \]  whenever $t > 0$, which is exactly \eqref{eqn:y-ie}. The fact that \eqref{eqn:y-ie} holds when $t = 0$ is immediate from the definition \eqref{eqn:opI} of $\gI$. 
    Once we have $\vy$ that satisfies \eqref{eqn:y-ie}, it follows from \eqref{eqn:fubinied} that \eqref{eqn:x-ie} also holds, and the formula for $\vz$ holds by \eqref{eqn:ode-y}. 

    \ref{item:g1b}~$\Longrightarrow$~\ref{item:g1a} \hspace{1.5ex plus 0.5ex minus .2ex} 
    Suppose that $\vy \in C([0, T], \rr^d)$ satisfies \eqref{eqn:y-ie}, and for such $\vy$ the function $\vx \in C([0, T], \rr^d)$ is defined as in~\eqref{eqn:x-ie} with $\vz = \frac{1}{1-\beta}(\vy - \beta\vx)$. 
    Since $\vx(0) = \vy(0) = \vz(0)$ is then immediate, it only remains to show that these functions solve the \sfode{} on $(0, T]$. 
    
    By the assumption that $\vz = \frac{1}{1-\beta}(\vy - \beta\vx)$, we have \eqref{eqn:ode-y}. Meanwhile, subtracting $\beta$ times \eqref{eqn:x-ie} from \eqref{eqn:y-ie}, we obtain \[
    \vy(t) - \beta \vx(t) = (1-\beta) \vx_0-(1-\beta) \int_{0}^{t} \nabla f ( \vy(s) ) \, \mathrm{d} s.
    \] Differentiating with respect to $t$ and using the already obtained \eqref{eqn:ode-y}, we end up with \[
    \dot{\vz}(t) = \frac{1}{1-\beta} \frac{\mathrm{d}}{\mathrm{d} t}\bigl(\vy(t) - \beta \vx(t)\bigr) = -\nabla f(\vy(t))
    \] which is exactly \eqref{eqn:ode-z}. 
    Finally, subtract \eqref{eqn:x-ie} directly from \eqref{eqn:y-ie}, to get \begin{align*}
    \vy(t) - \vx(t) &= -(1-\beta)\gI(\nabla f \circ \vy)(t) \\
    &= -\frac{1-\beta}{t} \int_0^t s \nabla f ( \vy(s) ) \, \mathrm{d} s.
    \end{align*}
    As \eqref{eqn:ode-y} also implies the identity $\vy - \vx = (1-\beta) (\vz - \vx)$, from the above we also have that \begin{equation} \label{eqn:x-dot-ie}
    t\bigl(\vz(t) - \vx(t)\bigr) =  - \int_0^t s \nabla f ( \vy(s) ) \, \mathrm{d} s. 
    \end{equation} Differentiating with respect to $t$ and using the already obtained \eqref{eqn:ode-z} leads to \begin{align*}
    \vz(t) - \vx(t) + t\dot{\vz}(t) - t\dot{\vx}(t) &= -t \nabla f ( \vy(t) ) \\
    &= t \dot{\vz}(t).
    \end{align*} Simply rearranging the terms gives us \eqref{eqn:ode-x}, recovering all three equations in the \sfode. 
\end{proof}

\Cref{prop:sf-ie} gives us a way to ``solve'' the \sfode{} on the \emph{closed} interval $[0, T]$. 
Indeed, for any given initial point $\vx_0\in\rr^d$, by solving the integral equation \eqref{eqn:y-ie}, one can construct $\vx(t), \vy(t), \vz(t) \in C([0, T], \rr^d)$ which together become a solution of the \sfode{} on the interval $(0, T]$ and satisfies $\lim_{t\to 0+}\bigl(\vx(t), \vy(t), \vz(t) \bigr) = (\vx_0, \vx_0, \vx_0)$. 

Thus, the main question is now whether the integral equation \eqref{eqn:y-ie} has a well-defined solution. 
We present an affirmative answer to this question. 
By mimicking the standard proof of the celebrated \mbox{Picard--Lindel\"{o}f} theorem for ODEs, we show that the solution of the integral equation \eqref{eqn:y-ie} uniquely exists on the interval $[0, T]$ whenever $T > 0$ is sufficiently small. 

To this end, given any fixed $\vx_0 \in \rr^d$, let us set \begin{equation} \label{eqn:small-large-t}
T =  \frac{1}{2L+2\norm{\nabla f(\vx_0)}}
\end{equation} and define an affine subset $\mathsf{B}$ of the Banach space $C([0, T], \rr^d)$ as \[
\mathsf{B} \coloneqq \left\{ \vy \in C([0, T], \rr^d) : \vy(0) = \vx_0 \right\}.
\] 
Let us also define an operator $\gT : \mathsf{B} \to \mathsf{B}$ as \begin{equation} \label{eqn:contraction}
\gT(\vy)(t) = \vx_0 - \int_0^t \nabla f(\vy(s)) \, \mathrm{d}s + \beta \gI(\nabla f \circ \vy) (t)
\end{equation} so that $\vy$ is a solution to the integral equation \eqref{eqn:y-ie} if and only if it is a fixed point of $\gT$. 
Observe that the definition \eqref{eqn:opI} of $\gI$ indeed implies that $\gT(\vy)(0) = \vx_0$ for any $\vy \in \mathsf{B}$, from which we get that $\gT(\mathsf{B}) \subset \mathsf{B}$ and there is no problem considering $\mathsf{B}$ as the codomain of $\gT$. 

\begin{prop} \label[prop]{prop:t-is-contraction} $\gT$ is a contraction on $\mathsf{B}$.
\end{prop} 
\begin{proof}
    For any two $\vy, \vw \in \mathsf{B}$ and $t > 0$, it holds that \[
        \gT(\vy)(t) - \gT(\vw)(t) = \int_0^t \nabla f(\vw(s)) - \nabla f(\vy(s)) \, \mathrm{d}s + \frac{\beta}{t} \int_0^t s(\nabla f(\vy(s)) - \nabla f(\vw(s))) \, \mathrm{d} s.
    \] Thus, by the $L$-smoothness of $f$, we obtain \begin{align*}
        \norm{\gT(\vy)(t) - \gT(\vw)(t)} &\leq \norm{\int_0^t \nabla f(\vw(s)) - \nabla f(\vy(s)) \, \mathrm{d}s} + \norm{\frac{\beta}{t} \int_0^t s(\nabla f(\vy(s)) - \nabla f(\vw(s))) \, \mathrm{d} s} \\ 
        &\leq \int_0^t \norm{\nabla f(\vw(s)) - \nabla f(\vy(s))} \, \mathrm{d}s + \frac{\beta}{t} \int_0^t s \norm{\nabla f(\vw(s)) - \nabla f(\vy(s))} \, \mathrm{d} s \\ 
        &\leq L \int_0^t \norm{\vy(s) -  \vw(s)} \, \mathrm{d}s + \frac{\beta L}{t} \int_0^t s\norm{\vy(s) - \vw(s)} \, \mathrm{d} s \\ 
        &\leq \norm{\vy - \vw}_\infty \left(L t + \frac{\beta L}{t} \int_0^t s  \, \mathrm{d} s \right) \\ 
        &=  \norm{\vy -  \vw}_\infty \left(1 + \frac{\beta}{2}\right)L t \\
        &\leq  \frac{3}{2} LT \norm{\vy -  \vw}_\infty 
    \end{align*} where in the last line we used the fact that $\beta \leq 1$. Because $T$ is chosen as in \eqref{eqn:small-large-t}, we have $T \leq \frac{1}{2L}$. 
    Furthermore, as we already know from $\gT(\mathsf{B}) \subset \mathsf{B}$ that $\gT(\vy)(0) = \gT(\vw)(0) = \vx_0$, we have $\gT(\vy)(0) - \gT(\vw)(0) = \vzero$, and hence it holds that \[
    \norm{\gT(\vy) - \gT(\vw)}_\infty = \sup_{t \in [0, T]}\; \norm{\gT(\vy)(t) - \gT(\vw)(t)} \leq \frac{3}{4} \norm{\vy -  \vw}_\infty .
    \]  Therefore, we conclude that $\gT$ is a contraction on $\mathsf{B}$ whose Lipschitz constant is at most $\frac{3}4$. 
\end{proof}

\begin{lem}
   On the interval $[0, T]$, the integral equation \eqref{eqn:y-ie} has a unique well-defined continuous solution.
\end{lem}
\begin{proof}
    Consider the operator $\gT : \mathsf{B} \to \mathsf{B}$ defined in \eqref{eqn:contraction}, which is a contraction, as we saw in \Cref{prop:t-is-contraction}. 
    Based on this observation, we wish to use the Banach fixed point theorem. 
    To this end, we claim that $\mathsf{B}$ itself is a nonempty complete Banach space. 
    As the constant function $t \mapsto \vx_0$ is in $\mathsf{B}$, it is clear that $\mathsf{B}$ is nonempty. 
    To see that $\mathsf{B}$ is complete, suppose that $\{\vw_k\}_{k \geq 1}$ is a Cauchy sequence in $\mathsf{B}$.
    Being a Cauchy sequence in a complete Banach space $C([0, T], \rr^d)$, we know that $\{\vw_k\}_{k \geq 1}$ converges to some $\vw_* \in  C([0, T], \rr^d)$. 
    As this convergence is uniform, we also have pointwise convergence, implying that $\vw_*(0) = \lim_{k \to \infty} \vw_k(0) = \vx_0$ and thus $\vw_* \in \mathsf{B}$.  
    It follows that $\mathsf{B}$ is indeed a complete Banach space. 

    We may now apply the Banach fixed point theorem on $\mathsf{B}$, to obtain the existence of a unique fixed point $\vy$ of $\gT$ in $\mathsf{B}$. 
    By definition, $\vy$ is a fixed point of $\gT$ if and only if $\vy$ is a solution to the integral equation \eqref{eqn:y-ie}, so we are done.
\end{proof}

Therefore, by considering the functions defined as in the statement of \Cref{prop:sf-ie}\ref{item:g1b}, it becomes possible to rigorously study the solution $(\vx, \vy, \vz)$ of the \sfode{} with the initial conditions given at $t = 0$. 
One property of such a solution that is useful when analyzing the Lyapunov function \eqref{eqn:conti-lyap} is as follows.  
\begin{prop} \label[prop]{prop:c1-extension}
    For $\vx$, $\vy$, and $\vz$ defined as in the statement of \Cref{prop:sf-ie}\ref{item:g1b}, it holds that \[
    \lim_{t\to 0+} \frac{\vz(t) - \vx(t)}{t} = -\frac{\nabla f ( \vx_0 )}{2}.
    \]
\end{prop} \begin{proof}
From \eqref{eqn:x-dot-ie}, it follows that \[
\frac{\vz(t) - \vx(t)}{t} =  -\frac{1}{t^2} \int_0^t s \nabla f ( \vy(s) ) \, \mathrm{d} s.
\]   Taking the limit as $t \to 0+$, by L'H\^{o}pital's rule we get \begin{align*}
   \lim_{t \to 0+} \frac{\vz(t) - \vx(t)}{t} &=  -\lim_{t \to 0+}\frac{\int_0^t s \nabla f ( \vy(s) ) \, \mathrm{d} s}{t^2} \\[1ex]
   &=  -\lim_{t \to 0+}\frac{  t \nabla f ( \vy(t) ) }{2 t}  \\[1ex]
   &=  -\frac{\nabla f ( \vx_0 )}{2}
\end{align*} because $\nabla f$ is $L$-Lipschitz by \Cref{asmp:smooth} hence continuous, while \Cref{prop:sf-ie} asserts that $\vy(t) \to \vx_0$ as $t\to 0+$. 
\end{proof}

By \eqref{eqn:ode-x} and \Cref{prop:c1-extension}, we have that \[
\lim_{t\to 0+} \dot{\vx}(t) = \lim_{t\to 0+} \frac{\vz(t) - \vx(t)}{t} = -\frac{\nabla f ( \vx_0 )}{2}. 
\] Meanwhile, as \eqref{eqn:ode-z} implies that \[
\lim_{t\to 0+} \dot{\vz}(t) = \lim_{t\to 0+} -\nabla f(\vy(t)) = -\nabla f(\vx_0) 
\] and \eqref{eqn:ode-y} additionally asserts that $\dot{\vy} = (1-\beta) \dot{\vz} + \beta \dot{\vx}$, we conclude that the solution can be extended to $t=0$ so that it is not only continuous, but also has a derivative that is continuous up to $t = 0$.

%% file: checklist.tex
\section*{NeurIPS Paper Checklist}

\begin{enumerate}

\item {\bf Claims}
    \item[] Question: Do the main claims made in the abstract and introduction accurately reflect the paper's contributions and scope?
    \item[] Answer: \answerYes{} %
    \item[] Justification: We wrote the abstract and introduction so that the main takeaways from \Cref{sec:sfode,sec:saddle,sec:why-xk,sec:conv} are included and summarized.  %
    \item[] Guidelines:
    \begin{itemize}
        \item The answer \answerNA{} means that the abstract and introduction do not include the claims made in the paper.
        \item The abstract and/or introduction should clearly state the claims made, including the contributions made in the paper and important assumptions and limitations. A \answerNo{} or \answerNA{} answer to this question will not be perceived well by the reviewers. 
        \item The claims made should match theoretical and experimental results, and reflect how much the results can be expected to generalize to other settings. 
        \item It is fine to include aspirational goals as motivation as long as it is clear that these goals are not attained by the paper. 
    \end{itemize}

\item {\bf Limitations}
    \item[] Question: Does the paper discuss the limitations of the work performed by the authors?
    \item[] Answer: \answerYes{} %
    \item[] Justification: Although we were unable to include a separate ``Limitations'' section due to the page limit, the limitations are discussed in \Cref{sec:conclusion}. %
    \item[] Guidelines:
    \begin{itemize}
        \item The answer \answerNA{} means that the paper has no limitation while the answer \answerNo{} means that the paper has limitations, but those are not discussed in the paper. 
        \item The authors are encouraged to create a separate ``Limitations'' section in their paper.
        \item The paper should point out any strong assumptions and how robust the results are to violations of these assumptions (e.g., independence assumptions, noiseless settings, model well-specification, asymptotic approximations only holding locally). The authors should reflect on how these assumptions might be violated in practice and what the implications would be.
        \item The authors should reflect on the scope of the claims made, e.g., if the approach was only tested on a few datasets or with a few runs. In general, empirical results often depend on implicit assumptions, which should be articulated.
        \item The authors should reflect on the factors that influence the performance of the approach. For example, a facial recognition algorithm may perform poorly when image resolution is low or images are taken in low lighting. Or a speech-to-text system might not be used reliably to provide closed captions for online lectures because it fails to handle technical jargon.
        \item The authors should discuss the computational efficiency of the proposed algorithms and how they scale with dataset size.
        \item If applicable, the authors should discuss possible limitations of their approach to address problems of privacy and fairness.
        \item While the authors might fear that complete honesty about limitations might be used by reviewers as grounds for rejection, a worse outcome might be that reviewers discover limitations that aren't acknowledged in the paper. The authors should use their best judgment and recognize that individual actions in favor of transparency play an important role in developing norms that preserve the integrity of the community. Reviewers will be specifically instructed to not penalize honesty concerning limitations.
    \end{itemize}

\item {\bf Theory assumptions and proofs}
    \item[] Question: For each theoretical result, does the paper provide the full set of assumptions and a complete (and correct) proof?
    \item[] Answer: \answerYes{} %
    \item[] Justification: All theoretical results explicitly state their required assumptions. Complete proofs are provided in the appendices, with the exception of direct corollaries that follow immediately from other theoretical results. %
    \item[] Guidelines:
    \begin{itemize}
        \item The answer \answerNA{} means that the paper does not include theoretical results. 
        \item All the theorems, formulas, and proofs in the paper should be numbered and cross-referenced.
        \item All assumptions should be clearly stated or referenced in the statement of any theorems.
        \item The proofs can either appear in the main paper or the supplemental material, but if they appear in the supplemental material, the authors are encouraged to provide a short proof sketch to provide intuition. 
        \item Inversely, any informal proof provided in the core of the paper should be complemented by formal proofs provided in appendix or supplemental material.
        \item Theorems and Lemmas that the proof relies upon should be properly referenced. 
    \end{itemize}

    \item {\bf Experimental result reproducibility}
    \item[] Question: Does the paper fully disclose all the information needed to reproduce the main experimental results of the paper to the extent that it affects the main claims and/or conclusions of the paper (regardless of whether the code and data are provided or not)?
    \item[] Answer: \answerYes{} %
    \item[] Justification: While we use the terminology PEP \emph{experiments}, these amount to numerically solving convex optimization problems rather than conducting empirical experiments as commonly understood in machine learning. Nevertheless, all details needed to reproduce the results are described in \Cref{sec:pep-experiments}. %
    \item[] Guidelines:
    \begin{itemize}
        \item The answer \answerNA{} means that the paper does not include experiments.
        \item If the paper includes experiments, a \answerNo{} answer to this question will not be perceived well by the reviewers: Making the paper reproducible is important, regardless of whether the code and data are provided or not.
        \item If the contribution is a dataset and\slash or model, the authors should describe the steps taken to make their results reproducible or verifiable. 
        \item Depending on the contribution, reproducibility can be accomplished in various ways. For example, if the contribution is a novel architecture, describing the architecture fully might suffice, or if the contribution is a specific model and empirical evaluation, it may be necessary to either make it possible for others to replicate the model with the same dataset, or provide access to the model. In general. releasing code and data is often one good way to accomplish this, but reproducibility can also be provided via detailed instructions for how to replicate the results, access to a hosted model (e.g., in the case of a large language model), releasing of a model checkpoint, or other means that are appropriate to the research performed.
        \item While NeurIPS does not require releasing code, the conference does require all submissions to provide some reasonable avenue for reproducibility, which may depend on the nature of the contribution. For example
        \begin{enumerate}
            \item If the contribution is primarily a new algorithm, the paper should make it clear how to reproduce that algorithm.
            \item If the contribution is primarily a new model architecture, the paper should describe the architecture clearly and fully.
            \item If the contribution is a new model (e.g., a large language model), then there should either be a way to access this model for reproducing the results or a way to reproduce the model (e.g., with an open-source dataset or instructions for how to construct the dataset).
            \item We recognize that reproducibility may be tricky in some cases, in which case authors are welcome to describe the particular way they provide for reproducibility. In the case of closed-source models, it may be that access to the model is limited in some way (e.g., to registered users), but it should be possible for other researchers to have some path to reproducing or verifying the results.
        \end{enumerate}
    \end{itemize}

\item {\bf Open access to data and code}
    \item[] Question: Does the paper provide open access to the data and code, with sufficient instructions to faithfully reproduce the main experimental results, as described in supplemental material?
    \item[] Answer: \answerYes{} %
    \item[] Justification: The codes we used for the experiments are submitted along with the paper. No external datasets were used. %
    \item[] Guidelines:
    \begin{itemize}
        \item The answer \answerNA{} means that paper does not include experiments requiring code.
        \item Please see the NeurIPS code and data submission guidelines (\url{https://neurips.cc/public/guides/CodeSubmissionPolicy}) for more details.
        \item While we encourage the release of code and data, we understand that this might not be possible, so \answerNo{} is an acceptable answer. Papers cannot be rejected simply for not including code, unless this is central to the contribution (e.g., for a new open-source benchmark).
        \item The instructions should contain the exact command and environment needed to run to reproduce the results. See the NeurIPS code and data submission guidelines (\url{https://neurips.cc/public/guides/CodeSubmissionPolicy}) for more details.
        \item The authors should provide instructions on data access and preparation, including how to access the raw data, preprocessed data, intermediate data, and generated data, etc.
        \item The authors should provide scripts to reproduce all experimental results for the new proposed method and baselines. If only a subset of experiments are reproducible, they should state which ones are omitted from the script and why.
        \item At submission time, to preserve anonymity, the authors should release anonymized versions (if applicable).
        \item Providing as much information as possible in supplemental material (appended to the paper) is recommended, but including URLs to data and code is permitted.
    \end{itemize}

\item {\bf Experimental setting/details}
    \item[] Question: Does the paper specify all the training and test details (e.g., data splits, hyperparameters, how they were chosen, type of optimizer) necessary to understand the results?
    \item[] Answer: \answerYes{} %
    \item[] Justification: How PEP instances are constructed is fully described in \Cref{sec:pep-experiments}. %
    \item[] Guidelines:
    \begin{itemize}
        \item The answer \answerNA{} means that the paper does not include experiments.
        \item The experimental setting should be presented in the core of the paper to a level of detail that is necessary to appreciate the results and make sense of them.
        \item The full details can be provided either with the code, in appendix, or as supplemental material.
    \end{itemize}

\item {\bf Experiment statistical significance}
    \item[] Question: Does the paper report error bars suitably and correctly defined or other appropriate information about the statistical significance of the experiments?
    \item[] Answer: \answerNo{} %
    \item[] Justification: PEP results provide worst-case bounds that are exact up to the precision of the convex optimization solver used. Therefore, statistical significance testing is not applicable. %
    \item[] Guidelines:
    \begin{itemize}
        \item The answer \answerNA{} means that the paper does not include experiments.
        \item The authors should answer \answerYes{} if the results are accompanied by error bars, confidence intervals, or statistical significance tests, at least for the experiments that support the main claims of the paper.
        \item The factors of variability that the error bars are capturing should be clearly stated (for example, train/test split, initialization, random drawing of some parameter, or overall run with given experimental conditions).
        \item The method for calculating the error bars should be explained (closed form formula, call to a library function, bootstrap, etc.)
        \item The assumptions made should be given (e.g., Normally distributed errors).
        \item It should be clear whether the error bar is the standard deviation or the standard error of the mean.
        \item It is OK to report 1-sigma error bars, but one should state it. The authors should preferably report a 2-sigma error bar than state that they have a 96\% CI, if the hypothesis of Normality of errors is not verified.
        \item For asymmetric distributions, the authors should be careful not to show in tables or figures symmetric error bars that would yield results that are out of range (e.g., negative error rates).
        \item If error bars are reported in tables or plots, the authors should explain in the text how they were calculated and reference the corresponding figures or tables in the text.
    \end{itemize}

\item {\bf Experiments compute resources}
    \item[] Question: For each experiment, does the paper provide sufficient information on the computer resources (type of compute workers, memory, time of execution) needed to reproduce the experiments?
    \item[] Answer: \answerYes{} %
    \item[] Justification: While there are no strict requirements on the computational resources needed to reproduce our PEP results, for completeness, we report the computing environment used for our computations in \Cref{sec:pep-experiments}. %
    \item[] Guidelines:
    \begin{itemize}
        \item The answer \answerNA{} means that the paper does not include experiments.
        \item The paper should indicate the type of compute workers CPU or GPU, internal cluster, or cloud provider, including relevant memory and storage.
        \item The paper should provide the amount of compute required for each of the individual experimental runs as well as estimate the total compute. 
        \item The paper should disclose whether the full research project required more compute than the experiments reported in the paper (e.g., preliminary or failed experiments that didn't make it into the paper). 
    \end{itemize}
    
\item {\bf Code of ethics}
    \item[] Question: Does the research conducted in the paper conform, in every respect, with the NeurIPS Code of Ethics \url{https://neurips.cc/public/EthicsGuidelines}?
    \item[] Answer: \answerYes{} %
    \item[] Justification: We have read the Code of Ethics, and found no conflicts with it. %
    \item[] Guidelines:
    \begin{itemize}
        \item The answer \answerNA{} means that the authors have not reviewed the NeurIPS Code of Ethics.
        \item If the authors answer \answerNo, they should explain the special circumstances that require a deviation from the Code of Ethics.
        \item The authors should make sure to preserve anonymity (e.g., if there is a special consideration due to laws or regulations in their jurisdiction).
    \end{itemize}

\item {\bf Broader impacts}
    \item[] Question: Does the paper discuss both potential positive societal impacts and negative societal impacts of the work performed?
    \item[] Answer: \answerNA{} %
    \item[] Justification: Since our work focuses on the theoretical aspects of an optimization scheme, we do not expect it to have any direct societal impact.  %
    \item[] Guidelines:
    \begin{itemize}
        \item The answer \answerNA{} means that there is no societal impact of the work performed.
        \item If the authors answer \answerNA{} or \answerNo, they should explain why their work has no societal impact or why the paper does not address societal impact.
        \item Examples of negative societal impacts include potential malicious or unintended uses (e.g., disinformation, generating fake profiles, surveillance), fairness considerations (e.g., deployment of technologies that could make decisions that unfairly impact specific groups), privacy considerations, and security considerations.
        \item The conference expects that many papers will be foundational research and not tied to particular applications, let alone deployments. However, if there is a direct path to any negative applications, the authors should point it out. For example, it is legitimate to point out that an improvement in the quality of generative models could be used to generate Deepfakes for disinformation. On the other hand, it is not needed to point out that a generic algorithm for optimizing neural networks could enable people to train models that generate Deepfakes faster.
        \item The authors should consider possible harms that could arise when the technology is being used as intended and functioning correctly, harms that could arise when the technology is being used as intended but gives incorrect results, and harms following from (intentional or unintentional) misuse of the technology.
        \item If there are negative societal impacts, the authors could also discuss possible mitigation strategies (e.g., gated release of models, providing defenses in addition to attacks, mechanisms for monitoring misuse, mechanisms to monitor how a system learns from feedback over time, improving the efficiency and accessibility of ML).
    \end{itemize}
    
\item {\bf Safeguards}
    \item[] Question: Does the paper describe safeguards that have been put in place for responsible release of data or models that have a high risk for misuse (e.g., pre-trained language models, image generators, or scraped datasets)?
    \item[] Answer: \answerNA{} %
    \item[] Justification: This paper poses no such risks. %
    \item[] Guidelines:
    \begin{itemize}
        \item The answer \answerNA{} means that the paper poses no such risks.
        \item Released models that have a high risk for misuse or dual-use should be released with necessary safeguards to allow for controlled use of the model, for example by requiring that users adhere to usage guidelines or restrictions to access the model or implementing safety filters. 
        \item Datasets that have been scraped from the Internet could pose safety risks. The authors should describe how they avoided releasing unsafe images.
        \item We recognize that providing effective safeguards is challenging, and many papers do not require this, but we encourage authors to take this into account and make a best faith effort.
    \end{itemize}

\item {\bf Licenses for existing assets}
    \item[] Question: Are the creators or original owners of assets (e.g., code, data, models), used in the paper, properly credited and are the license and terms of use explicitly mentioned and properly respected?
    \item[] Answer: \answerYes{} %
    \item[] Justification: The \texttt{Python} package used in the PEP experiments, namely \texttt{PEPit}, is cited, but since it is publicly available, we did not deem it necessary to explicitly mention its license and terms of use in the paper.
    No specific existing data or models have been used. %
    \item[] Guidelines:
    \begin{itemize}
        \item The answer \answerNA{} means that the paper does not use existing assets.
        \item The authors should cite the original paper that produced the code package or dataset.
        \item The authors should state which version of the asset is used and, if possible, include a URL.
        \item The name of the license (e.g., CC-BY 4.0) should be included for each asset.
        \item For scraped data from a particular source (e.g., website), the copyright and terms of service of that source should be provided.
        \item If assets are released, the license, copyright information, and terms of use in the package should be provided. For popular datasets, \url{paperswithcode.com/datasets} has curated licenses for some datasets. Their licensing guide can help determine the license of a dataset.
        \item For existing datasets that are re-packaged, both the original license and the license of the derived asset (if it has changed) should be provided.
        \item If this information is not available online, the authors are encouraged to reach out to the asset's creators.
    \end{itemize}

\item {\bf New assets}
    \item[] Question: Are new assets introduced in the paper well documented and is the documentation provided alongside the assets?
    \item[] Answer: \answerNA{} %
    \item[] Justification: This paper does not release new assets. %
    \item[] Guidelines:
    \begin{itemize}
        \item The answer \answerNA{} means that the paper does not release new assets.
        \item Researchers should communicate the details of the dataset\slash code\slash model as part of their submissions via structured templates. This includes details about training, license, limitations, etc. 
        \item The paper should discuss whether and how consent was obtained from people whose asset is used.
        \item At submission time, remember to anonymize your assets (if applicable). You can either create an anonymized URL or include an anonymized zip file.
    \end{itemize}

\item {\bf Crowdsourcing and research with human subjects}
    \item[] Question: For crowdsourcing experiments and research with human subjects, does the paper include the full text of instructions given to participants and screenshots, if applicable, as well as details about compensation (if any)? 
    \item[] Answer: \answerNA{} %
    \item[] Justification: This paper does not involve crowdsourcing nor research with human subjects. %
    \item[] Guidelines:
    \begin{itemize}
        \item The answer \answerNA{} means that the paper does not involve crowdsourcing nor research with human subjects.
        \item Including this information in the supplemental material is fine, but if the main contribution of the paper involves human subjects, then as much detail as possible should be included in the main paper. 
        \item According to the NeurIPS Code of Ethics, workers involved in data collection, curation, or other labor should be paid at least the minimum wage in the country of the data collector. 
    \end{itemize}

\item {\bf Institutional review board (IRB) approvals or equivalent for research with human subjects}
    \item[] Question: Does the paper describe potential risks incurred by study participants, whether such risks were disclosed to the subjects, and whether Institutional Review Board (IRB) approvals (or an equivalent approval/review based on the requirements of your country or institution) were obtained?
    \item[] Answer: \answerNA{} %
    \item[] Justification: This paper does not involve crowdsourcing nor research with human subjects. %
    \item[] Guidelines:
    \begin{itemize}
        \item The answer \answerNA{} means that the paper does not involve crowdsourcing nor research with human subjects.
        \item Depending on the country in which research is conducted, IRB approval (or equivalent) may be required for any human subjects research. If you obtained IRB approval, you should clearly state this in the paper. 
        \item We recognize that the procedures for this may vary significantly between institutions and locations, and we expect authors to adhere to the NeurIPS Code of Ethics and the guidelines for their institution. 
        \item For initial submissions, do not include any information that would break anonymity (if applicable), such as the institution conducting the review.
    \end{itemize}

\item {\bf Declaration of LLM usage}
    \item[] Question: Does the paper describe the usage of LLMs if it is an important, original, or non-standard component of the core methods in this research? Note that if the LLM is used only for writing, editing, or formatting purposes and does \emph{not} impact the core methodology, scientific rigor, or originality of the research, declaration is not required.
    \item[] Answer: \answerNA{} %
    \item[] Justification: The core method development in this research does not involve LLMs as any important, original, or non-standard components. %
    \item[] Guidelines:
    \begin{itemize}
        \item The answer \answerNA{} means that the core method development in this research does not involve LLMs as any important, original, or non-standard components.
        \item Please refer to our LLM policy in the NeurIPS handbook for what should or should not be described.
    \end{itemize}

\end{enumerate}